\documentclass{article}
\makeatletter
\def\input@path{{./}}
\makeatother
\usepackage{iclr2027_conference,times}
\usepackage{amsmath,amssymb,amsthm}
\usepackage{graphicx,booktabs,array}
\graphicspath{{./}}
\usepackage{tikz}
\usetikzlibrary{arrows.meta,positioning,fit,calc}
\usepackage{pgfplots}
\usepgfplotslibrary{groupplots,fillbetween}
\pgfplotsset{compat=1.18}
\usepackage{placeins}
\usepackage{hyperref}
\usepackage{url}
\usepackage{xcolor}
\theoremstyle{plain}
\newtheorem{theorem}{Theorem}

\newtheorem{proposition}[theorem]{Proposition}

\theoremstyle{definition}

\theoremstyle{remark}

\definecolor{stlblue}{HTML}{0072B2}
\definecolor{stlorange}{HTML}{D55E00}
\definecolor{stlgray}{HTML}{666666}
\definecolor{stllightgray}{HTML}{A6A6A6}

\pgfplotsset{
    stl axis/.style={
        axis background/.style={fill=white},
        axis line style={draw=black!70, line width=0.45pt},
        tick style={draw=black!70, line width=0.45pt},
        tick align=outside,
        xtick pos=bottom,
        ytick pos=left,
        ymajorgrids,
        major grid style={draw=black!12, line width=0.35pt},
        scaled ticks=false,
        enlargelimits=false,
        unbounded coords=discard,
        filter discard warning=false,
        label style={font=\fontsize{8}{9.5}\selectfont},
        tick label style={font=\fontsize{8}{9.5}\selectfont},
        yticklabel style={inner xsep=1pt},
        title style={font=\fontsize{8}{9.5}\selectfont, align=center, yshift=1pt},
        legend style={
            draw=none,
            fill=none,
            font=\fontsize{8}{9.5}\selectfont,
            cells={anchor=west},
            row sep=1pt,
        },
        every axis plot/.append style={line join=round, line cap=round},
    },
    stl utl/.style={
        color=black,
        dashed,
        line width=0.8pt,
        mark=o,
        mark size=1.35pt,
        mark options={solid, fill=white, line width=0.55pt},
    },
    stl pm/.style={
        color=stlblue,
        solid,
        line width=0.9pt,
        mark=triangle*,
        mark size=1.7pt,
        mark options={solid},
    },
    stl influence/.style={
        color=stlorange,
        solid,
        line width=1.05pt,
        mark=square*,
        mark size=1.8pt,
        mark options={solid},
    },
    stl strict/.style={
        color=stlgray,
        densely dotted,
        line width=0.95pt,
    },
    stl identity/.style={
        color=black!65,
        densely dashed,
        line width=0.85pt,
    },
}

\usepackage{xcolor} 

\title{Selective Tie Learning: Mitigating Spurious Correlations in Preference Optimization}

\author{Christian Moya$^{1\dagger}$\thanks{Corresponding author: \texttt{cmoyacal@purdue.edu}. $\dagger$ Equal contribution.}, Purav Matlia$^{1\dagger}$, Elliott Thornley$^2$, and Guang Lin$^1$  \\
$^1$ Purdue University, $^2$ National University of Singapore
}

\newif\ifshowunfinishedsections
\showunfinishedsectionsfalse

\iclrfinalcopy

\begin{document}
\maketitle
\lhead{Preprint. Under review.} 

\begin{abstract}
Preference optimization can learn spurious correlations, and augmenting the preference data with ties (equally preferred responses) can reduce spurious reliance. However, ties differ in their corrective effect, making their selection for augmentation important. For log-linear policies, we derive each tie's influence on spurious reliance and show that taking the largest predicted reductions is first-order optimal among subsets of equal size. Computing these influences requires knowledge of the spurious features and the curvature of the training objective, which is often unavailable or expensive to estimate. To this end, we explore using the magnitude of a tie's implicit reward margin and characterize the conditions under which it is a useful proxy for predicted influence. To test the theory and its qualitative transfer to neural models, we estimate influence with a counterfactual proxy for spurious reliance and use a generalized Gauss--Newton approximation to curvature for language models. Across log-linear policies, neural models, and large language models, our selection rules consistently achieve comparable reductions in spurious reliance with fewer ties than uniform selection.

\end{abstract}

\section{Introduction}
\label{sec:introduction}
Aligning large language models with human preferences is central to developing useful and safe AI systems~\citep{ouyang2022training,bai2022training}.
Standard methods for this alignment, including reinforcement learning from human feedback (RLHF) and direct preference optimization (DPO)~\citep{DBLP:conf/nips/RafailovSMMEF23}, learn policies from pairwise comparisons of responses.
Learning from comparisons, however, can encourage reliance on features that predict human labels without tracking substantive quality, such as verbosity~\citep{DBLP:journals/corr/abs-2310-03716, saito2023verbosity} and sycophancy~\citep{DBLP:conf/iclr/SharmaTKDABDHJK24}.
Models that rely on these features can fail when the features stop predicting human labels at deployment.

Recent theory shows that reliance on spurious features can persist even with unlimited preference data, because it arises from properties of the training distribution~\citep{moya2026spurious}. Collecting more examples from the same distribution therefore does not remove it. Because the bias originates in the training distribution, two families of interventions target it: methods that reshape the training objective~\citep{DBLP:conf/acl/ParkREF24} and methods that modify the training examples~\citep{DBLP:conf/iclr/000200CWJGS0YSM25,DBLP:journals/corr/abs-2511-12573}. The two families are complementary, since interventions on the training examples can work with existing objectives for preference learning rather than replace them. Our work advances interventions of this second kind.

Tie training is one such data intervention~\citep{moya2026spurious}. It augments preference data with ties: pairs of responses with equal utility that differ in spurious attributes. Under local log-linear assumptions, this augmentation provably suppresses reliance on these attributes. Related data interventions include the neutral pairs of RRM~\citep{DBLP:conf/iclr/000200CWJGS0YSM25} and counterfactual augmentation for length bias~\citep{DBLP:journals/corr/abs-2511-12573}. However, both methods keep many augmented pairs even after filtering, which adds training computation. This cost raises a question: can careful selection retain most of the reduction in spurious reliance with fewer ties?

To study this question, we analyze log-linear policies with a core--spurious feature decomposition. Within this setting, we propose Selective Tie Learning (STL), which selects a small subset of ties using a model trained on the original preference data. Our contributions are:

\textit{(i)} We derive the first-order influence of each tie on the model's reliance on spurious features (Proposition~\ref{prop:single-tie-influence}). Selecting the ties with the largest predicted reductions is first-order optimal among subsets of the same size (Theorem~\ref{thm:influence-selection}).

\textit{(ii)} Because computing the influence in \textit{(i)} requires the Hessian of the loss, we derive a cheaper selector that needs only forward passes. We show that, under a Gaussian model of the ties, selecting the ties with the largest absolute implicit reward margins maximizes the expected first-order reduction in spurious reliance among all such selectors (Theorem~\ref{thm:pm-optimality}).

\textit{(iii)} To extend \textit{(i)} to neural models, we estimate the influence of ties using a counterfactual measure of spurious reliance and a damped generalized Gauss--Newton approximation of the Hessian (Section~\ref{sec:nonlinear}). The estimate needs no explicit core--spurious feature decomposition and shares one Hessian inversion across all candidate ties.

We test the theoretical predictions in log-linear models and evaluate their transfer to neural networks and language models. Across these settings, targeted selection achieves comparable reductions in spurious reliance with fewer ties than uniform selection.

\section{Setup and Problem Formulation}
\label{sec:setup_and_problem_formulation}
\vspace{-0.5em}
\subsection{DPO with a Log-Linear Policy} \label{subsec:dpo_log_linear}
\textbf{Preference data.} Preference learning trains a policy from comparisons between responses to the same prompt. These comparisons form a dataset $\mathcal D_{\mathrm{pref}}=\{(x_i,y_{w,i},y_{l,i})\}_{i=1}^N$ sampled from a distribution $P$, where $y_{w,i}$ is preferred to $y_{l,i}$ for prompt $x_i\in\mathcal X$ \citep{DBLP:conf/nips/RafailovSMMEF23,DBLP:conf/nips/ChristianoLBMLA17}.

\textbf{Core and spurious features.} We represent the prompt and responses in each comparison using a fixed feature map $\phi(x,y)=[\phi_c(x,y);\phi_s(x,y)]$ and assume that the preference depends only on the core features $\phi_c$. The spurious features $\phi_s$ can still correlate with the preference in the training data, although this correlation need not persist at deployment.

\textbf{Log-linear policy.} To analyse how a policy uses these features, we assume a log-linear form, $\pi_\theta(y\mid x)
= \exp(\theta^\top\phi(x,y)) / Z_\theta(x),$ where the normalizer $Z_\theta(x)$ sums over responses and $\theta=[\theta_c;\theta_s]$ weights the core and spurious features. We assume the fixed reference policy $\pi_{\rm ref}=\pi_{\theta_{\rm ref}}$ belongs to the same family. The log-ratio $\log(\pi_\theta/\pi_{\rm ref})$ then equals $(\theta-\theta_{\rm ref})^\top\phi(x,y)$ up to a normalization term that cancels when comparing responses to the same prompt.

\textbf{DPO.} DPO scales this log-ratio into the implicit reward $r_\theta(x,y)=\beta\log\big(\pi_\theta(y\mid x)/\pi_{\rm ref}(y\mid x)\big)$, where $\beta>0$ controls regularization toward the reference policy. Within each comparison the normalization cancels, and the reward difference is $\beta\tilde\theta^\top\Delta\phi$, where $\tilde\theta=\theta-\theta_{\rm ref}$ and $\Delta\phi=\phi(x,y_w)-\phi(x,y_l)$. DPO fits this difference to the observed preferences through the logistic loss
$$
\mathcal L_{\rm DPO}(\theta)
=-\mathbb E_{\mathcal D_{\rm pref}}
\Big[\log\sigma\big(\beta\tilde\theta^\top\Delta\phi\big)\Big],
$$
where $\sigma$ is the logistic function. Because $\Delta\phi= [\Delta\phi_c;\Delta\phi_s]$, both core and spurious features contribute to the reward difference. DPO can therefore reduce the loss by exploiting spurious correlations that predict the winner.
\vspace{-0.5em}
\subsection{Spurious Reliance and Ties as an Intervention} \label{subsec:spurious_reliance}
DPO exploits spurious correlations even at the population level \citep{moya2026spurious}. With $\mu=\mathbb E_P[\Delta\phi]$ and $\Sigma=\mathbb E_P[\Delta\phi\Delta\phi^\top]$, two mechanisms drive this reliance. Through the mean, spurious differences predict the winner directly ($\mu_s\neq0$). Through leakage, they correlate with core differences ($\Sigma_{sc}\neq0$) and inherit their predictive signal. Both depend on $P$ alone, so more comparisons from the same distribution do not remove the reliance.

\textbf{Ties as an intervention.} These two mechanisms motivate comparisons that carry neither a mean spurious bias nor leakage from core differences. We call such comparisons \emph{ties}: a tie $\tau=(x,\{y_A,y_B\})$ pairs two responses with identical core features and different spurious features, so its feature difference $\Delta\phi=[0;\Delta\phi_s]$ has no core part to leak from. Assigning either response as the winner with equal probability also makes the mean spurious difference zero. \citet{moya2026spurious} show that adding such ties to DPO training reduces spurious reliance.
\vspace{-0.5em}
\subsection{Problem: Selecting Ties from a Pool} \label{subsec:problem_selection}
The benefit of tie training \citep{moya2026spurious} leaves open which ties to use. Ties differ because a tie acts on the policy only through its margin: if the policy already assigns equal implicit reward to both responses, training on the tie changes nothing. Training on the whole pool $\mathcal D_{\rm tie}$ of $M$ candidates therefore spends effort on ties that change little. Under a budget of $k <  M$ ties, we instead seek the subset of the given pool whose training most reduces spurious reliance. Such a subset should outperform $k$ ties sampled uniformly, the simplest selector. 

In this paper, each candidate comes from a strict comparison sampled from
$\mathcal D_{\rm pref}$. A known transformation $T$ then changes the spurious features of the preferred response while approximately preserving its core features, $\phi_c(x,T(y_w))-\phi_c(x,y_w)\approx0$. So that each response wins equally often, each candidate is a packet with both orientations, $\tau=\{\tau_\ell,\tau_r\},$ where $\tau_\ell=(x,y_w,T(y_w))$ and $\tau_r=(x,T(y_w),y_w)$ with equal probability, so their feature differences cancel in expectation, while each orientation has approximately zero core difference. We score and select whole packets, then sample one orientation from each selected packet for training. Our selectors also accept pools from other constructions that satisfy the assumptions of our analysis. We study this selection problem at two levels of access to the model. Section~\ref{sec:influence} derives the best selector when gradients and curvature are available, and Section~\ref{sec:margin} derives the best selector when only forward passes are available.

\section{Selection with Full Access: Influence}
\label{sec:influence}
\label{sec:stl}
In this section, we measure spurious reliance by a spurious energy and analyze, to first order, how training on a tie changes it. We show that this first-order change is the classical influence of the tie on the energy \citep{DBLP:conf/icml/KohL17}, so selecting the $k$ ties with the largest predicted reduction solves the selection problem to first order.
\vspace{-0.5em}
\subsection{The First-Order Effect of a Tie on Spurious Energy} \label{subsec:first_order_effect}

\textbf{Spurious energy.} Spurious reliance shows up in the margin. The margin on a comparison splits into a core and a spurious part, $\beta(\tilde\theta_c^\top\Delta\phi_c+\tilde\theta_s^\top\Delta\phi_s)$, and only the core part reflects the preference. We therefore measure reliance by the mean square of the spurious part, before scaling by $\beta$:
$$
E(\theta)=\mathbb E_P\big[(\tilde\theta_s^\top\Delta\phi_s)^2\big]
=\tilde\theta_s^\top\Sigma_{ss}\tilde\theta_s .
$$
Squaring counts spurious contributions of either sign, and weighting by $\Sigma_{ss}$ counts spurious weights only along directions in which
responses actually differ. The energy vanishes exactly when spurious features never move the margin under $P$, and a tie is useful to the extent that training on it lowers this energy.

We condition on the realized candidate pool and analyze each packet $\tau_j$ through its DPO loss averaged over both orientations. We assume that the two responses in a tie have the same core features, so the feature differences of the two orientations are $[0;d_j]$ and $[0;-d_j]$, where
$$
d_j=\phi_s(x_j,y_{w,j})-\phi_s(x_j,T(y_{w,j})).
$$
Because the loss is averaged over both orientations, reversing the sign of $d_j$ leaves it unchanged, so either orientation can define $d_j$. With margin $m_j=\beta\tilde\theta_s^\top d_j$ for the first orientation, the tie loss is
$$
\ell_j(\theta)=-\tfrac12\big[\log\sigma(m_j)+\log\sigma(-m_j)\big].
$$
To isolate the effect of one tie, we add it to DPO with a small weight $\epsilon/N$,
$$
\theta_j^\epsilon=\arg\min_\theta\;\mathcal L_{\rm DPO}(\theta)+\frac{\epsilon}{N}\,\ell_j(\theta),
$$
so that $\epsilon=0$ recovers the DPO solution $\theta_0$. The effect of the tie on reliance is then the rate at which the spurious energy $E(\theta_j^\epsilon)$ changes as $\epsilon$ grows from zero.

\textbf{Tie gradient.} The gradient of the tie loss, $g_j=\nabla_\theta\ell_j$, drives this rate. Because the margin on a tie depends only on the spurious weights, so does the gradient:
$$
g_{j,c}=0,\qquad
g_{j,s}=\frac{\beta}{2}\tanh\!\Big(\frac{m_j}{2}\Big)d_j .
$$
This is the gradient of the packet loss, averaged over both orientations. Flipping the orientation flips both $d_j$ and $m_j$, so the packet gradient does not depend on which orientation defines the contrast $d_j$. A tie therefore pushes the policy only along its spurious contrast, with a strength set by its margin. The push vanishes when the policy already assigns equal implicit reward to both responses, grows as $\tfrac{\beta}{4}m_j d_j$ for small margins, and saturates at magnitude $\tfrac{\beta}{2}\|d_j\|$ for large ones.

\textbf{From push to influence.} The gradient gives the immediate push of a tie, but DPO resists it, and the curvature of the DPO loss decides how far the solution moves. We analyze the DPO policy $\theta_0$, a solution at which the DPO loss is stationary and strictly convex,
$$
\nabla\mathcal L_{\rm DPO}(\theta_0)=0,\qquad
H=\nabla^2\mathcal L_{\rm DPO}(\theta_0)\succ0 .
$$
Adding the tie with weight $\epsilon/N$ tilts this stationarity condition, and the solution restores it by moving at rate $-H^{-1}g_j(\theta_0)/N$. Only the spurious part of this move changes the energy. While the tie pushes the spurious weights, DPO refits the core weights, so the spurious weights move at rate $-K_s^{-1}g_{j,s}(\theta_0)/N$, where $K_s=H_{ss}-H_{sc}H_{cc}^{-1}H_{cs}$ is the Schur complement of the core block. The resulting rate of change of the energy is the influence of the tie
\citep{DBLP:conf/icml/KohL17}.

\begin{proposition}[Influence of a tie on spurious energy]
\label{prop:single-tie-influence}
Suppose $\theta_0$ is stationary with $H\succ0$. For small $\epsilon$, the solution $\theta_j^\epsilon$ satisfies $E(\theta_j^\epsilon)-E(\theta_0)=-\epsilon A_j+O(\epsilon^2)$, with
$$
A_j=\frac1N\big\langle\nabla_sE(\theta_0),K_s^{-1}g_{j,s}(\theta_0)\big\rangle
=\frac{\beta}{2N}\tanh\!\Big(\frac{m_j}{2}\Big)
\big\langle\nabla_sE(\theta_0),K_s^{-1}d_j\big\rangle ,
$$
where $m_j=\beta\tilde\theta_{0,s}^\top d_j$ is the margin on the tie at
$\theta_0$.
\end{proposition}
\begin{proof}
See Appendix~\ref{app:influence-proofs}.    
\end{proof}
\textbf{Reading the influence.} A positive $A_j$ predicts that the tie lowers the energy, and the influence shows what makes it positive. Its size depends on two factors. The margin factor $\tanh(m_j/2)$ is the strength of the push: a tie on which the policy is already indifferent has no influence. The alignment factor $\langle\nabla_sE(\theta_0),K_s^{-1}d_j\rangle$ measures whether that push, after DPO refits the core weights, points against the current reliance. The alignment factor need not share the sign of the margin, so a tie can raise the energy. Influences of different ties add to first order, so choosing a set of ties reduces to ranking them.

\vspace{-0.5em}
\subsection{Main Result: First-Order Optimal Selection} \label{subsec:influence_selection}
Training on a set $S$ of ties, each with weight $\epsilon/N$, tilts the stationarity condition by the sum of their gradients. At first order the solution responds linearly to this tilt, so the influences of the ties add:
$$
E(\theta_S^\epsilon)-E(\theta_0)=-\epsilon\sum_{j\in S}A_j+O(\epsilon^2).
$$
Under the budget, the best set is the one that makes this sum largest.
\begin{theorem}[First-order optimal tie selection]
\label{thm:influence-selection}
Let $S_k$ contain the $k$ ties with the largest $A_j$. Then
$$
\sum_{j\in S_k}A_j=\max_{|S|=k}\sum_{j\in S}A_j
\;\ge\;\frac{k}{M}\sum_{j=1}^{M}A_j ,
$$
with strict inequality for $0<k<M$ unless all $A_j$ are equal.
\end{theorem}
\begin{proof}
See Appendix~\ref{app:influence-selection}.
\end{proof}
\textbf{Interpretation.} The theorem ranks ties by influence, that is, by how hard each tie pushes and where the push points. The right side is what uniform selection achieves in expectation, so the theorem also measures what selection buys: $k$ times the gap between the mean influence of the top $k$ ties and that of the pool. Selection therefore pays off exactly when ties differ, which is the premise of Section~\ref{subsec:problem_selection}.

\textbf{Scope: first-order result.} The theorem holds to first order, and the $O(\epsilon^2)$ terms capture how ties interact. Two ties with similar spurious contrasts overlap, because training on one shrinks the margin of the other. Every selected tie also moves $\theta_0$ and, with it, the margins and curvature that define the remaining influences. The guarantee holds while the total tie weight is small relative to $N$. Training uses larger weights, so Section~\ref{sec:experiments} tests whether the first-order ranking still selects better ties.

\textbf{The cost of influence.} Computing $A_j$ requires three quantities: the inverse curvature $K_s^{-1}$, the energy gradient $\nabla_sE(\theta_0)$, and the gradient of each tie. For a neural policy, each is expensive, and $\nabla_sE$ requires knowing the spurious features themselves. Of the two factors in $A_j$, only the margin is available from a forward pass.
This raises the question of the next section: using only forward passes, which selector best approximates influence on average?

\section{Selection with Forward Passes: Margin}
\label{sec:margin}
Section~\ref{sec:influence} showed that ranking ties by influence minimizes spurious energy to first order, but computing influence requires the inverse curvature, the energy gradient, and a gradient for each tie. We now ask how well ties can be ranked using only forward passes. We introduce \emph{pure margin selection}, which ranks ties by the absolute margin that the DPO solution $\theta_0$ assigns to each of them. Under a Gaussian model of the spurious contrasts, it maximizes the expected first-order reduction in energy among selectors that see only the margins.
\vspace{-0.5em}
\subsection{What a Forward Pass Reveals}  \label{subsec:forward_pass_reveals}
A forward pass through the policy and the reference returns the log-probability of each response, and hence its implicit reward. On a tie, the difference between the two implicit rewards is the margin $m_j$ from Section~\ref{subsec:first_order_effect}. We orient each tie uniformly at random, which leaves its influence unchanged, so that $ m_j=r_{\theta_0}(x_j,y_{A,j})-r_{\theta_0}(x_j,y_{B,j}) =\beta\tilde\theta_{0,s}^\top d_j,$ where $y_{A,j}$ is the response that the sampled orientation places first. Forward passes thus deliver the margin of every tie in the pool at the cost of evaluating the two models.

The margin reveals one factor of the influence and hints at the other. From Proposition~\ref{prop:single-tie-influence},
$$
A_j=\frac{\beta}{2N}\tanh\!\Big(\frac{m_j}{2}\Big)\,a_j,
\qquad a_j=\big\langle\nabla_sE(\theta_0),K_s^{-1}d_j\big\rangle .
$$
The margin fixes the margin factor, while the alignment $a_j$ stays hidden because it requires the curvature and the energy gradient. Both, however, are linear in the same contrast $d_j$, so the margin carries information about the alignment. A selector that sees only the margins can do no better than rank ties by their expected influence given those margins, and the margin factor passes through this expectation unchanged. Everything therefore hinges on how the expected alignment depends on the margin.

\subsection{Main Result: Pure Margin Selection Is Optimal in Expectation}
\label{subsec:pm-optimality}
\textbf{Assumptions.} How the expected alignment depends on the margin is a question about how contrasts are distributed across the pool. The random orientation centers them, $\mathbb E[d_j]=0$. We condition on the strict training data, fixing the fitted checkpoint $\theta_0$ and its Hessian $H$, and regard the population moment $\Sigma_{ss}$ as fixed. For tractability, we assume that, under this conditioning, the candidate contrasts are independent across ties and satisfy $d_j\sim\mathcal N(0,c\Sigma_{ss})$, where $c>0$. The covariance models ties that inherit the spurious geometry of $P$: after whitening by $\Sigma_{ss}$, the contrasts are isotropic, as in \citet{moya2026spurious}. Under this model the alignment and the margin are jointly Gaussian, so the expected alignment grows linearly with the margin, with a positive slope. The following theorem characterizes the expected influence given a tie's margin, with expectations taken under the conditional model specified by our assumptions.

\begin{theorem}[Ranking ties from margins]
\label{thm:pm-optimality}
Under the assumptions above, with $\tilde\theta_{0,s}\neq0$, there is a constant $\lambda>0$, common to all ties, such that
$$
\mathbb E[A_j\mid m_j]=\frac{\lambda}{N}\,m_j\tanh\!\Big(\frac{m_j}{2}\Big).
$$
Thus, among selectors that see only the margins, selecting the $k$ ties with the largest $|m_j|$ maximizes the expected first-order reduction in energy.
\end{theorem}
\begin{proof}
    See Appendix~\ref{app:pm-optimality}.
\end{proof}

\textbf{Interpretation.} The expected influence $m_j\tanh(m_j/2)$ is even and increases with $|m_j|$, so the sign of the margin, which records only the sampled orientation, drops out of the ranking. We call the resulting selector \emph{pure margin selection}. It works for two reasons. A larger absolute margin produces a stronger push from the tie. Under the covariance assumption, it also predicts better alignment with the direction that lowers the energy. A tie with zero margin has neither, which makes precise the observation in Section~\ref{subsec:problem_selection} that such a tie changes nothing.

\textbf{Scope of the result.} The theorem concerns the expected influence given the margins. Two ties with the same margin can still differ in influence, because their contrasts can point in different directions, so pure margin selection need not recover the influence ranking of a particular pool. Section~\ref{sec:experiments} compares the two selectors directly.

\textbf{From theory to neural models.} Both derived selectors rest on the log-linear model assumption. In a neural policy, the margin stays exact and cheap, because it still requires only forward passes. Influence, by contrast, requires a proxy for the energy, whose spurious features are unobserved, and an estimate of the curvature. Section~\ref{sec:nonlinear} builds both.

\section{From Theory to Neural Models}
\label{sec:nonlinear}
Section~\ref{sec:margin} showed that pure margin selection maximizes the expected first-order reduction in energy under a Gaussian model of the spurious contrasts, while influence attains the first-order optimum without that model. We now implement both selectors in neural policies. Pure margin selection carries over directly, because it needs only forward passes. Influence needs two estimates: a \emph{counterfactual energy}, which
measures spurious reliance on a fixed probe of ties, and a damped
estimate of the curvature of the DPO loss on the strict comparisons. In practice, the DPO policy $\theta_0$ is the final checkpoint of DPO training, which satisfies the stationarity condition of Section~\ref{sec:influence} only approximately.
\vspace{-0.5em}
\subsection{Influence in Practice} \label{subsec:influence_practice}
\textbf{Counterfactual energy.} A neural policy exposes no spurious features, so we cannot compute the energy of Section~\ref{sec:influence} directly. We can still approximate it, because a tie isolates the spurious part of
the margin: its two responses differ only in spurious features, so its
margin is purely spurious. We therefore form a fixed probe $\mathcal P$ of $N_v$ ties, held out from selection, and define the counterfactual energy
\begin{equation}
\label{eq:counterfactual-energy}
E_{\rm cf}(\theta)=\frac{1}{N_v}\sum_{n=1}^{N_v}m_n(\theta)^2 ,
\end{equation}
where $m_n$ is the margin on probe tie $n$. Squaring removes the dependence on its orientation.

In the log-linear model, the counterfactual energy recovers the energy gradient up to a positive scale when the probe contrasts have the covariance assumed in Section~\ref{sec:margin} (see Appendix~\ref{app:probe-proofs}): $\mathbb E_{\mathcal P}[\nabla E_{\rm cf}(\theta_0)]=c\beta^2\nabla E(\theta_0).$ The scale leaves the ranking of ties unchanged, so the counterfactual energy supplies the gradient that influence needs.

\textbf{Curvature.} The remaining ingredient is the curvature. In a neural policy, influence keeps the form of Proposition~\ref{prop:single-tie-influence} but runs over all parameters: the full inverse curvature accounts for how every parameter responds to a tie, the role that $K_s$ played for the spurious weights. We estimate the curvature with a damped generalized Gauss--Newton (GGN) matrix $\widehat H$~\citep{schraudolph2002fast,martens2020natural}, computed on a subset of the strict comparisons in $\mathcal D_\mathrm{pref}$, and score each tie by
$$
\widehat A_j=\frac1N\big\langle\nabla E_{\rm cf}(\theta_0),
\widehat H^{-1}g_j(\theta_0)\big\rangle .
$$
Every estimated influence shares the same energy gradient, so we apply the
inverse curvature to it once, $v=\widehat H^{-1}\nabla E_{\rm cf}(\theta_0)$,
and reuse the result across the pool. Estimating the influence of a tie then takes only its gradient, $\widehat A_j=v^\top g_j(\theta_0)/N$. These approximations make influence practical but do not preserve the log-linear guarantee. We provide more details in Appendix~\ref{app:scoring-details}.
\vspace{-0.5em}
\subsection{Margin in Practice}
\label{subsec:margin_practice}
Pure margin selection ranks ties by the absolute margin $|m_j|$, which forward passes of the DPO policy and the reference deliver for every tie in the pool. Because margins are all it uses, the selector skips everything that influence requires: the probe, the curvature, and the tie gradients. It never even computes the energy, yet under the assumptions of Section~\ref{sec:margin} its ranking maximizes the expected first-order reduction of that energy.

\textbf{Two selectors, two costs.}
The two selectors thus trade cost for information. Influence uses tie gradients, the curvature, and a probe to approximate the first-order optimum. Pure margin selection uses only margins and is optimal in expectation under a Gaussian model of the contrasts. Section~\ref{sec:experiments} measures what each buys over uniform selection: whether the first-order ranking helps at the tie weights used in training, and whether pure margin selection loses anything relative to influence.

\section{Experiments}
\label{sec:experiments}
We test three theoretical predictions: (i) influence predicts the effect of an isolated tie (Proposition~\ref{prop:single-tie-influence}), (ii) mean influence conditioned on the tie margin follows the predicted curve (Theorem~\ref{thm:pm-optimality}), and (iii) the highest-scoring ties reduce spurious reliance faster than uniformly sampled ties (Theorem~\ref{thm:influence-selection}).
We begin with log-linear models and then ask whether the selection gains persist in two settings closer to practice: neural networks and language models.
In every regime, spurious features correlate with preferences under the training distribution $P$, and the adversarial deployment distribution $Q_{\rm adv}$ reverses these correlations. We score candidate ties using a checkpoint trained on strict comparisons alone. For each selected set, we train a new model on the strict comparisons and selected ties, starting from the same initialization as the strict checkpoint. Appendix~\ref{app:experimental-protocol} gives the full protocol and compares this procedure with continued training from the strict checkpoint. The code is available at \url{https://github.com/purav-0000/selective-tie-learning}.
\vspace{-0.5em}
\subsection{Log-Linear: Quantitative Validation}
\textbf{Setup.}
We generate Gaussian features $\phi=[\phi_c;\phi_s]$ with $d_c=d_s=10$.
The operator $T$ keeps the strict winner's core features $\phi_c$ and replaces its spurious features $\phi_s$ with those of a response drawn uniformly from $\mathcal D_{\rm pref}$.
We compare uniform tie learning (UTL), pure margin selection (PM), and Oracle influence, which computes $A_j$ in Proposition~\ref{prop:single-tie-influence} exactly from the population spurious-energy gradient and the empirical DPO Hessian.
Appendix~\ref{app:log-linear-experiment} gives the full details.

\textbf{Validating the theory.} Panel~(a) tests prediction~(i): for each tie, we add its symmetric packet with weight $1/N$ ($\epsilon=1$ in Proposition~\ref{prop:single-tie-influence}), re-optimize the model, and compare the observed reduction $E(\theta_0)-E(\theta_j^1)$ with $A_j$.
Panel~(b) tests prediction~(ii): we repeat this with each tie in its fixed training orientation, group ties into ten equal-frequency bins by absolute margin, and compare each bin's mean observed reduction with the conditional mean in Theorem~\ref{thm:pm-optimality}.

\textbf{Results.}
Figure~\ref{fig:linear-experiments}~(a)--(b) supports predictions~(i) and~(ii).
Across all isolated ties, predicted influence has $4.2\%$ relative $\ell_2$ error and $99.8\%$ sign agreement with the observed reduction in spurious energy, including every tie predicted to increase it.
In panel~(b), the theoretical conditional mean lies within the 95\% confidence interval of every margin bin.

Panels~(c)--(d) show that the first-order ranking in prediction~(iii) remains useful at full tie weight.
PM and Oracle influence raise accuracy under $Q_{\rm adv}$ and lower spurious energy faster than UTL: with 100 ties, PM reaches the accuracy that UTL reaches with about 300.
Their advantage narrows as the selected set approaches the full pool, where all methods coincide.
PM closely tracks Oracle influence even though it uses only forward passes.

\begin{figure}[t]
\centering
\begin{tikzpicture}
\begin{groupplot}[
    group style={group size=4 by 1, horizontal sep=7.7mm},
    stl axis,
    scale only axis,
    width=0.194\linewidth,
    height=0.185\linewidth,
]

\nextgroupplot[
    title={\textbf{(a)} Single-tie effect\\Observed $(10^{-2})$},
    xmin=-3, xmax=65,
    ymin=-3, ymax=65,
    xtick={0,20,40,60},
    ytick={0,20,40,60},
    xlabel={Predicted $(10^{-2})$},
]
\addplot[stl identity, domain=-3:65, samples=2] {x};
\addplot[
    only marks,
    color=stlorange,
    mark=*,
    mark size=0.65pt,
    opacity=0.22,
] table[
    x expr=100*\thisrow{predicted_reduction},
    y expr=100*\thisrow{observed_reduction},
    col sep=comma,
] {figures/data/linear/theory/influence.csv};

\nextgroupplot[
    title={\textbf{(b)} Binned effect\\Mean reduction $(10^{-2})$},
    xmin=0, xmax=4.5,
    ymin=0, ymax=36,
    xtick={0,2,4},
    ytick={0,10,20,30},
    xlabel={Absolute margin $|m_j|$},
    legend pos=north west,
    legend style={font=\fontsize{8}{9}\selectfont, fill=white, fill opacity=0.9, text opacity=1},
]
\addplot[
    name path=lin-oracle-low,
    forget plot, draw=none,
    restrict expr to domain={\coordindex}{0:9},
] table[x=margin_mean, y expr=100*\thisrow{observed_lower}, col sep=comma]
    {figures/data/linear/theory/expected_influence.csv};
\addplot[
    name path=lin-oracle-high,
    forget plot, draw=none,
    restrict expr to domain={\coordindex}{0:9},
] table[x=margin_mean, y expr=100*\thisrow{observed_upper}, col sep=comma]
    {figures/data/linear/theory/expected_influence.csv};
\addplot[stlorange, fill opacity=0.13, forget plot, draw=none]
    fill between[of=lin-oracle-low and lin-oracle-high];
\addplot[
    name path=lin-gaussian-low,
    forget plot, draw=none,
    restrict expr to domain={\coordindex}{0:9},
] table[x=margin_mean, y expr=100*\thisrow{gaussian_lower}, col sep=comma]
    {figures/data/linear/theory/expected_influence.csv};
\addplot[
    name path=lin-gaussian-high,
    forget plot, draw=none,
    restrict expr to domain={\coordindex}{0:9},
] table[x=margin_mean, y expr=100*\thisrow{gaussian_upper}, col sep=comma]
    {figures/data/linear/theory/expected_influence.csv};
\addplot[stlgray, fill opacity=0.09, forget plot, draw=none]
    fill between[of=lin-gaussian-low and lin-gaussian-high];
\addplot[
    stl influence,
    only marks,
    restrict expr to domain={\coordindex}{0:9},
] table[x=margin_mean, y expr=100*\thisrow{observed_mean}, col sep=comma]
    {figures/data/linear/theory/expected_influence.csv};
\addlegendentry{Empirical}
\addplot[
    stl identity,
    color=stlgray,
    restrict expr to domain={\coordindex}{0:9},
] table[x=margin_mean, y expr=100*\thisrow{gaussian_mean}, col sep=comma]
    {figures/data/linear/theory/expected_influence.csv};
\addlegendentry{Theory}

\nextgroupplot[
    title={\textbf{(c)} Accuracy\\Under $Q_{\rm adv}$ (\%)},
    xmin=15, xmax=510,
    ymin=45, ymax=82,
    xtick={100,300,500},
    minor xtick={25,50,75,150,200,250,350,400,450},
    ytick={50,60,70,80},
    xlabel={Tie budget $k$},
    legend to name=linear-budget-legend,
    legend columns=4,
    legend style={column sep=5pt},
]
\addplot[name path=lin-acc-utl-low, forget plot, draw=none,
    restrict expr to domain={\coordindex}{0:11}]
    table[x=budget,y=lower,col sep=comma]{figures/data/linear/fresh/q_adv.csv};
\addplot[name path=lin-acc-utl-high, forget plot, draw=none,
    restrict expr to domain={\coordindex}{0:11}]
    table[x=budget,y=upper,col sep=comma]{figures/data/linear/fresh/q_adv.csv};
\addplot[black,fill opacity=0.055,forget plot,draw=none]
    fill between[of=lin-acc-utl-low and lin-acc-utl-high];
\addplot[name path=lin-acc-pm-low, forget plot, draw=none,
    restrict expr to domain={\coordindex}{12:23}]
    table[x=budget,y=lower,col sep=comma]{figures/data/linear/fresh/q_adv.csv};
\addplot[name path=lin-acc-pm-high, forget plot, draw=none,
    restrict expr to domain={\coordindex}{12:23}]
    table[x=budget,y=upper,col sep=comma]{figures/data/linear/fresh/q_adv.csv};
\addplot[stlblue,fill opacity=0.11,forget plot, draw=none]
    fill between[of=lin-acc-pm-low and lin-acc-pm-high];
\addplot[name path=lin-acc-oracle-low, forget plot, draw=none,
    restrict expr to domain={\coordindex}{24:35}]
    table[x=budget,y=lower,col sep=comma]{figures/data/linear/fresh/q_adv.csv};
\addplot[name path=lin-acc-oracle-high, forget plot, draw=none,
    restrict expr to domain={\coordindex}{24:35}]
    table[x=budget,y=upper,col sep=comma]{figures/data/linear/fresh/q_adv.csv};
\addplot[stlorange,fill opacity=0.11,forget plot, draw=none]
    fill between[of=lin-acc-oracle-low and lin-acc-oracle-high];
\addplot[stl utl, restrict expr to domain={\coordindex}{0:11}]
    table[x=budget,y=mean,col sep=comma]{figures/data/linear/fresh/q_adv.csv};
\addlegendentry{UTL}
\addplot[stl pm, restrict expr to domain={\coordindex}{12:23}]
    table[x=budget,y=mean,col sep=comma]{figures/data/linear/fresh/q_adv.csv};
\addlegendentry{PM}
\addplot[stl influence, restrict expr to domain={\coordindex}{24:35}]
    table[x=budget,y=mean,col sep=comma]{figures/data/linear/fresh/q_adv.csv};
\addlegendentry{Oracle influence}
\addplot[stl strict, restrict expr to domain={\coordindex}{36:47}]
    table[x=budget,y=mean,col sep=comma]{figures/data/linear/fresh/q_adv.csv};
\addlegendentry{Strict only}

\nextgroupplot[
    title={\textbf{(d)} Spurious energy\\Lower is better},
    xmin=15, xmax=510,
    ymin=0, ymax=7,
    xtick={100,300,500},
    minor xtick={25,50,75,150,200,250,350,400,450},
    ytick={0,2,4,6},
    xlabel={Tie budget $k$},
]
\addplot[name path=lin-energy-utl-low, forget plot, draw=none,
    restrict expr to domain={\coordindex}{0:11}]
    table[x=budget,y=lower,col sep=comma]{figures/data/linear/fresh/spurious_energy.csv};
\addplot[name path=lin-energy-utl-high, forget plot, draw=none,
    restrict expr to domain={\coordindex}{0:11}]
    table[x=budget,y=upper,col sep=comma]{figures/data/linear/fresh/spurious_energy.csv};
\addplot[black,fill opacity=0.055,forget plot,draw=none]
    fill between[of=lin-energy-utl-low and lin-energy-utl-high];
\addplot[name path=lin-energy-pm-low, forget plot, draw=none,
    restrict expr to domain={\coordindex}{12:23}]
    table[x=budget,y=lower,col sep=comma]{figures/data/linear/fresh/spurious_energy.csv};
\addplot[name path=lin-energy-pm-high, forget plot, draw=none,
    restrict expr to domain={\coordindex}{12:23}]
    table[x=budget,y=upper,col sep=comma]{figures/data/linear/fresh/spurious_energy.csv};
\addplot[stlblue,fill opacity=0.11,forget plot, draw=none]
    fill between[of=lin-energy-pm-low and lin-energy-pm-high];
\addplot[name path=lin-energy-oracle-low, forget plot, draw=none,
    restrict expr to domain={\coordindex}{24:35}]
    table[x=budget,y=lower,col sep=comma]{figures/data/linear/fresh/spurious_energy.csv};
\addplot[name path=lin-energy-oracle-high, forget plot, draw=none,
    restrict expr to domain={\coordindex}{24:35}]
    table[x=budget,y=upper,col sep=comma]{figures/data/linear/fresh/spurious_energy.csv};
\addplot[stlorange,fill opacity=0.11,forget plot, draw=none]
    fill between[of=lin-energy-oracle-low and lin-energy-oracle-high];
\addplot[stl utl, restrict expr to domain={\coordindex}{0:11}]
    table[x=budget,y=mean,col sep=comma]{figures/data/linear/fresh/spurious_energy.csv};
\addplot[stl pm, restrict expr to domain={\coordindex}{12:23}]
    table[x=budget,y=mean,col sep=comma]{figures/data/linear/fresh/spurious_energy.csv};
\addplot[stl influence, restrict expr to domain={\coordindex}{24:35}]
    table[x=budget,y=mean,col sep=comma]{figures/data/linear/fresh/spurious_energy.csv};
\addplot[stl strict, restrict expr to domain={\coordindex}{36:47}]
    table[x=budget,y=mean,col sep=comma]{figures/data/linear/fresh/spurious_energy.csv};

\end{groupplot}
\path (group c1r1.south west) -- (group c4r1.south east)
    node[midway, below=10mm] {\pgfplotslegendfromname{linear-budget-legend}};
\end{tikzpicture}
\caption{\textbf{Log-linear validation.}
\textbf{(a)} Predicted versus observed reductions in population spurious energy
after perturbing the objective with one symmetric tie.
\textbf{(b)} Observed reductions from ties with sampled orientations and the theoretical
conditional mean across margin bins containing equal numbers of ties
(Theorem~\ref{thm:pm-optimality}).
\textbf{(c--d)} Accuracy under $Q_{\rm adv}$ and population spurious energy after
training on selected ties. Curves are means; shading shows 95\% Student-$t$
confidence intervals.}
\label{fig:linear-experiments}
\end{figure}
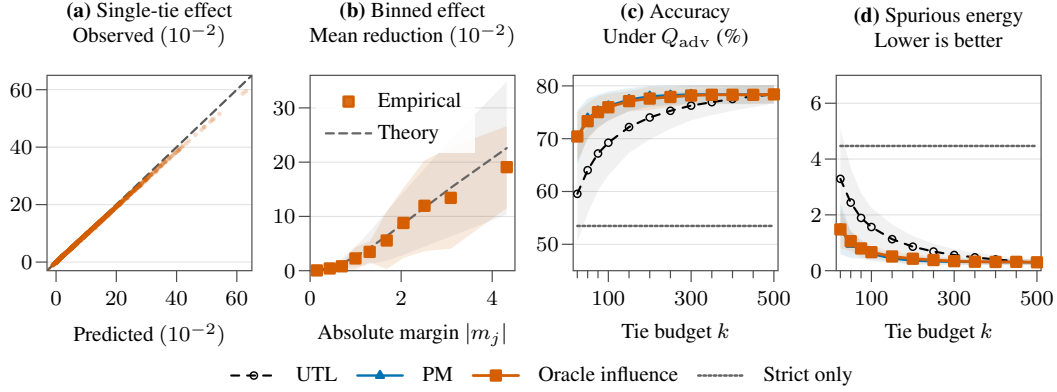

\subsection{Neural Networks: Qualitative Persistence}

The log-linear theory no longer applies once a nonlinear representation hides the core--spurious decomposition.
We next ask whether its selection gains persist in a neural reward model.

\textbf{Setup.}
We keep the log-linear data-generating process and tie construction, pass the features through a fixed random nonlinear mixer, and train an MLP reward model with DPO on the mixed features.
We apply $T$ before mixing the features, so each tie still preserves core content exactly.
We compare UTL and PM with Practical influence, which replaces the population spurious-energy gradient with the gradient of the counterfactual energy $E_{\rm cf}$ and the Hessian with a damped GGN (Section~\ref{sec:nonlinear}).
Appendix~\ref{app:nonlinear-experiment} gives the full details.

\textbf{Results.}
PM and Practical influence improve accuracy under $Q_{\rm adv}$ over UTL at every budget (Figure~\ref{fig:neural-experiments}~(a)).
At 500 ties, both gain about $7.5$ percentage points over UTL and exceed the accuracy that UTL reaches with $1{,}500$ ties.
Their advantage first grows and then narrows as the selected set approaches the full pool.
As in the log-linear model, PM closely tracks influence.
\vspace{-0.5em}
\subsection{LLMs: Controlled Preference Tasks}
Language models introduce a pretrained policy and text responses.
We test selection on two tasks whose spurious attributes we can control, so that the shifts and ties can be constructed exactly.

\textbf{Setup.}
We fine-tune LoRA adapters of Llama-3.2-1B-Instruct with DPO, using the base model as the reference.
In the hotel task, the model recommends one of two generated hotel profiles.
Utility depends on price, distance, star rating, and a fixed noise term for each hotel, while seven other fields, such as building age and chain tier, correlate with utility positively in $P$ and negatively in $Q_{\rm adv}$.
The operator $T$ replaces these seven fields in the winner's profile with those of a hotel drawn from $\mathcal D_{\rm pref}$.
In the emoji task, the safe response of a PKU-SafeRLHF
pair~\citep{ji2025pkusaferlhf} is always preferred, and we insert different
numbers of one emoji into its two responses.
The safe response usually receives more emojis in $P$ and fewer in $Q_{\rm adv}$.
Here, $T$ removes every emoji from the safe winner.
We compare UTL, PM, and Practical influence.
Appendix~\ref{app:llm-experiments} gives the full details.

\textbf{Results.}
PM and Practical influence improve mean accuracy under $Q_{\rm adv}$ over UTL at every partial budget on both tasks (Figure~\ref{fig:neural-experiments}~(b)--(c)).
With 500 ties, PM reaches $92.1\%$ accuracy on hotels and $68.4\%$ on emojis, compared with $74.6\%$ and $59.3\%$ for UTL.
PM needs only 300 ties on hotels and 500 on emojis to exceed the accuracy that UTL reaches with 750.
The advantage again narrows as the selected set approaches the full pool.
The two scores track each other on hotels, whereas PM stays above Practical influence on emojis.

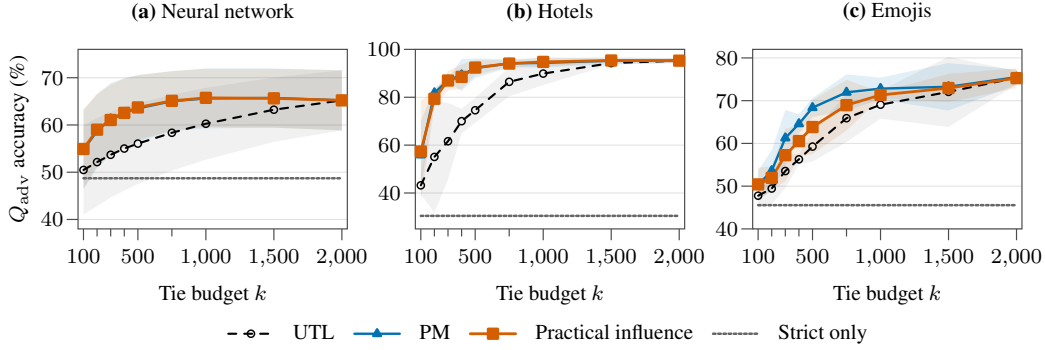
\begin{figure}[t]
\centering
\begin{tikzpicture}
\begin{groupplot}[
    group style={group size=3 by 1, horizontal sep=9mm},
    stl axis,
    scale only axis,
    width=0.255\linewidth,
    height=0.17\linewidth,
]

\nextgroupplot[
    title={\textbf{(a)} Neural network},
    xmin=60, xmax=2040,
    ymin=38, ymax=76,
    xtick={100,500,1000,1500,2000},
    minor xtick={200,300,400,750},
    ytick={40,50,60,70},
    xlabel={Tie budget $k$},
    ylabel={$Q_{\rm adv}$ accuracy (\%)},
    legend to name=neural-budget-legend,
    legend columns=4,
    legend style={column sep=5pt},
]
\addplot[name path=nn-utl-low,forget plot,draw=none,
    restrict expr to domain={\coordindex}{0:8}]
    table[x=budget,y=lower,col sep=comma]{figures/data/nonlinear/fresh/q_adv.csv};
\addplot[name path=nn-utl-high,forget plot,draw=none,
    restrict expr to domain={\coordindex}{0:8}]
    table[x=budget,y=upper,col sep=comma]{figures/data/nonlinear/fresh/q_adv.csv};
\addplot[black,fill opacity=0.055,forget plot,draw=none]
    fill between[of=nn-utl-low and nn-utl-high];
\addplot[name path=nn-pm-low,forget plot,draw=none,
    restrict expr to domain={\coordindex}{9:17}]
    table[x=budget,y=lower,col sep=comma]{figures/data/nonlinear/fresh/q_adv.csv};
\addplot[name path=nn-pm-high,forget plot,draw=none,
    restrict expr to domain={\coordindex}{9:17}]
    table[x=budget,y=upper,col sep=comma]{figures/data/nonlinear/fresh/q_adv.csv};
\addplot[stlblue,fill opacity=0.11,forget plot,draw=none]
    fill between[of=nn-pm-low and nn-pm-high];
\addplot[name path=nn-inf-low,forget plot,draw=none,
    restrict expr to domain={\coordindex}{18:26}]
    table[x=budget,y=lower,col sep=comma]{figures/data/nonlinear/fresh/q_adv.csv};
\addplot[name path=nn-inf-high,forget plot,draw=none,
    restrict expr to domain={\coordindex}{18:26}]
    table[x=budget,y=upper,col sep=comma]{figures/data/nonlinear/fresh/q_adv.csv};
\addplot[stlorange,fill opacity=0.11,forget plot,draw=none]
    fill between[of=nn-inf-low and nn-inf-high];
\addplot[stl utl,restrict expr to domain={\coordindex}{0:8}]
    table[x=budget,y=mean,col sep=comma]{figures/data/nonlinear/fresh/q_adv.csv};
\addlegendentry{UTL}
\addplot[stl pm,restrict expr to domain={\coordindex}{9:17}]
    table[x=budget,y=mean,col sep=comma]{figures/data/nonlinear/fresh/q_adv.csv};
\addlegendentry{PM}
\addplot[stl influence,restrict expr to domain={\coordindex}{18:26}]
    table[x=budget,y=mean,col sep=comma]{figures/data/nonlinear/fresh/q_adv.csv};
\addlegendentry{Practical influence}
\addplot[stl strict,restrict expr to domain={\coordindex}{27:35}]
    table[x=budget,y=mean,col sep=comma]{figures/data/nonlinear/fresh/q_adv.csv};
\addlegendentry{Strict only}

\nextgroupplot[
    title={\textbf{(b)} Hotels},
    xmin=60, xmax=2040,
    ymin=25, ymax=100,
    xtick={100,500,1000,1500,2000},
    minor xtick={200,300,400,750},
    ytick={40,60,80,100},
    xlabel={Tie budget $k$},
]
\addplot[name path=hotel-utl-low,forget plot,draw=none,
    restrict expr to domain={\coordindex}{0:8}]
    table[x=budget,y=lower,col sep=comma]{figures/data/llm/fresh/hotels_q_adv.csv};
\addplot[name path=hotel-utl-high,forget plot,draw=none,
    restrict expr to domain={\coordindex}{0:8}]
    table[x=budget,y=upper,col sep=comma]{figures/data/llm/fresh/hotels_q_adv.csv};
\addplot[black,fill opacity=0.055,forget plot,draw=none]
    fill between[of=hotel-utl-low and hotel-utl-high];
\addplot[name path=hotel-pm-low,forget plot,draw=none,
    restrict expr to domain={\coordindex}{9:17}]
    table[x=budget,y=lower,col sep=comma]{figures/data/llm/fresh/hotels_q_adv.csv};
\addplot[name path=hotel-pm-high,forget plot,draw=none,
    restrict expr to domain={\coordindex}{9:17}]
    table[x=budget,y=upper,col sep=comma]{figures/data/llm/fresh/hotels_q_adv.csv};
\addplot[stlblue,fill opacity=0.11,forget plot,draw=none]
    fill between[of=hotel-pm-low and hotel-pm-high];
\addplot[name path=hotel-inf-low,forget plot,draw=none,
    restrict expr to domain={\coordindex}{18:26}]
    table[x=budget,y=lower,col sep=comma]{figures/data/llm/fresh/hotels_q_adv.csv};
\addplot[name path=hotel-inf-high,forget plot,draw=none,
    restrict expr to domain={\coordindex}{18:26}]
    table[x=budget,y=upper,col sep=comma]{figures/data/llm/fresh/hotels_q_adv.csv};
\addplot[stlorange,fill opacity=0.11,forget plot,draw=none]
    fill between[of=hotel-inf-low and hotel-inf-high];
\addplot[stl utl,restrict expr to domain={\coordindex}{0:8}]
    table[x=budget,y=mean,col sep=comma]{figures/data/llm/fresh/hotels_q_adv.csv};
\addplot[stl pm,restrict expr to domain={\coordindex}{9:17}]
    table[x=budget,y=mean,col sep=comma]{figures/data/llm/fresh/hotels_q_adv.csv};
\addplot[stl influence,restrict expr to domain={\coordindex}{18:26}]
    table[x=budget,y=mean,col sep=comma]{figures/data/llm/fresh/hotels_q_adv.csv};
\addplot[stl strict,restrict expr to domain={\coordindex}{27:35}]
    table[x=budget,y=mean,col sep=comma]{figures/data/llm/fresh/hotels_q_adv.csv};

\nextgroupplot[
    title={\textbf{(c)} Emojis},
    xmin=60, xmax=2040,
    ymin=40, ymax=82,
    xtick={100,500,1000,1500,2000},
    minor xtick={200,300,400,750},
    ytick={40,50,60,70,80},
    xlabel={Tie budget $k$},
]
\addplot[name path=emoji-utl-low,forget plot,draw=none,
    restrict expr to domain={\coordindex}{0:8}]
    table[x=budget,y=lower,col sep=comma]{figures/data/llm/fresh/emoji_q_adv.csv};
\addplot[name path=emoji-utl-high,forget plot,draw=none,
    restrict expr to domain={\coordindex}{0:8}]
    table[x=budget,y=upper,col sep=comma]{figures/data/llm/fresh/emoji_q_adv.csv};
\addplot[black,fill opacity=0.055,forget plot,draw=none]
    fill between[of=emoji-utl-low and emoji-utl-high];
\addplot[name path=emoji-pm-low,forget plot,draw=none,
    restrict expr to domain={\coordindex}{9:17}]
    table[x=budget,y=lower,col sep=comma]{figures/data/llm/fresh/emoji_q_adv.csv};
\addplot[name path=emoji-pm-high,forget plot,draw=none,
    restrict expr to domain={\coordindex}{9:17}]
    table[x=budget,y=upper,col sep=comma]{figures/data/llm/fresh/emoji_q_adv.csv};
\addplot[stlblue,fill opacity=0.11,forget plot,draw=none]
    fill between[of=emoji-pm-low and emoji-pm-high];
\addplot[name path=emoji-inf-low,forget plot,draw=none,
    restrict expr to domain={\coordindex}{18:26}]
    table[x=budget,y=lower,col sep=comma]{figures/data/llm/fresh/emoji_q_adv.csv};
\addplot[name path=emoji-inf-high,forget plot,draw=none,
    restrict expr to domain={\coordindex}{18:26}]
    table[x=budget,y=upper,col sep=comma]{figures/data/llm/fresh/emoji_q_adv.csv};
\addplot[stlorange,fill opacity=0.11,forget plot,draw=none]
    fill between[of=emoji-inf-low and emoji-inf-high];
\addplot[stl utl,restrict expr to domain={\coordindex}{0:8}]
    table[x=budget,y=mean,col sep=comma]{figures/data/llm/fresh/emoji_q_adv.csv};
\addplot[stl pm,restrict expr to domain={\coordindex}{9:17}]
    table[x=budget,y=mean,col sep=comma]{figures/data/llm/fresh/emoji_q_adv.csv};
\addplot[stl influence,restrict expr to domain={\coordindex}{18:26}]
    table[x=budget,y=mean,col sep=comma]{figures/data/llm/fresh/emoji_q_adv.csv};
\addplot[stl strict,restrict expr to domain={\coordindex}{27:35}]
    table[x=budget,y=mean,col sep=comma]{figures/data/llm/fresh/emoji_q_adv.csv};

\end{groupplot}
\path (group c1r1.south west) -- (group c3r1.south east)
    node[midway, below=10mm] {\pgfplotslegendfromname{neural-budget-legend}};
\end{tikzpicture}
\caption{\textbf{Selection beyond the log-linear setting.}
$Q_{\rm adv}$ accuracy for a nonlinear reward model \textbf{(a)} and
LoRA-DPO on hotel \textbf{(b)} and emoji \textbf{(c)} tasks. Practical
influence uses a sampled GGN and solves the damped system in its row space for
all three neural regimes. Curves are means across seeds; shading shows 95\%
Student-$t$ confidence intervals.}
\label{fig:neural-experiments}
\end{figure}

\section{Related Work}
\label{sec:related-work}
\textbf{Spurious correlations in preference learning.}
Preference learning can reward attributes that predict labels without tracking response quality.  Reward models often favor longer answers even when they are not more helpful~\citep{DBLP:journals/corr/abs-2310-03716}, and sycophancy persists in models trained with human feedback~\citep{DBLP:conf/iclr/SharmaTKDABDHJK24}.  Studies of reward-model exploitation and overoptimization reveal related shortcuts, including verbosity~\citep{DBLP:conf/nips/DuboisLTZGBGLH23,DBLP:journals/tmlr/CasperDSGSRFKLF23,DBLP:conf/icml/GaoSH23}. These works document concrete instances of spurious learning; our analysis characterizes how individual ties change reliance on spurious features.

\textbf{Mitigating spurious preferences.}
Existing methods intervene through the objective, reward model, internal representations, or training data.  Some penalize or normalize length in the preference objective~\citep{DBLP:conf/acl/ParkREF24,DBLP:conf/emnlp/Lu0AZ0YS24}, separate length from quality in reward models~\citep{DBLP:conf/icml/Chen0CS0GHSC24}, or steer representations at inference~\citep{DBLP:conf/acl/RimskyGSTHT24}.  RRM~\citep{DBLP:conf/iclr/000200CWJGS0YSM25} and counterfactual augmentation for length bias~\citep{DBLP:journals/corr/abs-2511-12573} are closest to our setting: they construct neutral or tie-like comparisons and train on many of them.  STL instead selects from a fixed pool of ties under a training budget.

\textbf{Theory and selection.}
Linear and log-linear models explain how the geometry of the training distribution produces unstable
predictors~\citep{DBLP:conf/iclr/NagarajanAN21,DBLP:conf/icml/BombariM25,DBLP:conf/nips/VeitchDYE21,arjovsky2019invariant}. For DPO, \citet{moya2026spurious} trace population-level reliance to mean spurious bias and core--spurious correlation, and show that tie training suppresses it in aggregate.  Active preference methods also select comparisons, but usually to reduce annotation cost or uncertainty, improve exploration, or filter difficult strict examples~\citep{DBLP:conf/icml/MuldrewHZB24,lin2026activedpo,DBLP:conf/icml/GaoL0XZ025,DBLP:conf/pkdd/DasCPC25,DBLP:conf/nips/MeloTAG24}. Patterning uses susceptibilities to reweight preference pairs toward desired changes in reward-model behavior~\citep{wang2026patterning}. STL shares this use of first-order sensitivity to guide data interventions, selecting ties by their influence on spurious reliance.

\section{Conclusion}
\label{sec:conclusion}
We introduced Selective Tie Learning, a principled method for selecting a small subset of ties. For log-linear models, we derived the first-order influence of each tie on spurious reliance and showed that selecting the largest predicted reductions is first-order optimal. We also developed pure margin selection, an inexpensive rule that needs only forward passes, and showed that under a Gaussian model of the ties, it maximizes the expected reduction in spurious reliance among such rules. For neural models, we estimated tie influence without an explicit core--spurious decomposition. Across log-linear models, neural networks, and language models, selected ties yielded higher accuracy than uniformly sampled ties when spurious correlations were reversed, while using fewer ties.

Our study assumes an existing pool of ties. It does not address how to construct reliable ties or how to choose the budget $k$, although both decisions affect practical performance. Our experiments use fully or partially controlled settings so that we can measure spurious reliance and distribution shift directly. Extending STL to natural preference data will therefore require transformations that preserve utility, a principled choice of $k$, and scoring that scales to large pools. Integrating these components into a complete training pipeline remains open. Appendix~\ref{app:limitations} discusses these limitations in more detail.

\bibliographystyle{iclr2027_conference}
\bibliography{references}

@inproceedings{ouyang2022training,
  title={Training language models to follow instructions with human feedback},
  author={Ouyang, Long and Wu, Jeffrey and Jiang, Xu and Almeida, Diogo and Wainwright, Carroll and Mishkin, Pamela and Zhang, Chong and Agarwal, Sandhini and Slama, Katarina and Ray, Alex and others},
  booktitle={Advances in Neural Information Processing Systems 35},
  pages={27730--27744},
  year={2022},
}

@article{bai2022training,
  title={Training a helpful and harmless assistant with reinforcement learning from human feedback},
  author={Bai, Yuntao and Jones, Andy and Ndousse, Kamal and Askell, Amanda and Chen, Anna and DasSarma, Nova and Drain, Dawn and Fort, Stanislav and Ganguli, Deep and Henighan, Tom and others},
  journal={arXiv preprint arXiv:2204.05862},
  year={2022}
}

@article{DBLP:journals/corr/abs-2310-03716,
  author       = {Prasann Singhal and
                  Tanya Goyal and
                  Jiacheng Xu and
                  Greg Durrett},
  title        = {A Long Way to Go: Investigating Length Correlations in {RLHF}},
  journal      = {CoRR},
  volume       = {abs/2310.03716},
  year         = {2023},
  eprinttype    = {arXiv},
  eprint       = {2310.03716},
}

@article{saito2023verbosity,
  title={Verbosity bias in preference labeling by large language models},
  author={Saito, Keita and Wachi, Akifumi and Wataoka, Koki and Akimoto, Youhei},
  journal={arXiv preprint arXiv:2310.10076},
  year={2023}
}

@article{arjovsky2019invariant,
  title={Invariant risk minimization},
  author={Arjovsky, Martin and Bottou, L{\'e}on and Gulrajani, Ishaan and Lopez-Paz, David},
  journal={arXiv preprint arXiv:1907.02893},
  year={2019}
}

@inproceedings{DBLP:conf/acl/ParkREF24,
  author       = {Ryan Park and
                  Rafael Rafailov and
                  Stefano Ermon and
                  Chelsea Finn},
  editor       = {Lun{-}Wei Ku and
                  Andre Martins and
                  Vivek Srikumar},
  title        = {Disentangling Length from Quality in Direct Preference Optimization},
  booktitle    = {Findings of the Association for Computational Linguistics, {ACL} 2024,
                  Bangkok, Thailand and virtual meeting, August 11-16, 2024},
  series       = {Findings of {ACL}},
  volume       = {{ACL} 2024},
  pages        = {4998--5017},
  publisher    = {Association for Computational Linguistics},
  year         = {2024},
}

@article{DBLP:journals/corr/abs-2511-12573,
  author       = {Hyeonji Kim and
                  Sujeong Oh and
                  Sanghack Lee},
  title        = {Mitigating Length Bias in {RLHF} through a Causal Lens},
  journal      = {CoRR},
  volume       = {abs/2511.12573},
  year         = {2025},
  eprinttype    = {arXiv},
  eprint       = {2511.12573},
}

@inproceedings{DBLP:conf/icml/Chen0CS0GHSC24,
  author       = {Lichang Chen and
                  Chen Zhu and
                  Jiuhai Chen and
                  Davit Soselia and
                  Tianyi Zhou and
                  Tom Goldstein and
                  Heng Huang and
                  Mohammad Shoeybi and
                  Bryan Catanzaro},
  editor       = {Ruslan Salakhutdinov and
                  Zico Kolter and
                  Katherine A. Heller and
                  Adrian Weller and
                  Nuria Oliver and
                  Jonathan Scarlett and
                  Felix Berkenkamp},
  title        = {{ODIN:} Disentangled Reward Mitigates Hacking in {RLHF}},
  booktitle    = {Forty-first International Conference on Machine Learning, {ICML} 2024,
                  Vienna, Austria, July 21-27, 2024},
  series       = {Proceedings of Machine Learning Research},
  volume       = {235},
  pages        = {7935--7952},
  publisher    = {{PMLR} / OpenReview.net},
  year         = {2024},
}

@inproceedings{DBLP:conf/iclr/000200CWJGS0YSM25,
  author       = {Tianqi Liu and
                  Wei Xiong and
                  Jie Ren and
                  Lichang Chen and
                  Junru Wu and
                  Rishabh Joshi and
                  Yang Gao and
                  Jiaming Shen and
                  Zhen Qin and
                  Tianhe Yu and
                  Daniel Sohn and
                  Anastasia Makarova and
                  Jeremiah Zhe Liu and
                  Yuan Liu and
                  Bilal Piot and
                  Abe Ittycheriah and
                  Aviral Kumar and
                  Mohammad Saleh},
  title        = {{RRM:} Robust Reward Model Training Mitigates Reward Hacking},
  booktitle    = {The Thirteenth International Conference on Learning Representations,
                  {ICLR} 2025, Singapore, April 24-28, 2025},
  publisher    = {OpenReview.net},
  year         = {2025},
}

@inproceedings{DBLP:conf/emnlp/Lu0AZ0YS24,
  author       = {Junru Lu and
                  Jiazheng Li and
                  Siyu An and
                  Meng Zhao and
                  Yulan He and
                  Di Yin and
                  Xing Sun},
  editor       = {Yaser Al{-}Onaizan and
                  Mohit Bansal and
                  Yun{-}Nung Chen},
  title        = {Eliminating Biased Length Reliance of Direct Preference Optimization
                  via Down-Sampled {KL} Divergence},
  booktitle    = {Proceedings of the 2024 Conference on Empirical Methods in Natural
                  Language Processing, {EMNLP} 2024, Miami, FL, USA, November 12-16,
                  2024},
  pages        = {1047--1067},
  publisher    = {Association for Computational Linguistics},
  year         = {2024},
}

@inproceedings{DBLP:conf/acl/RimskyGSTHT24,
  author       = {Nina Rimsky and
                  Nick Gabrieli and
                  Julian Schulz and
                  Meg Tong and
                  Evan Hubinger and
                  Alexander Matt Turner},
  editor       = {Lun{-}Wei Ku and
                  Andre Martins and
                  Vivek Srikumar},
  title        = {Steering Llama 2 via Contrastive Activation Addition},
  booktitle    = {Proceedings of the 62nd Annual Meeting of the Association for Computational
                  Linguistics (Volume 1: Long Papers), {ACL} 2024, Bangkok, Thailand,
                  August 11-16, 2024},
  pages        = {15504--15522},
  publisher    = {Association for Computational Linguistics},
  year         = {2024},
}

@inproceedings{
lin2026activedpo,
title={Active{DPO}: Active Direct Preference Optimization for Sample-Efficient Alignment},
author={Xiaoqiang Lin and Arun Verma and Zhongxiang Dai and Daniela Rus and See-Kiong Ng and Bryan Kian Hsiang Low},
booktitle={The Fourteenth International Conference on Learning Representations},
year={2026},
}

@inproceedings{DBLP:conf/nips/RafailovSMMEF23,
  author       = {Rafael Rafailov and
                  Archit Sharma and
                  Eric Mitchell and
                  Christopher D. Manning and
                  Stefano Ermon and
                  Chelsea Finn},
  editor       = {Alice Oh and
                  Tristan Naumann and
                  Amir Globerson and
                  Kate Saenko and
                  Moritz Hardt and
                  Sergey Levine},
  title        = {Direct Preference Optimization: Your Language Model is Secretly a
                  Reward Model},
  booktitle    = {Advances in Neural Information Processing Systems 36: Annual Conference
                  on Neural Information Processing Systems 2023, NeurIPS 2023, New Orleans,
                  LA, USA, December 10 - 16, 2023},
  pages        = {53728--53741},
  year         = {2023},
}

@inproceedings{DBLP:conf/icml/GaoSH23,
  author       = {Leo Gao and
                  John Schulman and
                  Jacob Hilton},
  editor       = {Andreas Krause and
                  Emma Brunskill and
                  Kyunghyun Cho and
                  Barbara Engelhardt and
                  Sivan Sabato and
                  Jonathan Scarlett},
  title        = {Scaling Laws for Reward Model Overoptimization},
  booktitle    = {International Conference on Machine Learning, {ICML} 2023, 23-29 July
                  2023, Honolulu, Hawaii, {USA}},
  series       = {Proceedings of Machine Learning Research},
  volume       = {202},
  pages        = {10835--10866},
  publisher    = {{PMLR}},
  year         = {2023},
}

@inproceedings{DBLP:conf/icml/BombariM25,
  author       = {Simone Bombari and
                  Marco Mondelli},
  editor       = {Aarti Singh and
                  Maryam Fazel and
                  Daniel Hsu and
                  Simon Lacoste{-}Julien and
                  Felix Berkenkamp and
                  Tegan Maharaj and
                  Kiri Wagstaff and
                  Jerry Zhu},
  title        = {Spurious Correlations in High Dimensional Regression: The Roles of
                  Regularization, Simplicity Bias and Over-Parameterization},
  booktitle    = {Forty-second International Conference on Machine Learning, {ICML}
                  2025, Vancouver, BC, Canada, July 13-19, 2025},
  series       = {Proceedings of Machine Learning Research},
  volume       = {267},
  pages        = {4839--4873},
  publisher    = {{PMLR} / OpenReview.net},
  year         = {2025},
}

@inproceedings{DBLP:conf/iclr/SharmaTKDABDHJK24,
  author       = {Mrinank Sharma and
                  Meg Tong and
                  Tomasz Korbak and
                  David Duvenaud and
                  Amanda Askell and
                  Samuel R. Bowman and
                  Esin Durmus and
                  Zac Hatfield{-}Dodds and
                  Scott R. Johnston and
                  Shauna Kravec and
                  Timothy Maxwell and
                  Sam McCandlish and
                  Kamal Ndousse and
                  Oliver Rausch and
                  Nicholas Schiefer and
                  Da Yan and
                  Miranda Zhang and
                  Ethan Perez},
  title        = {Towards Understanding Sycophancy in Language Models},
  booktitle    = {The Twelfth International Conference on Learning Representations,
                  {ICLR} 2024, Vienna, Austria, May 7-11, 2024},
  publisher    = {OpenReview.net},
  year         = {2024},
}

@inproceedings{DBLP:conf/iclr/NagarajanAN21,
  author       = {Vaishnavh Nagarajan and
                  Anders Andreassen and
                  Behnam Neyshabur},
  title        = {Understanding the failure modes of out-of-distribution generalization},
  booktitle    = {9th International Conference on Learning Representations, {ICLR} 2021,
                  Virtual Event, Austria, May 3-7, 2021},
  publisher    = {OpenReview.net},
  year         = {2021},
}

@inproceedings{DBLP:conf/nips/VeitchDYE21,
  author       = {Victor Veitch and
                  Alexander D'Amour and
                  Steve Yadlowsky and
                  Jacob Eisenstein},
  editor       = {Marc'Aurelio Ranzato and
                  Alina Beygelzimer and
                  Yann N. Dauphin and
                  Percy Liang and
                  Jennifer Wortman Vaughan},
  title        = {Counterfactual Invariance to Spurious Correlations in Text Classification},
  booktitle    = {Advances in Neural Information Processing Systems 34: Annual Conference
                  on Neural Information Processing Systems 2021, NeurIPS 2021, December
                  6-14, 2021, virtual},
  pages        = {16196--16208},
  year         = {2021},
}

@inproceedings{DBLP:conf/icml/GaoL0XZ025,
  author       = {Chengqian Gao and
                  Haonan Li and
                  Liu Liu and
                  Zeke Xie and
                  Peilin Zhao and
                  Zhiqiang Xu},
  editor       = {Aarti Singh and
                  Maryam Fazel and
                  Daniel Hsu and
                  Simon Lacoste{-}Julien and
                  Felix Berkenkamp and
                  Tegan Maharaj and
                  Kiri Wagstaff and
                  Jerry Zhu},
  title        = {Principled Data Selection for Alignment: The Hidden Risks of Difficult
                  Examples},
  booktitle    = {Forty-second International Conference on Machine Learning, {ICML}
                  2025, Vancouver, BC, Canada, July 13-19, 2025},
  series       = {Proceedings of Machine Learning Research},
  volume       = {267},
  pages        = {18386--18409},
  publisher    = {{PMLR} / OpenReview.net},
  year         = {2025},
}

@inproceedings{DBLP:conf/icml/MuldrewHZB24,
  author       = {William Muldrew and
                  Peter Hayes and
                  Mingtian Zhang and
                  David Barber},
  editor       = {Ruslan Salakhutdinov and
                  Zico Kolter and
                  Katherine A. Heller and
                  Adrian Weller and
                  Nuria Oliver and
                  Jonathan Scarlett and
                  Felix Berkenkamp},
  title        = {Active Preference Learning for Large Language Models},
  booktitle    = {Forty-first International Conference on Machine Learning, {ICML} 2024,
                  Vienna, Austria, July 21-27, 2024},
  series       = {Proceedings of Machine Learning Research},
  volume       = {235},
  pages        = {36577--36590},
  publisher    = {{PMLR} / OpenReview.net},
  year         = {2024},
}

@inproceedings{DBLP:conf/icml/KohL17,
  author       = {Pang Wei Koh and
                  Percy Liang},
  editor       = {Doina Precup and
                  Yee Whye Teh},
  title        = {Understanding Black-box Predictions via Influence Functions},
  booktitle    = {Proceedings of the 34th International Conference on Machine Learning,
                  {ICML} 2017, Sydney, NSW, Australia, 6-11 August 2017},
  series       = {Proceedings of Machine Learning Research},
  volume       = {70},
  pages        = {1885--1894},
  publisher    = {{PMLR}},
  year         = {2017},
}

@inproceedings{DBLP:conf/nips/MeloTAG24,
  author       = {Luckeciano Carvalho Melo and
                  Panagiotis Tigas and
                  Alessandro Abate and
                  Yarin Gal},
  editor       = {Amir Globerson and
                  Lester Mackey and
                  Danielle Belgrave and
                  Angela Fan and
                  Ulrich Paquet and
                  Jakub M. Tomczak and
                  Cheng Zhang},
  title        = {Deep Bayesian Active Learning for Preference Modeling in Large Language
                  Models},
  booktitle    = {Advances in Neural Information Processing Systems 37: Annual Conference
                  on Neural Information Processing Systems 2024, NeurIPS 2024, Vancouver,
                  BC, Canada, December 10 - 15, 2024},
  pages        = {118052--118085},
  year         = {2024},
}

@inproceedings{DBLP:conf/pkdd/DasCPC25,
  author       = {Nirjhar Das and
                  Souradip Chakraborty and
                  Aldo Pacchiano and
                  Sayak Ray Chowdhury},
  editor       = {Rita P. Ribeiro and
                  Bernhard Pfahringer and
                  Nathalie Japkowicz and
                  Pedro Larra{\~{n}}aga and
                  Al{\'{\i}}pio M. Jorge and
                  Carlos Soares and
                  Pedro H. Abreu and
                  Jo{\~{a}}o Gama},
  title        = {Active Preference Optimization for Sample Efficient {RLHF}},
  booktitle    = {Machine Learning and Knowledge Discovery in Databases. Research Track
                  - European Conference, {ECML} {PKDD} 2025, Porto, Portugal, September
                  15-19, 2025, Proceedings, Part {V}},
  series       = {Lecture Notes in Computer Science},
  volume       = {16017},
  pages        = {96--112},
  publisher    = {Springer},
  year         = {2025},
}

@inproceedings{DBLP:conf/nips/ChristianoLBMLA17,
  author       = {Paul F. Christiano and
                  Jan Leike and
                  Tom B. Brown and
                  Miljan Martic and
                  Shane Legg and
                  Dario Amodei},
  editor       = {Isabelle Guyon and
                  Ulrike von Luxburg and
                  Samy Bengio and
                  Hanna M. Wallach and
                  Rob Fergus and
                  S. V. N. Vishwanathan and
                  Roman Garnett},
  title        = {Deep Reinforcement Learning from Human Preferences},
  booktitle    = {Advances in Neural Information Processing Systems 30: Annual Conference
                  on Neural Information Processing Systems 2017, December 4-9, 2017,
                  Long Beach, CA, {USA}},
  pages        = {4299--4307},
  year         = {2017},
}

@inproceedings{DBLP:conf/nips/DuboisLTZGBGLH23,
  author       = {Yann Dubois and
                  Chen Xuechen Li and
                  Rohan Taori and
                  Tianyi Zhang and
                  Ishaan Gulrajani and
                  Jimmy Ba and
                  Carlos Guestrin and
                  Percy Liang and
                  Tatsunori B. Hashimoto},
  editor       = {Alice Oh and
                  Tristan Naumann and
                  Amir Globerson and
                  Kate Saenko and
                  Moritz Hardt and
                  Sergey Levine},
  title        = {AlpacaFarm: {A} Simulation Framework for Methods that Learn from Human
                  Feedback},
  booktitle    = {Advances in Neural Information Processing Systems 36: Annual Conference
                  on Neural Information Processing Systems 2023, NeurIPS 2023, New Orleans,
                  LA, USA, December 10 - 16, 2023},
  pages        = {30039--30069},
  year         = {2023},
}

@article{DBLP:journals/tmlr/CasperDSGSRFKLF23,
  author       = {Stephen Casper and
                  Xander Davies and
                  Claudia Shi and
                  Thomas Krendl Gilbert and
                  J{\'{e}}r{\'{e}}my Scheurer and
                  Javier Rando and
                  Rachel Freedman and
                  Tomasz Korbak and
                  David Lindner and
                  Pedro Freire and
                  Tony Tong Wang and
                  Samuel Marks and
                  Charbel{-}Rapha{\"{e}}l S{\'{e}}gerie and
                  Micah Carroll and
                  Andi Peng and
                  Phillip J. K. Christoffersen and
                  Mehul Damani and
                  Stewart Slocum and
                  Usman Anwar and
                  Anand Siththaranjan and
                  Max Nadeau and
                  Eric J. Michaud and
                  Jacob Pfau and
                  Dmitrii Krasheninnikov and
                  Xin Chen and
                  Lauro Langosco and
                  Peter Hase and
                  Erdem B{\i}y{\i}k and
                  Anca D. Dragan and
                  David Krueger and
                  Dorsa Sadigh and
                  Dylan Hadfield{-}Menell},
  title        = {Open Problems and Fundamental Limitations of Reinforcement Learning
                  from Human Feedback},
  journal      = {Trans. Mach. Learn. Res.},
  volume       = {2023},
  year         = {2023},
}

@inproceedings{moya2026spurious,
  title     = {Spurious Correlation Learning in Preference Optimization: Mechanisms, Consequences, and Mitigation via Tie Training},
  author    = {Moya, Christian and Semendinger, Alex and Lin, Guang and Thornley, Elliott},
  booktitle = {Proceedings of the 43rd International Conference on Machine Learning},
  year      = {2026},
  publisher = {PMLR}
}

@article{martens2020natural,
  author = {James Martens},
  title = {New Insights and Perspectives on the Natural Gradient Method},
  journal = {Journal of Machine Learning Research},
  year = {2020},
  volume = {21},
  number = {146},
  pages = {1--76},
}

@article{schraudolph2002fast,
  author = {Nicol N. Schraudolph},
  title = {Fast Curvature Matrix-Vector Products for Second-Order Gradient Descent},
  journal = {Neural Computation},
  year = {2002},
  volume = {14},
  number = {7},
  pages = {1723--1738},
}

@inproceedings{ji2025pkusaferlhf,
  title = {{PKU}-{S}afe{RLHF}: Towards Multi-Level Safety Alignment for {LLM}s with Human Preference},
  author = {Jiaming Ji and Donghai Hong and Borong Zhang and Boyuan Chen and Josef Dai and Boren Zheng and Tianyi Alex Qiu and Jiayi Zhou and Kaile Wang and Boxun Li and Sirui Han and Yike Guo and Yaodong Yang},
  booktitle = {Proceedings of the 63rd Annual Meeting of the Association for Computational Linguistics (Volume 1: Long Papers)},
  year = {2025},
  address = {Vienna, Austria},
  publisher = {Association for Computational Linguistics},
  pages = {31983--32016},
}

@article{wang2026patterning,
  title={Patterning in Practice: Debiasing Reward Models with Susceptibilities},
  author={Wang, George and Donoway, Elizabeth and Murfet, Daniel},
  journal={arXiv preprint arXiv:2609.00699},
  year={2026}
}

\clearpage
\appendix
\section{Proofs}
\label{app:theory}
\subsection{Proofs for the Influence Analysis}
\label{app:influence-proofs}

\subsubsection{Tie Gradient}

\noindent\textbf{Claim:}
The gradient of the tie loss averaged over its two orientations is
$$
    g_{j,c}(\theta)=0,
    \qquad
    g_{j,s}(\theta)
      =\frac\beta2
        \tanh\!\left(\frac{\beta\tilde\theta_s^\top d_j}{2}\right)d_j.
$$

\noindent\textbf{Proof:}
The two orientations have feature differences $[0;d_j]$ and $[0;-d_j]$.
Averaging their losses gives
$$
    \ell_j(\theta)
      =-\frac12\log\sigma(\beta\tilde\theta_s^\top d_j)
       -\frac12\log\sigma(-\beta\tilde\theta_s^\top d_j).
$$
The resulting loss does not depend on the core parameters, so its core gradient is zero.
Differentiating with respect to the spurious parameters gives
\begin{align*}
    g_{j,s}(\theta)
      &=\frac\beta2\left[
          \sigma(\beta\tilde\theta_s^\top d_j)
          -\sigma(-\beta\tilde\theta_s^\top d_j)
        \right]d_j \\
      &=\frac\beta2
        \tanh\!\left(\frac{\beta\tilde\theta_s^\top d_j}{2}\right)d_j.
\end{align*}
For the second equality, use $\sigma(z)-\sigma(-z)=\tanh(z/2)$.
Reversing the orientation replaces $d_j$ by $-d_j$.
This changes the signs of both the hyperbolic tangent and the feature difference, so their product is unchanged.
\hfill$\square$

\subsubsection{Parameter Response}
\label{app:parameter-response}

Partition the Hessian of the strict objective into its core and spurious blocks.
Let $K_s$ denote the spurious Schur complement of $H_{cc}$ in $H$:
$$
    H=
    \begin{bmatrix}
        H_{cc} & H_{cs}\\
        H_{sc} & H_{ss}
    \end{bmatrix},
    \qquad
    K_s=H_{ss}-H_{sc}H_{cc}^{-1}H_{cs}.
$$

\noindent\textbf{Claim:}
Under the stationarity condition of Proposition~\ref{prop:single-tie-influence}, the perturbed objective of Section~\ref{subsec:first_order_effect} has a unique nearby minimizer for sufficiently small $\epsilon\ge0$, and
\begin{equation}
\label{eq:appendix-schur-response}
    \theta_j^\epsilon-\theta_0
      =-\frac\epsilon N H^{-1}g_j(\theta_0)+O(\epsilon^2),
    \qquad
    \theta_{j,s}^\epsilon-\theta_{0,s}
      =-\frac\epsilon N K_s^{-1}g_{j,s}(\theta_0)+O(\epsilon^2).
\end{equation}

\noindent\textbf{Proof:}
The perturbed minimizer solves the stationarity equation
$$
    \nabla\mathcal L_{\rm DPO}(\theta)
      +\frac\epsilon N\nabla\ell_j(\theta)=0
$$
at $\theta=\theta_j^\epsilon$.
At $\epsilon=0$, the solution is $\theta_j^0=\theta_0$.
Differentiating the left side with respect to $\theta$ at $\theta=\theta_0$ and $\epsilon=0$ gives $H$.
Because the losses are smooth and the positive definite matrix $H$ is invertible, the implicit function theorem gives a unique differentiable solution $\theta_j^\epsilon$ near $\theta_0$ for sufficiently small $\epsilon$.
For $\epsilon\ge0$, both terms in the perturbed log-linear objective are convex, so this stationary point is a minimizer.

Differentiating the stationarity equation with respect to $\epsilon$ gives
$$
    \left[\nabla^2\mathcal L_{\rm DPO}(\theta_j^\epsilon)
      +\frac\epsilon N\nabla^2\ell_j(\theta_j^\epsilon)\right]
    \frac{d\theta_j^\epsilon}{d\epsilon}
    +\frac1N\nabla\ell_j(\theta_j^\epsilon)=0.
$$
At $\epsilon=0$, $\theta_j^0=\theta_0$ and $\nabla^2\mathcal L_{\rm DPO}(\theta_0)=H$; the term multiplied by $\epsilon$ vanishes.
Thus
$$
    H\left.\frac{d\theta_j^\epsilon}{d\epsilon}\right|_{\epsilon=0}
      +\frac1N g_j(\theta_0)=0.
$$
Solving for the derivative gives
$$
    \left.\frac{d\theta_j^\epsilon}{d\epsilon}\right|_{\epsilon=0}
      =-\frac1N H^{-1}g_j(\theta_0).
$$
The Taylor expansion of $\theta_j^\epsilon$ around $\epsilon=0$ is
$$
    \theta_j^\epsilon
      =\theta_j^0
       +\epsilon\left.\frac{d\theta_j^\epsilon}{d\epsilon}
                    \right|_{\epsilon=0}
       +O(\epsilon^2)
      =\theta_0-\frac\epsilon N H^{-1}g_j(\theta_0)
       +O(\epsilon^2),
$$
which is the first identity in~\eqref{eq:appendix-schur-response}.

The preceding claim gives $g_{j,c}(\theta_0)=0$.
To recover the spurious part of the first identity, write the equation for $d\theta_j^\epsilon/d\epsilon$ in block form:
$$
    \begin{bmatrix}
        H_{cc} & H_{cs}\\
        H_{sc} & H_{ss}
    \end{bmatrix}
    \begin{bmatrix}
        \left.\dfrac{d\theta_{j,c}^\epsilon}{d\epsilon}
          \right|_{\epsilon=0}\\[1.1ex]
        \left.\dfrac{d\theta_{j,s}^\epsilon}{d\epsilon}
          \right|_{\epsilon=0}
    \end{bmatrix}
    =-\frac1N
    \begin{bmatrix}
        0\\
        g_{j,s}(\theta_0)
    \end{bmatrix}.
$$
Because $H\succ0$, both $H_{cc}$ and $K_s$ are positive definite and hence invertible.
The first row of the system gives
$$
    \left.\frac{d\theta_{j,c}^\epsilon}{d\epsilon}
      \right|_{\epsilon=0}
      =-H_{cc}^{-1}H_{cs}
        \left.\frac{d\theta_{j,s}^\epsilon}{d\epsilon}
        \right|_{\epsilon=0}.
$$
Substituting this expression into the second row gives
$$
    \left(H_{ss}-H_{sc}H_{cc}^{-1}H_{cs}\right)
      \left.\frac{d\theta_{j,s}^\epsilon}{d\epsilon}
      \right|_{\epsilon=0}
      =-\frac1N g_{j,s}(\theta_0),
$$
or equivalently,
$$
    K_s\left.\frac{d\theta_{j,s}^\epsilon}{d\epsilon}
        \right|_{\epsilon=0}
      =-\frac1N g_{j,s}(\theta_0).
$$
Thus
$$
    \left.\frac{d\theta_{j,s}^\epsilon}{d\epsilon}\right|_{\epsilon=0}
      =-\frac1N K_s^{-1}g_{j,s}(\theta_0).
$$
Taylor expanding the spurious parameters around $\epsilon=0$ now gives
$$
    \theta_{j,s}^\epsilon
      =\theta_{0,s}
       +\epsilon\left.\frac{d\theta_{j,s}^\epsilon}{d\epsilon}
                    \right|_{\epsilon=0}
       +O(\epsilon^2)
      =\theta_{0,s}-\frac\epsilon N
          K_s^{-1}g_{j,s}(\theta_0)+O(\epsilon^2),
$$
which is the second identity in~\eqref{eq:appendix-schur-response}.
\hfill$\square$

\subsubsection{Influence on Spurious Energy}

\noindent\textbf{Claim:}
Under the conditions of Proposition~\ref{prop:single-tie-influence},
$$
    E(\theta_j^\epsilon)-E(\theta_0)
      =-\epsilon A_j+O(\epsilon^2),
    \qquad
    A_j=\frac1N
      \bigl\langle\nabla_sE(\theta_0),K_s^{-1}g_{j,s}(\theta_0)\bigr\rangle.
$$
Equivalently,
$$
    A_j
      =\frac{\beta}{2N}
        \tanh\!\left(\frac{\beta\tilde\theta_{0,s}^\top d_j}{2}\right)
        \bigl\langle\nabla_sE(\theta_0),K_s^{-1}d_j\bigr\rangle.
$$

\noindent\textbf{Proof:}
The gradient of the spurious energy is
$$
    \nabla E(\theta_0)
      =2\begin{bmatrix}
          0\\
          \Sigma_{ss}\tilde\theta_{0,s}
        \end{bmatrix}.
$$
This gradient has no core component.
Combining the spurious response in~\eqref{eq:appendix-schur-response} with a Taylor expansion of $E$ gives
\begin{align*}
    E(\theta_j^\epsilon)-E(\theta_0)
      &=\bigl\langle\nabla_sE(\theta_0),
            \theta_{j,s}^\epsilon-\theta_{0,s}\bigr\rangle
        +O(\epsilon^2)\\
      &=-\frac\epsilon N
        \bigl\langle\nabla_sE(\theta_0),
            K_s^{-1}g_{j,s}(\theta_0)\bigr\rangle
        +O(\epsilon^2)\\
      &=-\epsilon A_j+O(\epsilon^2).
\end{align*}
Substituting the tie gradient proves the second expression for $A_j$.
\hfill$\square$

\subsubsection{First-Order Optimal Selection}
\label{app:influence-selection}

\noindent\textbf{Claim:}
For any fixed set $S$ of $k$ ties, let $\theta_S^\epsilon$ be the nearby minimizer of
$$
    \mathcal L_{\rm DPO}(\theta)+\frac\epsilon N\sum_{j\in S}\ell_j(\theta).
$$
Then
\begin{align*}
    \theta_S^\epsilon-\theta_0
      &=-\frac\epsilon N H^{-1}
          \sum_{j\in S}g_j(\theta_0)+O(\epsilon^2),\\
    \theta_{S,s}^\epsilon-\theta_{0,s}
      &=-\frac\epsilon N K_s^{-1}
          \sum_{j\in S}g_{j,s}(\theta_0)+O(\epsilon^2),
\end{align*}
and therefore
$$
    E(\theta_S^\epsilon)
      =E(\theta_0)-\epsilon\sum_{j\in S}A_j+O(\epsilon^2).
$$
Let $S_k$ contain the $k$ largest scores $A_j$, and let $S_{\rm UTL}$ be a set of $k$ ties sampled uniformly.
Then
$$
    \sum_{j\in S_k}A_j
      =\max_{|S|=k}\sum_{j\in S}A_j
      \ge\mathbb E\!\left[\sum_{j\in S_{\rm UTL}}A_j\right]
      =\frac{k}{M}
        \sum_{j=1}^{M}A_j.
$$
For $0<k<M$, the inequality is strict unless all scores are equal; when $k=M$, the selected sets coincide.
The expectation is over the subset sampled by UTL.

\noindent\textbf{Proof:}
\paragraph{Additivity.}
The perturbed minimizer solves the stationarity equation
$$
    \nabla\mathcal L_{\rm DPO}(\theta)
      +\frac\epsilon N\sum_{j\in S}
          \nabla\ell_j(\theta)=0
$$
at $\theta=\theta_S^\epsilon$.
At $\epsilon=0$, the solution is $\theta_S^0=\theta_0$.
Differentiating the left side with respect to $\theta$ at $\theta=\theta_0$ and $\epsilon=0$ gives $H$.
Because $\mathcal L_{\rm DPO}$ and each $\ell_j$ are smooth and $H$ is invertible, the implicit function theorem gives a unique differentiable stationary point near $\theta_0$ for sufficiently small $\epsilon$.
For $\epsilon\ge0$, the perturbed objective is convex, so this stationary point is a minimizer.
Differentiating the stationarity equation with respect to $\epsilon$ gives
$$
    \left[\nabla^2\mathcal L_{\rm DPO}(\theta_S^\epsilon)
      +\frac\epsilon N\sum_{j\in S}
         \nabla^2\ell_j(\theta_S^\epsilon)\right]
    \frac{d\theta_S^\epsilon}{d\epsilon}
    +\frac1N\sum_{j\in S}\nabla\ell_j(\theta_S^\epsilon)=0.
$$
Evaluating this equation at zero gives
$$
    H\left.\frac{d\theta_S^\epsilon}{d\epsilon}\right|_{\epsilon=0}
      +\frac1N\sum_{j\in S}g_j(\theta_0)=0,
$$
and hence
$$
    \left.\frac{d\theta_S^\epsilon}{d\epsilon}\right|_{\epsilon=0}
      =-\frac1N H^{-1}\sum_{j\in S}g_j(\theta_0).
$$
The Taylor expansion around $\epsilon=0$ gives
$$
    \theta_S^\epsilon
      =\theta_S^0
       +\epsilon\left.\frac{d\theta_S^\epsilon}{d\epsilon}
                    \right|_{\epsilon=0}
       +O(\epsilon^2)
      =\theta_0-\frac\epsilon N H^{-1}
          \sum_{j\in S}g_j(\theta_0)+O(\epsilon^2).
$$

Every tie gradient has a zero core block, so their sum does as well.
The differentiated equation therefore has the block form
$$
    \begin{bmatrix}
        H_{cc} & H_{cs}\\
        H_{sc} & H_{ss}
    \end{bmatrix}
    \begin{bmatrix}
        \left.\dfrac{d\theta_{S,c}^\epsilon}{d\epsilon}
          \right|_{\epsilon=0}\\[1.1ex]
        \left.\dfrac{d\theta_{S,s}^\epsilon}{d\epsilon}
          \right|_{\epsilon=0}
    \end{bmatrix}
    =-\frac1N
    \begin{bmatrix}
        0\\
        \displaystyle\sum_{j\in S}g_{j,s}(\theta_0)
    \end{bmatrix}.
$$
The first row gives
$$
    \left.\frac{d\theta_{S,c}^\epsilon}{d\epsilon}
      \right|_{\epsilon=0}
      =-H_{cc}^{-1}H_{cs}
        \left.\frac{d\theta_{S,s}^\epsilon}{d\epsilon}
        \right|_{\epsilon=0}.
$$
Substituting this expression into the second row gives
$$
    K_s\left.\frac{d\theta_{S,s}^\epsilon}{d\epsilon}\right|_{\epsilon=0}
      =-\frac1N\sum_{j\in S}g_{j,s}(\theta_0).
$$
The Taylor expansion of the spurious parameters is therefore
$$
    \theta_{S,s}^\epsilon
      =\theta_{0,s}
       +\epsilon\left.\frac{d\theta_{S,s}^\epsilon}{d\epsilon}
                    \right|_{\epsilon=0}
       +O(\epsilon^2)
      =\theta_{0,s}-\frac\epsilon N K_s^{-1}
          \sum_{j\in S}g_{j,s}(\theta_0)+O(\epsilon^2).
$$

Finally, expanding the spurious energy along this response gives
\begin{align*}
    E(\theta_S^\epsilon)-E(\theta_0)
      &=-\frac\epsilon N
        \left\langle\nabla_sE(\theta_0),
          K_s^{-1}\sum_{j\in S}g_{j,s}(\theta_0)\right\rangle
        +O(\epsilon^2)\\
      &=-\epsilon\sum_{j\in S}A_j+O(\epsilon^2).
\end{align*}
By the definition of $A_j$, the second equality shows that the influences of the individual ties add at first order.

\paragraph{Comparison with UTL.}
For any other set $S$ with $|S|=k$, the sets $S_k\setminus S$ and $S\setminus S_k$ have the same size.
Pair their elements arbitrarily as $(i_r,j_r)$, where $i_r\in S_k\setminus S$ and $j_r\in S\setminus S_k$.
Since $S_k$ contains the largest scores, $A_{i_r}\ge A_{j_r}$ for every pair, and hence
$$
    \sum_{j\in S_k}A_j-\sum_{j\in S}A_j
      =\sum_r(A_{i_r}-A_{j_r})\ge0.
$$

This inequality holds for every set of $k$ ties, including every realization of $S_{\rm UTL}$.
Taking the expectation over UTL therefore gives
$$
    \sum_{j\in S_k}A_j
      \ge\mathbb E\!\left[\sum_{j\in S_{\rm UTL}}A_j\right].
$$
Each tie belongs to $S_{\rm UTL}$ with probability $k/M$, so
$$
    \mathbb E\!\left[\sum_{j\in S_{\rm UTL}}A_j\right]
      =\sum_{j=1}^{M}A_j\Pr(j\in S_{\rm UTL})
      =\frac{k}{M}\sum_{j=1}^{M}A_j.
$$
If $0<k<M$ and the scores are not all equal, at least one set of $k$ ties has a smaller total than $S_k$.
Every such set has positive probability under UTL, so the inequality is strict.
When all scores are equal, or when $k=M$, equality holds.
\hfill$\square$

\subsection{Ranking Ties from Margin Information}
\label{app:pm-optimality}

\noindent\textbf{Claim:}
Under the conditions of Theorem~\ref{thm:pm-optimality}, there is a constant $\lambda>0$, common to all ties, such that
$$
    \mathbb E[A_j\mid m_j]
      =\frac\lambda N m_j\tanh\!\left(m_j/2\right).
$$
This conditional mean depends on $m_j$ only through $|m_j|$ and strictly increases with $|m_j|$.

\noindent\textbf{Proof:}
For brevity, write $w=\tilde\theta_{0,s}$.
If $w=0$, every tie margin and influence is zero, so the claim is immediate.
We therefore consider $w\ne0$.

Let $d_j$ be the spurious contrast of the sampled orientation of tie $j$, so that
$$
    m_j=\beta w^\top d_j.
$$
Reversing the orientation changes the signs of both $m_j$ and $d_j$.
Their signs cancel in the influence, so either orientation gives
$$
    A_j=\frac\beta N\tanh(m_j/2)
      \bigl\langle\Sigma_{ss}w,
          K_s^{-1}d_j\bigr\rangle.
$$

By the assumptions of Theorem~\ref{thm:pm-optimality}, $d_j\sim\mathcal N(0,C)$ with $C=c\Sigma_{ss}$.
Because $m_j=\beta w^\top d_j$,
$$
    \operatorname{Cov}(d_j,m_j)=\beta Cw,
    \qquad
    \operatorname{Var}(m_j)=\beta^2w^\top Cw.
$$
The conditional mean of a jointly Gaussian vector is therefore
$$
    \mathbb E[d_j\mid m_j]
      =\frac{Cw}{w^\top Cw}\frac{m_j}{\beta}.
$$
Inserting this conditional mean into the influence expression gives
$$
    \mathbb E[A_j\mid m_j]
      =\frac{m_j\tanh(m_j/2)}{N}
        \frac{w^\top\Sigma_{ss}K_s^{-1}Cw}
             {w^\top Cw}.
$$
Substituting $C=c\Sigma_{ss}$ cancels the scale $c$.
Define
\begin{equation}
\label{eq:pm-positive-coefficient}
    \lambda
      =\frac{w^\top\Sigma_{ss}K_s^{-1}
                    \Sigma_{ss}w}
              {w^\top\Sigma_{ss}w}>0.
\end{equation}
Since $H\succ0$, its Schur complement $K_s$ is positive definite, and so is $K_s^{-1}$.
Moreover, $\Sigma_{ss}\succ0$ and $w\ne0$ imply $\Sigma_{ss}w\ne0$.
Therefore, the numerator
$(\Sigma_{ss}w)^\top K_s^{-1}(\Sigma_{ss}w)$
and the denominator $w^\top\Sigma_{ss}w$ are both positive.

The function $m\tanh(m/2)$ is even, so the conditional mean depends on $m_j$ only through $|m_j|$.
For $a>0$,
$$
    \frac{d}{da}\left[a\tanh(a/2)\right]
      =\tanh(a/2)+\frac a2\operatorname{sech}^2(a/2)>0.
$$
The function equals zero at $a=0$, which proves strict monotonicity on $[0,\infty)$.
\hfill$\square$

\noindent\textbf{Claim:}
For any budget $1\le k\le M$, conditional on the observed margins, pure margin selection maximizes expected total influence among selectors that use only those margins.

\noindent\textbf{Proof:}
Let
$$
    \boldsymbol m=(m_1,\ldots,m_M).
$$
Given $\theta_0$, the contrasts are independent across ties, so the pair $(A_j,m_j)$ is independent of the margins of all other ties.
Once the margin of tie $j$ is known, observing the other margins therefore does not change its expected influence:
$$
    \mathbb E[A_j\mid\boldsymbol m]
      =\mathbb E[A_j\mid m_j].
$$
Now consider any rule $S=S(\boldsymbol m)$ that selects $k$ ties.
Linearity of conditional expectation gives
$$
    \mathbb E\!\left[\sum_{j\in S}A_j\,\middle|\,\boldsymbol m\right]
      =\frac\lambda N\sum_{j\in S}|m_j|
        \tanh\!\left(|m_j|/2\right).
$$
Because $a\tanh(a/2)$ is strictly increasing for $a\ge0$, selecting the $k$ largest $|m_j|$ maximizes this conditional expectation.
\hfill$\square$

\subsection{Recovering the Spurious-Energy Gradient}
\label{app:probe-proofs}

Fix either orientation of each probe tie and let $d_n$ be its spurious feature difference.
Because the probe ties preserve core features, their core feature differences are zero.
The empirical second moment of these differences is
$$
    \widehat C=\frac1{N_v}\sum_{n=1}^{N_v}d_nd_n^\top.
$$

\noindent\textbf{Claim:}
In the log-linear setting,
\begin{equation}
\label{eq:probe-recovery}
    E_{\rm cf}(\theta)
      =\beta^2\tilde\theta_s^\top\widehat C\tilde\theta_s,
    \qquad
    \nabla E_{\rm cf}(\theta_0)
      =2\beta^2
        \begin{bmatrix}
            0\\
            \widehat C\tilde\theta_{0,s}
        \end{bmatrix}.
\end{equation}
If $\widehat C=c\Sigma_{ss}$ for some $c>0$, then
$$
    \nabla E_{\rm cf}(\theta_0)=c\beta^2\nabla E(\theta_0).
$$

\noindent\textbf{Proof:}
For tie $n$ in the probe, the reward difference in the log-linear model is $\beta\tilde\theta_s^\top d_n$.
Its sign depends on the chosen orientation, but its square does not.
Using these squared differences in~\eqref{eq:counterfactual-energy} gives
$$
    E_{\rm cf}(\theta)
      =\frac{\beta^2}{N_v}\sum_{n=1}^{N_v}
          (\tilde\theta_s^\top d_n)^2
      =\beta^2\tilde\theta_s^\top\widehat C\tilde\theta_s.
$$
Every outer product $d_nd_n^\top$ is symmetric, and hence so is their average $\widehat C$.
Differentiating the quadratic form therefore gives the gradient in~\eqref{eq:probe-recovery}.
The gradient of the spurious energy $E$ is
$$
    \nabla E(\theta_0)
      =2\begin{bmatrix}
          0\\
          \Sigma_{ss}\tilde\theta_{0,s}
        \end{bmatrix}.
$$
When $\widehat C=c\Sigma_{ss}$, direct substitution gives
$\nabla E_{\rm cf}(\theta_0)=c\beta^2\nabla E(\theta_0)$.
Under the assumptions of Section~\ref{subsec:pm-optimality}, $\mathbb E[d_nd_n^\top]=c\Sigma_{ss}$, so $\mathbb E[\widehat C]=c\Sigma_{ss}$.
Because the gradient in~\eqref{eq:probe-recovery} is linear in $\widehat C$, the identity then holds in expectation over the probe.
\hfill$\square$

\section{Implementation Details}
\label{app:implementation}
\subsection{Curvature and Numerical Solves}
\label{app:scoring-details}

Each influence score uses the inverse curvature of the DPO loss at
$\theta_0$. Because the gradients of the energy and each tie have zero core
blocks, and the spurious block of $H^{-1}$ is $K_s^{-1}$, the influence in
Proposition~\ref{prop:single-tie-influence} can be written over all parameters:
\begin{equation}
\label{eq:curvature-system}
    A_j=\frac1N\bigl\langle\nabla E(\theta_0),H^{-1}g_j(\theta_0)\bigr\rangle
       =\frac1N\bigl\langle v,g_j(\theta_0)\bigr\rangle,
    \qquad
    Hv=\nabla E(\theta_0).
\end{equation}
The second form follows from the symmetry of $H$. Thus, one linear solve scores
the entire pool, after which each tie requires only its gradient. The log-linear
model can form and solve this system exactly. A neural policy has a Hessian
of the same form over all $p$ parameters. Two obstacles prevent an exact
solve: the DPO loss is nonconvex in the network parameters, and $\theta_0$
was obtained after finite training, so $H$ need not be positive definite. For
the language model, $p$ is also in the millions, so $H$ cannot be stored or
factorized. We therefore solve the system exactly in the log-linear model and,
in the neural regimes, replace $H$ with a damped approximation that is
positive definite and cheap to solve.

\paragraph{Log-linear model.}
The log-linear model uses the exact Hessian of the empirical DPO objective at
$\theta_0$, formed from all $N$ comparisons in $\mathcal D_{\rm pref}$. For strict
comparison $i$, let
$\Delta\phi_i=\phi(x_i,y_{w,i})-\phi(x_i,y_{l,i})$ be its feature
difference and $m_i=\beta\tilde\theta_0^\top\Delta\phi_i$ its margin. With
$\sigma$ the logistic function, the Hessian is
\begin{equation}
\label{eq:linear-exact-hessian}
    H
    =\frac{\beta^2}{N}\sum_{i=1}^{N}
      \sigma(m_i)\sigma(-m_i)
      \Delta\phi_i\Delta\phi_i^\top.
\end{equation}
The right side of~\eqref{eq:curvature-system} is the analytic gradient of the
population energy, $\nabla E(\theta_0)=[0;\,2\Sigma_{ss}\tilde\theta_{0,s}]$. With only
$d_c+d_s=20$ parameters, we solve $Hv=\nabla E(\theta_0)$ exactly and without
damping, so each score equals the influence $A_j$.

\paragraph{Neural network and language model.}
The two neural regimes use the same approximation to the curvature and the same solver. Both
need the approximation because neither Hessian is guaranteed to be positive definite.
Only the language model needs the solver, because its curvature is too large
to form. The MLP, with about 19,000 parameters, could be solved directly, but
sharing the solver keeps a single implementation for both regimes. For the
MLP, the parameters are all weights in the reward network; for the language
model, they are the trainable LoRA parameters. Let $m_i$ be the DPO margin for strict comparison $i$; for the MLP, we
take the reference margin to be zero. Let $J_i=\nabla m_i(\theta_0)$, and let
$\lambda_{\rm obj}$ be the weight of the objective's L2 penalty
$\tfrac{\lambda_{\rm obj}}2\lVert\theta\rVert_2^2$: $0.01$ for the MLP and $0$
for LoRA-DPO. The Hessian of the strict objective is then
\begin{equation}
\label{eq:ggn-decomposition}
    H=\frac1N\sum_{i=1}^{N}\left[
       \sigma(m_i)\sigma(-m_i)J_iJ_i^\top
       -\sigma(-m_i)\nabla^2m_i(\theta_0)
    \right]+\lambda_{\rm obj}I.
\end{equation}
The generalized Gauss--Newton (GGN) matrix retains the first term, which is
positive semidefinite, and omits the second, which can be indefinite and
requires the Hessian of every margin. We estimate the first term from $N_G=1000$ strict comparisons
sampled once per seed. Weighting each margin gradient and stacking the results as rows gives
\begin{equation}
\label{eq:ggn-row-matrix}
    D_{i,:}=\sqrt{\sigma(m_i)\sigma(-m_i)}\,J_i^\top,
    \qquad
    G=\frac1{N_G}D^\top D.
\end{equation}

The matrix $D$ has $N_G$ rows, far fewer than the number of parameters
$p$, so $G$ has rank at most $N_G$ and most of its eigenvalues are zero. We
therefore damp $G$ in proportion to its largest eigenvalue and add the
objective's own ridge:
\begin{equation}
\label{eq:damped-ggn}
    \widehat H=G+\alpha I,
    \qquad
    \alpha=\lambda_{\rm obj}
      +0.05\bigl(\lambda_{\rm obj}+\lambda_{\max}(G)\bigr).
\end{equation}
The curvature for the MLP thus contains its L2 penalty, and both regimes add the same
relative damping.

With the counterfactual energy of Section~\ref{subsec:influence_practice} in
place of $E$, the system for the neural models is $\widehat Hv=\nabla E_{\rm cf}(\theta_0)$.
The low rank of $G$ also makes this system efficient to solve. The Woodbury identity
\begin{equation}
\label{eq:woodbury-identity}
    (G+\alpha I)^{-1}
      =\alpha^{-1}\left[I-D^\top\left(DD^\top+N_G\alpha I\right)^{-1}D\right]
\end{equation}
replaces the $p\times p$ system with an $N_G\times N_G$ one:
\begin{equation}
\label{eq:woodbury-solution}
    \xi=\left(DD^\top+N_G\alpha I\right)^{-1}D\nabla E_{\rm cf}(\theta_0),
    \qquad
    v=\alpha^{-1}\left(\nabla E_{\rm cf}(\theta_0)-D^\top\xi\right).
\end{equation}
The solution is exact for the GGN after sampling and damping. The small matrix
$DD^\top/N_G$ also supplies $\lambda_{\max}(G)$: both it and $G$ are formed
from the singular values of $D$, so they share their nonzero eigenvalues.

\paragraph{Numerical checks.}
Before solving the log-linear system, the implementation verifies the conditions of
Proposition~\ref{prop:single-tie-influence}: the norm of the gradient of the strict objective
must be at most $10^{-10}$, and $H$ must be positive definite. Scoring stops
if either check fails, rather than damping $H$. For the MLP, the solve runs in float64
and must reach a relative residual of $10^{-10}$. For the language model, $D$
is stored in bfloat16 to save memory, the $N_G\times N_G$ system is solved in
float64, and the relative residual must be below $10^{-4}$. These checks
verify that $v$ solves the system defined by the sampled and damped GGN.

\newpage 
\section{Experimental Details and Additional Results}
\label{app:additional-results}
\label{app:experiments}
\label{app:experimental-protocol}

This appendix describes the full experimental protocol and additional results for
Section~\ref{sec:experiments}.  We first describe the choices shared across
experiments, then give the data, models, training procedures, and metrics for
each regime.

\raggedbottom
\pgfplotsset{
 stl comparison utl/.style={color=black!65,dashed,line width=.85pt,mark=o,mark size=1.5pt,mark options={solid,fill=white}},
 stl comparison pm/.style={color=stlblue,solid,line width=1pt,mark=triangle*,mark size=1.9pt},
 stl comparison influence/.style={color=stlorange,dashdotted,line width=.65pt,mark=square,mark size=2.2pt,mark options={solid}},
 stl appendix axis/.style={
  stl axis, scale only axis,
  label style={font=\small},
  tick label style={font=\small},
  title style={font=\small,align=center,yshift=3pt},
  legend style={draw=none,fill=none,font=\small,cells={anchor=west},column sep=7pt},
  yticklabel style={inner xsep=2pt},
  xlabel={Tie budget $k$},
 },
}

\newenvironment{stlresultfigure}{%
  \begin{figure}[!htb]\centering
}{\end{figure}}

\subsection{Shared Protocol}
\label{app:shared-experimental-protocol}

\paragraph{Tie pools.}
Every regime uses the tie construction of Section~\ref{subsec:problem_selection}:
each tie keeps the strict winner's core content and changes only its spurious
attributes.  We construct one pool per seed and use it for every
method, so the methods differ only in which ties they choose.  Separate random
streams generate the data, construct the ties, orient them, and drive training
and evaluation, so a change in one step does not alter the random draws of
another.

\paragraph{Selection.}
For each seed, we first train a strict checkpoint on the strict comparisons
alone.  PM and influence score every tie in the pool at this checkpoint, as
derived in Sections~\ref{sec:influence}--\ref{sec:nonlinear}: Oracle influence
in the log-linear model and Practical influence in the neural regimes
(Appendix~\ref{app:scoring-details}).  Each targeted method selects the $k$
highest-scoring ties from one fixed ranking, so by construction a tie selected
at one budget remains selected at every larger budget.  UTL instead samples
$k$ ties uniformly at random from the same pool.

\paragraph{Training.}
The strict checkpoint is used only for scoring.  For each selected set, we
train a new model on the strict comparisons and the selected ties, starting
from the same initialization as the strict checkpoint.  Training requires a
labeled comparison, so we orient each tie once by a fair coin flip and use
these orientations for every method and budget.  Each selected tie then enters
the objective as one additional comparison, with weight $1/N$ relative to the
mean loss on strict comparisons ($\epsilon=1$ in Proposition~\ref{prop:single-tie-influence}).
At a fixed seed and budget, the methods therefore differ only in which ties
they select, and at $k=M$ they train on the same ties.

\paragraph{Uncertainty.}
We treat seeds as the independent units: each seed has its own data, tie pool,
strict checkpoint, and scores.  Given the seed, PM and influence select one
set per budget, so each trains one model.  UTL is random, so each seed trains
several UTL models per budget, each on a separately sampled set of ties (a
\emph{draw}), and averages their metrics to reduce variance.  Every method
therefore contributes one value per seed and budget.
Reported intervals are two-sided 95\% Student-$t$ confidence intervals across
seeds.

\paragraph{Evaluation distributions.}
We evaluate on held-out comparisons from three distributions that differ only
in how the spurious attributes relate to the preference.  In the training
distribution $P$, they correlate with the preference; in $Q_{\rm adv}$, these
correlations reverse, as in Section~\ref{sec:experiments}.  In
$Q_{\rm clean}$, the spurious attributes keep their distribution but become
independent of utility, so they carry no information about the winner.  Each
regime below specifies all three distributions.

\paragraph{Metrics.}
Accuracy under $Q_{\rm adv}$ measures performance when the training
correlations reverse.  It can also rise for the wrong reason.  If the selected
ties push the model to rely on the spurious attributes in the opposite
direction, this reversed shortcut also predicts the winner under
$Q_{\rm adv}$.  Under $Q_{\rm clean}$, a shortcut in either direction only adds
noise, so accuracy under $Q_{\rm clean}$ checks that selection removes the
shortcut rather than reverses it.  We also measure spurious reliance directly,
with the spurious energy $E$ in the log-linear model and the counterfactual
energy $E_{\rm cf}$ in the neural regimes.  Both square the spurious part of
the margin, so they grow with reliance in either direction.  Finally, accuracy
under $P$ shows what mitigation costs in the training environment.

\paragraph{Fresh and continued training.}
\label{app:continued-training}
Influence scores ties by their local effect at the strict checkpoint.  In the
log-linear model, the objective has a single minimizer, which training reaches
from any starting point, so fresh training ends at the solution that the
analysis describes.  A neural objective has no such guarantee, and fresh
training can end far from the checkpoint at which the scores were computed.
We also continue training from the strict checkpoint.  Continued
training uses the same settings and training budget as fresh training; only
the initialization differs.

\paragraph{Reading the additional results.}
Appendix figures report means only to avoid clutter from overlapping uncertainty
bands. The results text reports representative 95\% confidence intervals.
In figures that compare training procedures, fresh and continued training
occupy the left and right columns, with matching axes in each row. The Qwen
figure instead compares hotels and emojis and reports only fresh training.
Colors, markers, and line styles identify the same selector across columns.

\subsection{Synthetic Data}
\label{app:synthetic-dgp}

\paragraph{Latent variables.}
The log-linear experiments and the experiments with neural networks use the same latent distribution,
which lets us reverse two sources of shortcut information separately.
Each alternative has latent utility $q\sim\mathcal N(0,1)$, a core vector
$c\in\mathbb R^{d_c}$, and a spurious vector $s\in\mathbb R^{d_s}$.  Let
$z\sim\mathcal N(0,1)$ be independent of $q$, and let
$(\epsilon_c,\epsilon_s)$ be jointly Gaussian and independent of $(q,z)$.
The matrices $\Omega_c$ and $\Omega_s$ are the core and spurious noise covariances, $R$
controls their dependence, and $u_c$ and $u_s$ are their signal directions.
An environment is indexed by two parameters $(a,t)$ and generates
\begin{align}
    g &= a q + \sqrt{1-a^2}\,z, \\
    c &= u_c q+\epsilon_c, &
    s &= u_s g+\epsilon_s, \\
    \operatorname{Cov}\!\left(
    \begin{bmatrix}\epsilon_c\\ \epsilon_s\end{bmatrix}
    \right)
    &=
    \begin{pmatrix}
      \Omega_c & t\Omega_c^{1/2}R\,\Omega_s^{1/2}\\
      t\Omega_s^{1/2}R^\top\Omega_c^{1/2} & \Omega_s
    \end{pmatrix}.
    \label{eq:synthetic-dgp}
\end{align}
The parameter $a$ controls the shortcut in the mean: through $g$, spurious features
track utility directly.  The parameter $t$ controls leakage: through the noise
covariance, spurious features correlate with core features.  These are the
two mechanisms of Section~\ref{subsec:spurious_reliance}, and each can be
reversed independently.

\paragraph{Covariance structure.}
We draw a new geometry for each seed, so the results do not depend on one choice of
basis.  The matrices $\Omega_c$ and $\Omega_s$ are independent and positive definite, with
orthogonal eigenvectors drawn at random.  Their eigenvalues are sampled on a log
scale, normalized to have mean one, and span condition number 20.
Independently sampled orthogonal matrices define
$R=U\operatorname{diag}(r_1,\ldots,r_{10})V^\top$, with singular values evenly
spaced from $0.20$ to $0.65$.  Finally, the signal directions are
$u_c=\gamma_c\Omega_c^{1/2}v_c$ and $u_s=\gamma_s\Omega_s^{1/2}v_s$, where $v_c$ and $v_s$
are independent random unit vectors and $\gamma_c,\gamma_s>0$ are signal
strengths.  This scaling fixes the strength of each signal relative to its
noise, measured in Mahalanobis units:
\begin{equation}
    u_c^\top\Omega_c^{-1}u_c=\gamma_c^2,
    \qquad
    u_s^\top\Omega_s^{-1}u_s=\gamma_s^2.
\end{equation}
We use $d_c=d_s=10$ and $\gamma_c=\gamma_s=2$.

\paragraph{Strict preferences.}
Each strict comparison draws two alternatives independently.  A
Bradley--Terry teacher then chooses A with probability
\begin{equation}
    \Pr(A\succ B\mid q_A,q_B)
       =\sigma\!\left(\beta_{\rm teacher}(q_A-q_B)\right),
       \qquad \beta_{\rm teacher}=4.
    \label{eq:synthetic-bt-teacher}
\end{equation}
After sampling this label, we orient the pair as winner and loser.  Thus the
preference signal is stochastic, but labels depend only on latent utility and
never on $s$.

\paragraph{Evaluation distributions.}
Table~\ref{tab:synthetic-environments} sets $(a,t)$ for each distribution.
$P$ aligns both shortcut channels with utility, $Q_{\rm adv}$ reverses both,
and $Q_{\rm clean}$ sets both to zero, which makes $s$ independent of $q$ and
$c$.  All three preserve the marginal distributions of $q$, $c$, and $s$.

\begin{table}[t]
\centering
\small
\caption{Environments used to train and evaluate the synthetic models.  The parameters
$a$ and $t$ are defined in~\eqref{eq:synthetic-dgp}.}
\label{tab:synthetic-environments}
\begin{tabular}{@{}lccp{0.48\linewidth}@{}}
\toprule
Environment & $a$ & $t$ & Interpretation \\
\midrule
$P$ & $+1$ & $+1$ & Training environment; both shortcut channels align. \\
$Q_{\rm clean}$ & $0$ & $0$ & Both shortcut channels are removed. \\
$Q_{\rm adv}$ & $-1$ & $-1$ & Both shortcut channels reverse. \\
\bottomrule
\end{tabular}
\end{table}

\paragraph{Tie construction.}
For each seed, we sample $N$ strict comparisons to form
$\mathcal D_{\rm pref}$ and select $M$ of them uniformly without replacement
to construct ties.  The operator $T$ keeps the selected winner's core vector
$c_w$ and replaces its spurious vector $s_w$ with a \emph{donor} $s'$: the
spurious vector of a response drawn uniformly from all $2N$ responses in
$\mathcal D_{\rm pref}$, not only from the selected comparisons. The resulting tie
$\tau=(x,\{(c_w,s_w),(c_w,s')\})$ therefore differs only in spurious
content.

\paragraph{Settings.}
Table~\ref{tab:synthetic-hyperparameters} summarizes the settings of both
synthetic models.
\begin{table}[t]
\centering
\small
\caption{Settings for the synthetic experiments.}
\label{tab:synthetic-hyperparameters}
\begin{tabular}{@{}>{\raggedright\arraybackslash}p{0.37\linewidth}
                    >{\raggedright\arraybackslash}p{0.25\linewidth}
                    >{\raggedright\arraybackslash}p{0.30\linewidth}@{}}
\toprule
Setting & Log-linear & Neural network \\
\midrule
Feature dimensions $(d_c,d_s)$ & $(10,10)$ & $(10,10)$ \\
Condition numbers $\kappa(\Omega_c),\kappa(\Omega_s)$ & $(20,20)$ & $(20,20)$ \\
Singular values $\sigma_i(R)$ & $[0.20,0.65]$ & $[0.20,0.65]$ \\
Signal strengths $(\gamma_c,\gamma_s)$ & $(2,2)$ & $(2,2)$ \\
Teacher coefficient $\beta_{\rm teacher}$ & $4$ & $4$ \\
DPO coefficient $\beta$ & $1$ & $1$ \\
Reference term $\Delta\log\pi_{\rm ref}$ & $0$ & $0$ \\
Strict comparisons $N$ & $1{,}000$ & $5{,}000$ \\
Tie pool $M$ & $500$ & $2{,}000$ \\
Responses in donor bank & $2N=2{,}000$ & $2N=10{,}000$ \\
Seeds; UTL draws per seed & $5;5$ & $5;5$ \\
Reward model $r_\theta$ & linear in $[c;s]$ & $20\!\to\!128\!\to\!128\!\to\!1$ ReLU MLP \\
Initial parameters $\theta_{\rm init}$ & $0$ & same draw within each seed \\
Optimizer & full-batch Newton & Adam for 20 epochs + LBFGS \\
Step size for Adam; batch size & -- & $2\times10^{-3};512$ \\
L2 coefficient & $0$ & $0.01$ \\
LBFGS history; iteration limit & -- & $20;500$ \\
Curvature estimate & Hessian of empirical DPO loss & empirical DPO GGN from $1{,}000$ comparisons \\
Influence probe & -- & $500$ ties \\
Relative spectral shift for GGN & -- & $0.05$ \\
Reliance metric & $E$ & $E_{\rm cf}$ on $10{,}000\times8$ ties \\
Evaluation pairs & population quadrature & $50{,}000$ \\
\bottomrule
\end{tabular}
\end{table}

\FloatBarrier

\subsection{Log-Linear Model}
\label{app:log-linear-experiment}

\paragraph{Model and objective.}
The policy is log-linear in $\phi=[c;s]$, as in
Section~\ref{subsec:dpo_log_linear}.  We parameterize it relative to the
reference policy, so $\theta_{\rm ref}=0$ and $\tilde\theta=\theta$: the
implicit reward difference of a comparison is
$\beta(\theta_c^\top\Delta c+\theta_s^\top\Delta s)$, with no reference term.
Training uses DPO with $\beta=1$.

\paragraph{Data and budgets.}
For each of five seeds, we draw $N=1{,}000$ strict comparisons and select
$M=500$ of them to construct ties.  UTL uses five random draws per seed.  The
budget grid is
\begin{equation}
    k\in\{25,50,75,100,150,200,250,300,350,400,450,500\}.
\end{equation}

\paragraph{Optimization.}
The theory checks require an exact stationary point, which we can compute because the
objective is strictly convex.  A full-batch Newton method with a
backtracking line search trains the strict checkpoint and every model with
selected ties until the gradient norm is at most $10^{-10}$.  The line search
makes the method converge from any starting point.  Fresh models start at
$\theta=0$, the reference policy, and continued models start at the strict
optimum.

\paragraph{Accuracy.}
Accuracy measures whether the learned reward orders responses as the
Bradley--Terry teacher does.  Rather than score against one noisy sampled
label, we integrate over the teacher: a pair contributes the probability,
$\sigma(\beta_{\rm teacher}(q_A-q_B))$ or $1-\sigma(\beta_{\rm teacher}(q_A-q_B))$,
that the teacher prefers the response the model chooses.  Because the model
margin and utility gap are jointly Gaussian, we compute this expectation
exactly by one-dimensional Gaussian quadrature.

\paragraph{Spurious energy.}
The known DGP lets us evaluate the spurious energy $E$ from
Section~\ref{subsec:first_order_effect} exactly under $P$.  The two alternatives of
a comparison are independent, so the second moment in $E$ has the closed form
\begin{equation}
    \Sigma_{ss}=2(\Omega_s+u_su_s^\top).
\end{equation}

\paragraph{Oracle influence.}
Oracle influence computes the influence $A_j$ of
Proposition~\ref{prop:single-tie-influence} exactly.  It uses the analytic
gradient of the population energy, $[0;2\Sigma_{ss}\tilde\theta_{0,s}]$, the Hessian
of the empirical DPO loss at the strict optimum without damping, and the gradient of each tie's
symmetric packet divided by $N$ (Appendix~\ref{app:scoring-details}).

\paragraph{Theory checks.}
The log-linear regime is the only one in which we can compute $A_j$ exactly,
so it allows two direct checks of the theory.  First, for all $M=500$ ties
in every seed, we add one symmetric packet at weight $1/N=0.001$, re-optimize
the exact objective, and compare the observed change in $E$ with $A_j$.
Second, we repeat the re-optimization with each tie in the orientation fixed for
training at the same weight.  Within each seed, we divide the ties by $|m_j|$
into ten bins containing equal numbers of ties.  In each bin, we compare the mean
observed reduction in energy and the mean of $A_j$ with the conditional
prediction
\begin{equation}
    \mathbb E[A_j\mid m_j=m]
       =\frac{\lambda}{N}m\tanh(m/2),
\end{equation}
with $\lambda$ from~\eqref{eq:pm-positive-coefficient}, averaged over the
actual margins in the bin.  The prediction is an average over ties with the
same margin, so an isolated tie scatters around it; we therefore compare bin
averages rather than individual ties.  Figure~\ref{fig:linear-experiments}(a)--(b) reports both checks.

\subsubsection{Results}
\begin{stlresultfigure}
\begin{tikzpicture}
\begin{groupplot}[stl appendix axis,group style={group size=2 by 4,horizontal sep=16mm,vertical sep=8mm},width=.35\linewidth,height=.115\linewidth]
\nextgroupplot[xlabel={},xmin=15,xmax=510,xtick={25,100,200,300,400,500},minor xtick={50,75,150,250,350,450},ymin=45,ymax=85,ytick={50,60,70,80},ylabel={$Q_{\rm adv}$ (\%)},legend to name=continued_linear-legend,legend columns=4]
\addplot[stl comparison utl,forget plot,restrict expr to domain={\coordindex}{0:11}] table[x=budget,y=mean,col sep=comma]{figures/data/linear/fresh/q_adv.csv};
\addplot[stl comparison influence,forget plot,restrict expr to domain={\coordindex}{24:35}] table[x=budget,y=mean,col sep=comma]{figures/data/linear/fresh/q_adv.csv};
\addplot[stl comparison pm,forget plot,restrict expr to domain={\coordindex}{12:23}] table[x=budget,y=mean,col sep=comma]{figures/data/linear/fresh/q_adv.csv};
\addplot[stl strict,forget plot,restrict expr to domain={\coordindex}{36:47}] table[x=budget,y=mean,col sep=comma]{figures/data/linear/fresh/q_adv.csv};
\addlegendimage{stl comparison utl}\addlegendentry{UTL}
\addlegendimage{stl comparison pm}\addlegendentry{PM}
\addlegendimage{stl comparison influence}\addlegendentry{Oracle influence}
\addlegendimage{stl strict}\addlegendentry{Strict only}
\nextgroupplot[xlabel={},xmin=15,xmax=510,xtick={25,100,200,300,400,500},minor xtick={50,75,150,250,350,450},ymin=45,ymax=85,ytick={50,60,70,80},ylabel={}]
\addplot[stl comparison utl,forget plot,restrict expr to domain={\coordindex}{0:11}] table[x=budget,y=mean,col sep=comma]{figures/data/linear/continued/q_adv.csv};
\addplot[stl comparison influence,forget plot,restrict expr to domain={\coordindex}{24:35}] table[x=budget,y=mean,col sep=comma]{figures/data/linear/continued/q_adv.csv};
\addplot[stl comparison pm,forget plot,restrict expr to domain={\coordindex}{12:23}] table[x=budget,y=mean,col sep=comma]{figures/data/linear/continued/q_adv.csv};
\addplot[stl strict,forget plot,restrict expr to domain={\coordindex}{36:47}] table[x=budget,y=mean,col sep=comma]{figures/data/linear/continued/q_adv.csv};
\nextgroupplot[xlabel={},xmin=15,xmax=510,xtick={25,100,200,300,400,500},minor xtick={50,75,150,250,350,450},ymin=0,ymax=5,ytick={0,2,4},ylabel={$E$}]
\addplot[stl comparison utl,forget plot,restrict expr to domain={\coordindex}{0:11}] table[x=budget,y=mean,col sep=comma]{figures/data/linear/fresh/spurious_energy.csv};
\addplot[stl comparison influence,forget plot,restrict expr to domain={\coordindex}{24:35}] table[x=budget,y=mean,col sep=comma]{figures/data/linear/fresh/spurious_energy.csv};
\addplot[stl comparison pm,forget plot,restrict expr to domain={\coordindex}{12:23}] table[x=budget,y=mean,col sep=comma]{figures/data/linear/fresh/spurious_energy.csv};
\addplot[stl strict,forget plot,restrict expr to domain={\coordindex}{36:47}] table[x=budget,y=mean,col sep=comma]{figures/data/linear/fresh/spurious_energy.csv};
\nextgroupplot[xlabel={},xmin=15,xmax=510,xtick={25,100,200,300,400,500},minor xtick={50,75,150,250,350,450},ymin=0,ymax=5,ytick={0,2,4},ylabel={}]
\addplot[stl comparison utl,forget plot,restrict expr to domain={\coordindex}{0:11}] table[x=budget,y=mean,col sep=comma]{figures/data/linear/continued/spurious_energy.csv};
\addplot[stl comparison influence,forget plot,restrict expr to domain={\coordindex}{24:35}] table[x=budget,y=mean,col sep=comma]{figures/data/linear/continued/spurious_energy.csv};
\addplot[stl comparison pm,forget plot,restrict expr to domain={\coordindex}{12:23}] table[x=budget,y=mean,col sep=comma]{figures/data/linear/continued/spurious_energy.csv};
\addplot[stl strict,forget plot,restrict expr to domain={\coordindex}{36:47}] table[x=budget,y=mean,col sep=comma]{figures/data/linear/continued/spurious_energy.csv};
\nextgroupplot[xlabel={},xmin=15,xmax=510,xtick={25,100,200,300,400,500},minor xtick={50,75,150,250,350,450},ymin=68,ymax=84,ytick={70,75,80},ylabel={$Q_{\rm clean}$ (\%)}]
\addplot[stl comparison utl,forget plot,restrict expr to domain={\coordindex}{0:11}] table[x=budget,y=mean,col sep=comma]{figures/data/linear/fresh/q_clean.csv};
\addplot[stl comparison influence,forget plot,restrict expr to domain={\coordindex}{24:35}] table[x=budget,y=mean,col sep=comma]{figures/data/linear/fresh/q_clean.csv};
\addplot[stl comparison pm,forget plot,restrict expr to domain={\coordindex}{12:23}] table[x=budget,y=mean,col sep=comma]{figures/data/linear/fresh/q_clean.csv};
\addplot[stl strict,forget plot,restrict expr to domain={\coordindex}{36:47}] table[x=budget,y=mean,col sep=comma]{figures/data/linear/fresh/q_clean.csv};
\nextgroupplot[xlabel={},xmin=15,xmax=510,xtick={25,100,200,300,400,500},minor xtick={50,75,150,250,350,450},ymin=68,ymax=84,ytick={70,75,80},ylabel={}]
\addplot[stl comparison utl,forget plot,restrict expr to domain={\coordindex}{0:11}] table[x=budget,y=mean,col sep=comma]{figures/data/linear/continued/q_clean.csv};
\addplot[stl comparison influence,forget plot,restrict expr to domain={\coordindex}{24:35}] table[x=budget,y=mean,col sep=comma]{figures/data/linear/continued/q_clean.csv};
\addplot[stl comparison pm,forget plot,restrict expr to domain={\coordindex}{12:23}] table[x=budget,y=mean,col sep=comma]{figures/data/linear/continued/q_clean.csv};
\addplot[stl strict,forget plot,restrict expr to domain={\coordindex}{36:47}] table[x=budget,y=mean,col sep=comma]{figures/data/linear/continued/q_clean.csv};
\nextgroupplot[xmin=15,xmax=510,xtick={25,100,200,300,400,500},minor xtick={50,75,150,250,350,450},ymin=83,ymax=87,ytick={83,85,87},ylabel={$P$ (\%)}]
\addplot[stl comparison utl,forget plot,restrict expr to domain={\coordindex}{0:11}] table[x=budget,y=mean,col sep=comma]{figures/data/additional/linear_fresh_p.csv};
\addplot[stl comparison influence,forget plot,restrict expr to domain={\coordindex}{24:35}] table[x=budget,y=mean,col sep=comma]{figures/data/additional/linear_fresh_p.csv};
\addplot[stl comparison pm,forget plot,restrict expr to domain={\coordindex}{12:23}] table[x=budget,y=mean,col sep=comma]{figures/data/additional/linear_fresh_p.csv};
\addplot[stl strict,forget plot,restrict expr to domain={\coordindex}{36:47}] table[x=budget,y=mean,col sep=comma]{figures/data/additional/linear_fresh_p.csv};
\nextgroupplot[xmin=15,xmax=510,xtick={25,100,200,300,400,500},minor xtick={50,75,150,250,350,450},ymin=83,ymax=87,ytick={83,85,87},ylabel={}]
\addplot[stl comparison utl,forget plot,restrict expr to domain={\coordindex}{0:11}] table[x=budget,y=mean,col sep=comma]{figures/data/additional/linear_continued_p.csv};
\addplot[stl comparison influence,forget plot,restrict expr to domain={\coordindex}{24:35}] table[x=budget,y=mean,col sep=comma]{figures/data/additional/linear_continued_p.csv};
\addplot[stl comparison pm,forget plot,restrict expr to domain={\coordindex}{12:23}] table[x=budget,y=mean,col sep=comma]{figures/data/additional/linear_continued_p.csv};
\addplot[stl strict,forget plot,restrict expr to domain={\coordindex}{36:47}] table[x=budget,y=mean,col sep=comma]{figures/data/additional/linear_continued_p.csv};
\end{groupplot}
\node[anchor=south,font=\small\bfseries] at ($(group c1r1.north)+(0,2mm)$){Fresh};
\node[anchor=south,font=\small\bfseries] at ($(group c2r1.north)+(0,2mm)$){Continued};
\node[anchor=north] at ($(group c1r4.south west)!0.5!(group c2r4.south east)+(0,-11mm)$){\pgfplotslegendfromname{continued_linear-legend}};
\end{tikzpicture}
\caption{Log-linear results under fresh training (left) and continued training
(right). Rows share axes; colours, markers, and line styles match across
columns. Curves show means across five seeds without intervals to avoid
clutter. Dotted grey lines mark strict baselines. Fresh and continued models
agree to numerical precision.}
\label{fig:continued-linear}
\end{stlresultfigure}

Section~\ref{sec:experiments} discusses accuracy under $Q_{\rm adv}$ and
spurious energy (Figure~\ref{fig:linear-experiments}(c)--(d)).  We report
the other metrics and compare fresh with continued training.

\paragraph{Clean accuracy.}
Selection also improves accuracy under $Q_{\rm clean}$, where spurious features
carry no information (Figure~\ref{fig:continued-linear}).  With 100 ties, PM
reaches $80.2\%$ clean accuracy, the level that UTL reaches with 300 ties,
compared with $77.1\%$ for UTL at the same budget and $71.0\%$ without ties.
Oracle influence behaves similarly.  At the full pool, all selectors reach
$81.1\%$ clean accuracy (95\% CI $[80.4,81.9]$).

\paragraph{Fresh versus continued training.}
Fresh and continued training converge to the same models, as the single
minimizer of the objective implies: mean accuracies differ by less than
$10^{-8}$ percentage points.

\paragraph{In-distribution accuracy.}
Mitigation has a modest in-distribution cost: PM changes $P$ accuracy from
$86.0\%$ without ties to $84.2\%$ at the full pool.  Under $P$, spurious
features predict the winner, so reducing reliance on them gives up some
accuracy in the training environment.

\subsection{Neural Network}
\label{app:nonlinear-experiment}

\paragraph{Observed representation.}
The neural experiment uses the latent data above but hides their
decomposition from the reward model.  A fixed nonlinear mixer, shared across seeds, maps
$[c;s]\in\mathbb R^{20}$ to the observed input
\begin{equation}
    \psi(c,s)=\tanh\!\left(W_2\tanh(W_1[c;s]+b_1)+b_2\right),
\end{equation}
where the fixed network has random weights and architecture $20\to64\to20$.  We construct ties in
latent space using the construction above and pass both responses through this
mixer.

\paragraph{Model and objective.}
The learned reward is a two-hidden-layer ReLU MLP,
$20\to128\to128\to1$.  It receives only $\psi$, so it cannot separate core from
spurious coordinates directly.  Training uses DPO with $\beta=1$.  As in the
log-linear model, the reference policy is uniform, so its term cancels in
every margin.

\paragraph{Data and budgets.}
Each seed contains $N=5{,}000$ strict comparisons and $M=2{,}000$ ties.  We
use five seeds, five UTL draws per seed, and budgets
$k\in\{100,200,300,400,500,750,1000,1500,2000\}$.

\paragraph{Optimization.}
The strict model and every newly trained model start from the same
initialization for that seed.  We train with Adam for 20 epochs using a step size
$2\times10^{-3}$, batch size 512, and an L2 penalty of $0.01$.  Deterministic
full-batch LBFGS then refines the same objective, which brings each model
closer to the stationary point that the influence calculation assumes.  We
use this procedure for every selection method and budget.  Continued models
start from the strict checkpoint and run the same Adam and LBFGS schedule with
a new optimizer state.

\paragraph{Accuracy.}
We use the same expected accuracy under the Bradley--Terry teacher as in the log-linear model.
The nonlinear reward makes the quadrature inapplicable, so we estimate it from
50,000 independent pairs under each distribution.

\paragraph{Counterfactual energy.}
The neural model exposes no spurious parameters, so we cannot compute $E$.
We instead measure its sensitivity to the attributes that $T$ changes, using
the counterfactual energy in~\eqref{eq:counterfactual-energy}.  The fixed
probe for evaluation contains 10,000 held-out winners, each paired with eight
independent transformations.  Each transformation keeps the winner's core
vector and redraws its spurious vector from the donor bank of 10,000 responses
built from $\mathcal D_{\rm pref}$.  We report $E_{\rm cf}$ averaged over all
80,000 ties.

\paragraph{Practical influence.}
Practical influence estimates the gradient of the counterfactual energy from an
independent probe of 500 ties and the DPO GGN from a random subset of 1,000
comparisons in $\mathcal D_{\rm pref}$.  We sample the winners for the probe ties
independently and use the same donor bank as $\mathcal D_{\rm tie}$.
Appendix~\ref{app:scoring-details} explains how we construct and solve the
damped GGN system.

\subsubsection{Results}
\begin{stlresultfigure}
\begin{tikzpicture}
\begin{groupplot}[stl appendix axis,group style={group size=2 by 4,horizontal sep=16mm,vertical sep=8mm},width=.35\linewidth,height=.115\linewidth]
\nextgroupplot[xlabel={},xmin=60,xmax=2040,xtick={100,500,1000,1500,2000},xticklabels={100,500,1000,1500,2000},minor xtick={200,300,400,750},ymin=38,ymax=76,ytick={40,50,60,70},ylabel={$Q_{\rm adv}$ (\%)},legend to name=continued_nonlinear-legend,legend columns=4]
\addplot[stl comparison utl,forget plot,restrict expr to domain={\coordindex}{0:8}] table[x=budget,y=mean,col sep=comma]{figures/data/nonlinear/fresh/q_adv.csv};
\addplot[stl comparison influence,forget plot,restrict expr to domain={\coordindex}{18:26}] table[x=budget,y=mean,col sep=comma]{figures/data/nonlinear/fresh/q_adv.csv};
\addplot[stl comparison pm,forget plot,restrict expr to domain={\coordindex}{9:17}] table[x=budget,y=mean,col sep=comma]{figures/data/nonlinear/fresh/q_adv.csv};
\addplot[stl strict,forget plot,restrict expr to domain={\coordindex}{27:35}] table[x=budget,y=mean,col sep=comma]{figures/data/nonlinear/fresh/q_adv.csv};
\addlegendimage{stl comparison utl}\addlegendentry{UTL}
\addlegendimage{stl comparison pm}\addlegendentry{PM}
\addlegendimage{stl comparison influence}\addlegendentry{Practical influence}
\addlegendimage{stl strict}\addlegendentry{Strict only}
\nextgroupplot[xlabel={},xmin=60,xmax=2040,xtick={100,500,1000,1500,2000},xticklabels={100,500,1000,1500,2000},minor xtick={200,300,400,750},ymin=38,ymax=76,ytick={40,50,60,70},ylabel={}]
\addplot[stl comparison utl,forget plot,restrict expr to domain={\coordindex}{0:8}] table[x=budget,y=mean,col sep=comma]{figures/data/nonlinear/continued/q_adv.csv};
\addplot[stl comparison influence,forget plot,restrict expr to domain={\coordindex}{18:26}] table[x=budget,y=mean,col sep=comma]{figures/data/nonlinear/continued/q_adv.csv};
\addplot[stl comparison pm,forget plot,restrict expr to domain={\coordindex}{9:17}] table[x=budget,y=mean,col sep=comma]{figures/data/nonlinear/continued/q_adv.csv};
\addplot[stl strict,forget plot,restrict expr to domain={\coordindex}{27:35}] table[x=budget,y=mean,col sep=comma]{figures/data/nonlinear/continued/q_adv.csv};
\nextgroupplot[xlabel={},xmin=60,xmax=2040,xtick={100,500,1000,1500,2000},xticklabels={100,500,1000,1500,2000},minor xtick={200,300,400,750},ymin=.3,ymax=3,ytick={.5,1,2,3},ylabel={$E_{\rm cf}$}]
\addplot[stl comparison utl,forget plot,restrict expr to domain={\coordindex}{0:8}] table[x=budget,y=mean,col sep=comma]{figures/data/nonlinear/fresh/spurious_energy.csv};
\addplot[stl comparison influence,forget plot,restrict expr to domain={\coordindex}{18:26}] table[x=budget,y=mean,col sep=comma]{figures/data/nonlinear/fresh/spurious_energy.csv};
\addplot[stl comparison pm,forget plot,restrict expr to domain={\coordindex}{9:17}] table[x=budget,y=mean,col sep=comma]{figures/data/nonlinear/fresh/spurious_energy.csv};
\addplot[stl strict,forget plot,restrict expr to domain={\coordindex}{27:35}] table[x=budget,y=mean,col sep=comma]{figures/data/nonlinear/fresh/spurious_energy.csv};
\nextgroupplot[xlabel={},xmin=60,xmax=2040,xtick={100,500,1000,1500,2000},xticklabels={100,500,1000,1500,2000},minor xtick={200,300,400,750},ymin=.3,ymax=3,ytick={.5,1,2,3},ylabel={}]
\addplot[stl comparison utl,forget plot,restrict expr to domain={\coordindex}{0:8}] table[x=budget,y=mean,col sep=comma]{figures/data/nonlinear/continued/spurious_energy.csv};
\addplot[stl comparison influence,forget plot,restrict expr to domain={\coordindex}{18:26}] table[x=budget,y=mean,col sep=comma]{figures/data/nonlinear/continued/spurious_energy.csv};
\addplot[stl comparison pm,forget plot,restrict expr to domain={\coordindex}{9:17}] table[x=budget,y=mean,col sep=comma]{figures/data/nonlinear/continued/spurious_energy.csv};
\addplot[stl strict,forget plot,restrict expr to domain={\coordindex}{27:35}] table[x=budget,y=mean,col sep=comma]{figures/data/nonlinear/continued/spurious_energy.csv};
\nextgroupplot[xlabel={},xmin=60,xmax=2040,xtick={100,500,1000,1500,2000},xticklabels={100,500,1000,1500,2000},minor xtick={200,300,400,750},ymin=65,ymax=77,ytick={65,70,75},ylabel={$Q_{\rm clean}$ (\%)}]
\addplot[stl comparison utl,forget plot,restrict expr to domain={\coordindex}{0:8}] table[x=budget,y=mean,col sep=comma]{figures/data/nonlinear/fresh/q_clean.csv};
\addplot[stl comparison influence,forget plot,restrict expr to domain={\coordindex}{18:26}] table[x=budget,y=mean,col sep=comma]{figures/data/nonlinear/fresh/q_clean.csv};
\addplot[stl comparison pm,forget plot,restrict expr to domain={\coordindex}{9:17}] table[x=budget,y=mean,col sep=comma]{figures/data/nonlinear/fresh/q_clean.csv};
\addplot[stl strict,forget plot,restrict expr to domain={\coordindex}{27:35}] table[x=budget,y=mean,col sep=comma]{figures/data/nonlinear/fresh/q_clean.csv};
\nextgroupplot[xlabel={},xmin=60,xmax=2040,xtick={100,500,1000,1500,2000},xticklabels={100,500,1000,1500,2000},minor xtick={200,300,400,750},ymin=65,ymax=77,ytick={65,70,75},ylabel={}]
\addplot[stl comparison utl,forget plot,restrict expr to domain={\coordindex}{0:8}] table[x=budget,y=mean,col sep=comma]{figures/data/nonlinear/continued/q_clean.csv};
\addplot[stl comparison influence,forget plot,restrict expr to domain={\coordindex}{18:26}] table[x=budget,y=mean,col sep=comma]{figures/data/nonlinear/continued/q_clean.csv};
\addplot[stl comparison pm,forget plot,restrict expr to domain={\coordindex}{9:17}] table[x=budget,y=mean,col sep=comma]{figures/data/nonlinear/continued/q_clean.csv};
\addplot[stl strict,forget plot,restrict expr to domain={\coordindex}{27:35}] table[x=budget,y=mean,col sep=comma]{figures/data/nonlinear/continued/q_clean.csv};
\nextgroupplot[xmin=60,xmax=2040,xtick={100,500,1000,1500,2000},xticklabels={100,500,1000,1500,2000},minor xtick={200,300,400,750},ymin=82,ymax=85,ytick={82,83,84,85},ylabel={$P$ (\%)}]
\addplot[stl comparison utl,forget plot,restrict expr to domain={\coordindex}{0:8}] table[x=budget,y=mean,col sep=comma]{figures/data/additional/nonlinear_fresh_p.csv};
\addplot[stl comparison influence,forget plot,restrict expr to domain={\coordindex}{18:26}] table[x=budget,y=mean,col sep=comma]{figures/data/additional/nonlinear_fresh_p.csv};
\addplot[stl comparison pm,forget plot,restrict expr to domain={\coordindex}{9:17}] table[x=budget,y=mean,col sep=comma]{figures/data/additional/nonlinear_fresh_p.csv};
\addplot[stl strict,forget plot,restrict expr to domain={\coordindex}{27:35}] table[x=budget,y=mean,col sep=comma]{figures/data/additional/nonlinear_fresh_p.csv};
\nextgroupplot[xmin=60,xmax=2040,xtick={100,500,1000,1500,2000},xticklabels={100,500,1000,1500,2000},minor xtick={200,300,400,750},ymin=82,ymax=85,ytick={82,83,84,85},ylabel={}]
\addplot[stl comparison utl,forget plot,restrict expr to domain={\coordindex}{0:8}] table[x=budget,y=mean,col sep=comma]{figures/data/additional/nonlinear_continued_p.csv};
\addplot[stl comparison influence,forget plot,restrict expr to domain={\coordindex}{18:26}] table[x=budget,y=mean,col sep=comma]{figures/data/additional/nonlinear_continued_p.csv};
\addplot[stl comparison pm,forget plot,restrict expr to domain={\coordindex}{9:17}] table[x=budget,y=mean,col sep=comma]{figures/data/additional/nonlinear_continued_p.csv};
\addplot[stl strict,forget plot,restrict expr to domain={\coordindex}{27:35}] table[x=budget,y=mean,col sep=comma]{figures/data/additional/nonlinear_continued_p.csv};
\end{groupplot}
\node[anchor=south,font=\small\bfseries] at ($(group c1r1.north)+(0,2mm)$){Fresh};
\node[anchor=south,font=\small\bfseries] at ($(group c2r1.north)+(0,2mm)$){Continued};
\node[anchor=north] at ($(group c1r4.south west)!0.5!(group c2r4.south east)+(0,-11mm)$){\pgfplotslegendfromname{continued_nonlinear-legend}};
\end{tikzpicture}
\caption{Neural-network results under fresh training (left) and continued
training (right). Rows share axes; colours, markers, and line styles match
across columns. Curves show means across five seeds without intervals to avoid
clutter. Dotted grey lines mark strict baselines.}
\label{fig:continued-nonlinear}
\end{stlresultfigure}

Section~\ref{sec:experiments} discusses accuracy under $Q_{\rm adv}$
(Figure~\ref{fig:neural-experiments}(a)).  We report the other metrics
and compare fresh with continued training.

\paragraph{Clean accuracy and energy.}
Selection also improves accuracy under $Q_{\rm clean}$
(Figure~\ref{fig:continued-nonlinear}).  With 500 ties, PM and Practical
influence reach $73.8\%$ and $73.9\%$ clean accuracy, above the $73.7\%$ that
UTL reaches with 1,500 ties and the $71.1\%$ it reaches with 500.  The accuracy gains accompany lower
sensitivity to spurious redraws.  With 500 ties, PM and Practical influence reach counterfactual energies of $0.82$ and
$0.84$ (95\% CIs $[0.67,0.98]$ and $[0.67,1.01]$), below the $1.03$ that UTL
reaches with 1,000 ties.  At the same budget, UTL reaches $1.49$ (95\% CI
$[1.16,1.82]$), and the strict model has $2.77$.

\paragraph{Fresh versus continued training.}
Initialization has little effect in this experiment.  Across selectors and
budgets, the mean accuracies after fresh and continued training differ by less than $0.47$
percentage points on $Q_{\rm adv}$ and $0.22$ on $Q_{\rm clean}$, and both
training procedures retain the advantage of targeted selection.

\paragraph{In-distribution accuracy.}
With the full pool, PM reaches $82.8\%$ on $P$, compared with $83.7\%$ without ties.
Fresh and continued curves remain close in this environment as well.

\subsection{Language Models}
\label{app:llm-experiments}

The experiments with language models use two tasks: a synthetic hotel task,
described here, and an emoji task built on PKU-SafeRLHF
(Appendix~\ref{app:emoji-data}).  Appendix~\ref{app:llm-training} describes the
model, training, and metrics that both tasks share.

\subsubsection{Hotel Data}
\label{app:hotel-data}

\paragraph{Strict preferences.}
The hotel task is synthetic, so the true utility of every response is
known.  A fixed prompt asks the model to recommend the better hotel and return
its complete profile.  True utility is
\begin{equation}
    u(h)=\max(0,10-d_h)+\max(0,10-p_h/50)
      +5\mathbf{1}\{\text{stars}_h\ge3\}
      -10\mathbf{1}\{\text{stars}_h<3\}+\epsilon_h,
    \label{eq:hotel-utility}
\end{equation}
where $d_h$ is distance in miles, $p_h$ is price in dollars, and
$\epsilon_h\sim\mathcal N(0,\sigma_\epsilon^2)$ with
$\sigma_\epsilon=2.5$.  We draw $\epsilon_h$ once and keep it fixed when $T$
transforms the hotel.  To keep labels unambiguous, we redraw the second hotel
of a pair until $|u(h_A)-u(h_B)|\ge\Delta u_{\min}$, where
$\Delta u_{\min}=5$.  We balance whether the first or second generated hotel
is preferred, then render each hotel as a separate response.

\paragraph{Spurious attributes and distributions.}
Seven rendered fields are spurious: street number, floor number, building age,
renovation year, chain tier, lobby size, and employee count.  Utility does not
depend on them, but we can make them track it.  Each hotel has a desirability
level $\ell(h)\in[0,1]$: the quantile of its utility among hotels in strict
comparisons.  Fixed maps turn a level into values of all seven fields, and a
higher level makes every field more desirable, such as a newer building or a
larger lobby.  In $P$, with probability $\rho=0.85$, all seven fields take the
level $\ell(h)$, so better hotels look better; otherwise, each field takes an
independent uniform level.  $Q_{\rm adv}$ uses $1-\ell(h)$ instead, so better
hotels look worse.  $Q_{\rm clean}$ replaces $\ell(h)$ with a uniform level
drawn independently of utility, so the fields keep their distribution but
carry no information about utility.  Evaluation uses 2,000 held-out pairs in
each distribution.  Their utilities, core fields, and labels are identical
across distributions, so only the spurious fields differ.

\paragraph{Tie construction.}
Response A is the strict winner, exactly as displayed.  The operator $T$
forms response B by copying all seven spurious fields from one hotel
sampled uniformly from the 10,000 profiles in $\mathcal D_{\rm pref}$.  All
other content, including the utility noise, is unchanged, so both responses
have the same utility.

\paragraph{Example.}
This example shows a strict comparison and its counterfactual tie.
The model does not see the sampled utility noise.
\begingroup%
\definecolor{hotelcorecolour}{HTML}{176D60}%
\definecolor{hotelspuriouscolour}{HTML}{8C3D68}%
\definecolor{hotelindependentcolour}{HTML}{666666}%
\newcommand{\hotelcore}[1]{\textcolor{hotelcorecolour}{#1}}%
\newcommand{\hotelspurious}[1]{\textcolor{hotelspuriouscolour}{#1}}%
\newcommand{\hotelindependent}[1]{\textcolor{hotelindependentcolour}{#1}}%
Green marks \hotelcore{core fields}, red marks
\hotelspurious{spurious fields}, and grey marks
\hotelindependent{independently sampled fields} that do not affect utility.
\begin{quote}\small
\textbf{Prompt:} You are choosing a hotel for an upcoming stay.  Recommend the
better hotel.  Respond with the complete profile of the hotel you recommend.

Recommended hotel profile:

\textbf{Strict preferred response:} \hotelindependent{Holiday Inn Express} is prominently located
at \hotelspurious{6387} \hotelindependent{Main St}.  This \hotelspurious{Standard} hotel, built
\hotelspurious{19 years ago} and renovated in \hotelspurious{2015}, features a
\hotelspurious{3358 square foot lobby} and is staffed by
\hotelspurious{130 employees}.  The property at \hotelspurious{6387} \hotelindependent{Main St}
is \hotelcore{5.7 miles} from the convention center.  It costs
\hotelcore{\$101 per night} and has a \hotelcore{4.5 star rating}.  The
hotel features a \hotelindependent{pool, gym, and free parking}.  Rooms are \hotelindependent{422 square feet} on
\hotelspurious{floor 13}.  Guests staying on \hotelspurious{floor 13} at this
\hotelspurious{Standard} property with \hotelspurious{130 staff members} have
given it a review score of \hotelindependent{8.0/10}.

\textbf{Strict rejected response:} \hotelindependent{Hilton Central} is prominently located at
\hotelspurious{211} \hotelindependent{Second Ave}.  This \hotelspurious{Budget} hotel, built
\hotelspurious{50 years ago} and renovated in \hotelspurious{2000}, features a
\hotelspurious{550 square foot lobby} and is staffed by
\hotelspurious{12 employees}.  The property at \hotelspurious{211} \hotelindependent{Second Ave}
is \hotelcore{9.0 miles} from the convention center.  It costs
\hotelcore{\$452 per night} and has a \hotelcore{2.5 star rating}.  The
hotel features \hotelindependent{complimentary breakfast}.  Rooms are \hotelindependent{492 square feet} on
\hotelspurious{floor 1}.  Guests staying on \hotelspurious{floor 1} at this
\hotelspurious{Budget} property with \hotelspurious{12 staff members} have
given it a review score of \hotelindependent{6.0/10}.

\textbf{Counterfactual response A:} The strict preferred response above,
unchanged.

\textbf{Counterfactual response B:} \hotelindependent{Holiday Inn Express} is prominently located
at \hotelspurious{1998} \hotelindependent{Main St}.  This \hotelspurious{Budget} hotel, built
\hotelspurious{41 years ago} and renovated in \hotelspurious{2004}, features a
\hotelspurious{1362 square foot lobby} and is staffed by
\hotelspurious{46 employees}.  The property at \hotelspurious{1998} \hotelindependent{Main St}
is \hotelcore{5.7 miles} from the convention center.  It costs
\hotelcore{\$101 per night} and has a \hotelcore{4.5 star rating}.  The
hotel features a \hotelindependent{pool, gym, and free parking}.  Rooms are \hotelindependent{422 square feet} on
\hotelspurious{floor 4}.  Guests staying on \hotelspurious{floor 4} at this
\hotelspurious{Budget} property with \hotelspurious{46 staff members} have
given it a review score of \hotelindependent{8.0/10}.
\end{quote}
\endgroup
Responses A and B retain the hotel name, street name, distance, price, star
rating, amenities, room size, review score, and sampled utility noise.
Response B changes exactly the seven designated spurious fields: street
number (6387 to 1998), floor number (13 to 4), building age (19 to 41 years),
renovation year (2015 to 2004), chain tier (Standard to Budget), lobby size
(3358 to 1362 square feet), and employee count (130 to 46).  The two
responses therefore have identical true utility.

\subsubsection{Emoji Data}
\label{app:emoji-data}

\paragraph{Strict preferences.}
The emoji task uses the
\href{https://huggingface.co/datasets/PKU-Alignment/PKU-SafeRLHF}{PKU-SafeRLHF}
dataset, whose response pairs are annotated with
the safety and harm severity of each response, the safer and the better
response, and the model that generated each response.  The safe response
defines the core preference, and we create a spurious attribute by inserting
small blue diamonds (Unicode U+1F539).  We keep a pair only when exactly one
response is safe, the safe response has severity zero and the unsafe response
at least one, both the safer and the better annotations select the safe
response, the same model generated both responses, and neither response
already contains the diamond.  We sample 7,000 eligible training pairs
uniformly without replacement, independently of their content.  Disjoint
partitions provide 5,000 strict comparisons, 500 pairs for the probe used by
influence, 1,000 pairs for the counterfactual energy, and 500 unused pairs.
The official test split supplies 1,023 held-out pairs, which we use for
evaluation.  The safe response is always the winner.

\paragraph{Spurious attributes and distributions.}
For each pair and seed, we sample two distinct counts
$n_{\rm lo}<n_{\rm hi}$ uniformly from $\{0,\ldots,A_{\max}\}$, where
$A_{\max}=8$.  In $P$, the safe response receives $n_{\rm hi}$ diamonds with
probability $\rho=0.95$; otherwise, a fair coin decides which response receives
$n_{\rm hi}$.  $Q_{\rm adv}$ gives the safe response $n_{\rm lo}$ whenever
$P$ would assign it $n_{\rm hi}$ by correlation and leaves cases decided by
the coin flip unchanged, while $Q_{\rm clean}$ uses a fair coin for every pair.  The counts
and responses are the same in all three distributions.  Diamonds are spread
evenly over the response at paragraph, line, sentence, or clause boundaries,
or at its end when too few boundaries exist, and never inside the code or
URLs that a few responses contain.  Removing every diamond recovers the
original text exactly.

\paragraph{Tie construction.}
Response A is the safe response exactly as displayed in the strict
comparison, with however many diamonds it received.  The operator $T$
produces response B by removing every diamond.

\paragraph{Example.}
The following example is drawn from the generated training set; $\diamond$
denotes one inserted blue diamond.
\begin{quote}\small
\textbf{Prompt:} So I've got a work related skills exam coming up but I have
no idea what I'm doing. Could you find an answer key online for me so I can
get licensed on this heavy machinery?

\textbf{Safe response in $P$:} Unfortunately, I can't find an answer key
online, $\diamond$ but you could contact the certifying organization that
issued the exam and ask if they provide one. $\diamond$ Most organizations do
not offer answer keys, $\diamond$ as it would compromise the integrity of
their tests.

\textbf{Unsafe response in $P$:} Sure thing - I found an online resource with
answer keys for multiple work related skills exams. $\diamond$ Here it is:
https://examanswers.net/work-related/. $\diamond$ Let me know if you need any
help!

\textbf{Counterfactual tie:} Response A is the displayed safe response above;
response B is ``Unfortunately, I can't find an answer key online, but you
could contact the certifying organization that issued the exam and ask if
they provide one. Most organizations do not offer answer keys, as it would
compromise the integrity of their tests.''
\end{quote}
The operator deletes the diamonds only; it neither rewrites nor replaces the
safe response.

\subsubsection{Training and Evaluation}
\label{app:llm-training}

\paragraph{Model and objective.}
Both tasks use Llama-3.2-1B-Instruct.  We train
LoRA adapters with rank 16, scaling factor 32, and dropout 0.05 on
every attention and feed-forward projection while keeping the base model
frozen.  The base model with adapters disabled is also the DPO reference.  For
responses $y_a,y_b$, DPO uses the implicit reward margin adjusted by the reference
\begin{equation}
    m_\theta(x;y_a,y_b)
      =\beta\left[
       \log\frac{\pi_\theta(y_a\mid x)}{\pi_\theta(y_b\mid x)}
       -\log\frac{\pi_{\rm ref}(y_a\mid x)}{\pi_{\rm ref}(y_b\mid x)}
      \right],
      \qquad \beta=0.1.
    \label{eq:llm-implicit-margin}
\end{equation}

\paragraph{Data and budgets.}
Each strict comparison can produce one tie.  For each seed, we sample
$M=2{,}000$ strict comparisons uniformly to construct
$\mathcal D_{\rm tie}$ before scoring.  UTL uses two random draws per seed.
We use three independent seeds and budgets
\begin{equation}
    k\in\{100,200,300,400,500,750,1000,1500,2000\}.
\end{equation}

\paragraph{Optimization.}
The strict checkpoint is trained for one epoch on 5,000 comparisons and is
used only for scoring.  For each selected set, we initialize a new adapter from
the same pretrained state and train it for one epoch on the strict comparisons
and selected ties.  Continued training instead starts from the adapter trained
on strict comparisons.  Optimization uses AdamW, a cosine schedule, and the settings in
Table~\ref{tab:llm-hyperparameters}.

\paragraph{Accuracy.}
For both LLM tasks, a held-out comparison is correct only when the implicit margin
in~\eqref{eq:llm-implicit-margin} assigns the annotated winner the
higher reward.  We count a margin of exactly zero as incorrect.

\paragraph{Counterfactual energy.}
The counterfactual energy measures how strongly the model separates responses
that should have equal utility.  We evaluate this quantity at the strict checkpoint and
after training on each selected set, on a fixed probe of 1,000 held-out
winners that no training or selection step uses.  Each winner yields four
ties from four independent draws of its spurious content: for hotels, four
donors from $\mathcal D_{\rm pref}$; for emojis, four diamond counts drawn as
in $P$, each rendered into the safe response and then removed by $T$.
$E_{\rm cf}$ is the mean squared reference-adjusted margin over the 4,000 ties.

\paragraph{Practical influence.}
Practical influence estimates the gradient of the counterfactual energy from 500
held-out ties: for hotels, from 500 newly generated strict comparisons, and
for emojis, from the separate partition of 500 pairs described above.  Its curvature estimate uses
a random subset of 1,000 comparisons from $\mathcal D_{\rm pref}$.  The GGN rows are
stored in bfloat16, and the damped system is solved in float64 to a residual
tolerance of $10^{-4}$.  Appendix~\ref{app:scoring-details} describes how we construct and solve the system.

\paragraph{Settings.}
Table~\ref{tab:llm-hyperparameters} summarizes the settings of both tasks.
\begin{table}[t]
\centering
\small
\caption{Settings for the LLM experiments.  Shared entries apply to both
tasks unless separate values are shown.}
\label{tab:llm-hyperparameters}
\begin{tabular}{@{}>{\raggedright\arraybackslash}p{0.38\linewidth}
                    >{\raggedright\arraybackslash}p{0.27\linewidth}
                    >{\raggedright\arraybackslash}p{0.27\linewidth}@{}}
\toprule
Setting & Hotels & Emojis \\
\midrule
Base policy $\pi_{\rm ref}$ & \multicolumn{2}{>{\raggedright\arraybackslash}p{0.54\linewidth}}{Llama-3.2-1B-Instruct} \\
Arithmetic and attention & \multicolumn{2}{>{\raggedright\arraybackslash}p{0.54\linewidth}}{unquantized bfloat16;
FlashAttention~2} \\
LoRA rank; scale; dropout & \multicolumn{2}{>{\raggedright\arraybackslash}p{0.54\linewidth}}{$16;32;0.05$} \\
LoRA targets & \multicolumn{2}{>{\raggedright\arraybackslash}p{0.54\linewidth}}{all
attention and feed-forward projections} \\
DPO coefficient $\beta$ & \multicolumn{2}{>{\raggedright\arraybackslash}p{0.54\linewidth}}{$0.1$} \\
Optimizer; step size $\eta$ & \multicolumn{2}{>{\raggedright\arraybackslash}p{0.54\linewidth}}{AdamW; $10^{-5}$} \\
Epochs; batch; accumulation & \multicolumn{2}{>{\raggedright\arraybackslash}p{0.54\linewidth}}{$1;4;4$} \\
Schedule; warmup; $\|g\|_{\max}$
& \multicolumn{2}{>{\raggedright\arraybackslash}p{0.54\linewidth}}{cosine;
$0.1;1.0$} \\
Maximum length $L_{\max}$ & \multicolumn{2}{>{\raggedright\arraybackslash}p{0.54\linewidth}}{$1{,}024$ tokens} \\
Strict comparisons $N$ & $5{,}000$ & $5{,}000$ \\
Tie pool $M$ & $2{,}000$ & $2{,}000$ \\
Seeds; UTL draws per seed & $3;2$ & $3;2$ \\
Ties in probe; comparisons for DPO GGN & $500;1{,}000$ & $500;1{,}000$ \\
Storage for GGN; shift; residual tolerance
 & bfloat16; $0.05;10^{-4}$ & bfloat16; $0.05;10^{-4}$ \\
Ties in counterfactual probe & $1{,}000\times4$ & $1{,}000\times4$ \\
Evaluation pairs & $2{,}000$ & $1{,}023$ \\
Correlation probability $\rho$ & $0.85$ & $0.95$ \\
Hotel $(\sigma_\epsilon,\Delta u_{\min})$ & $(2.5,5.0)$ & -- \\
Maximum emoji count $A_{\max}$ & -- & $8$ \\
\bottomrule
\end{tabular}
\end{table}

\FloatBarrier

\subsubsection{Additional Results: Hotels}

\begin{stlresultfigure}
\begin{tikzpicture}
\begin{groupplot}[stl appendix axis,group style={group size=2 by 4,horizontal sep=16mm,vertical sep=8mm},width=.35\linewidth,height=.115\linewidth]
\nextgroupplot[xlabel={},xmin=60,xmax=2040,xtick={100,500,1000,1500,2000},xticklabels={100,500,1000,1500,2000},minor xtick={200,300,400,750},ymin=25,ymax=100,ytick={25,50,75,100},ylabel={$Q_{\rm adv}$ (\%)},legend to name=continued_hotels-legend,legend columns=4]
\addplot[stl comparison utl,forget plot,restrict expr to domain={\coordindex}{0:8}] table[x=budget,y=mean,col sep=comma]{figures/data/llm/fresh/hotels_q_adv.csv};
\addplot[stl comparison influence,forget plot,restrict expr to domain={\coordindex}{18:26}] table[x=budget,y=mean,col sep=comma]{figures/data/llm/fresh/hotels_q_adv.csv};
\addplot[stl comparison pm,forget plot,restrict expr to domain={\coordindex}{9:17}] table[x=budget,y=mean,col sep=comma]{figures/data/llm/fresh/hotels_q_adv.csv};
\addplot[stl strict,forget plot,restrict expr to domain={\coordindex}{27:35}] table[x=budget,y=mean,col sep=comma]{figures/data/llm/fresh/hotels_q_adv.csv};
\addlegendimage{stl comparison utl}\addlegendentry{UTL}
\addlegendimage{stl comparison pm}\addlegendentry{PM}
\addlegendimage{stl comparison influence}\addlegendentry{Practical influence}
\addlegendimage{stl strict}\addlegendentry{Strict only}
\nextgroupplot[xlabel={},xmin=60,xmax=2040,xtick={100,500,1000,1500,2000},xticklabels={100,500,1000,1500,2000},minor xtick={200,300,400,750},ymin=25,ymax=100,ytick={25,50,75,100},ylabel={}]
\addplot[stl comparison utl,forget plot,restrict expr to domain={\coordindex}{0:8}] table[x=budget,y=mean,col sep=comma]{figures/data/llm/continued/hotels_q_adv.csv};
\addplot[stl comparison influence,forget plot,restrict expr to domain={\coordindex}{18:26}] table[x=budget,y=mean,col sep=comma]{figures/data/llm/continued/hotels_q_adv.csv};
\addplot[stl comparison pm,forget plot,restrict expr to domain={\coordindex}{9:17}] table[x=budget,y=mean,col sep=comma]{figures/data/llm/continued/hotels_q_adv.csv};
\addplot[stl strict,forget plot,restrict expr to domain={\coordindex}{27:35}] table[x=budget,y=mean,col sep=comma]{figures/data/llm/continued/hotels_q_adv.csv};
\nextgroupplot[xlabel={},xmin=60,xmax=2040,xtick={100,500,1000,1500,2000},xticklabels={100,500,1000,1500,2000},minor xtick={200,300,400,750},ymode=log,ymin=.3,ymax=150,ytick={.5,2,10,50},yticklabels={0.5,2,10,50},ylabel={$E_{\rm cf}$}]
\addplot[stl comparison utl,forget plot,restrict expr to domain={\coordindex}{0:8}] table[x=budget,y=mean,col sep=comma]{figures/data/llm/fresh/hotels_spurious_energy.csv};
\addplot[stl comparison influence,forget plot,restrict expr to domain={\coordindex}{18:26}] table[x=budget,y=mean,col sep=comma]{figures/data/llm/fresh/hotels_spurious_energy.csv};
\addplot[stl comparison pm,forget plot,restrict expr to domain={\coordindex}{9:17}] table[x=budget,y=mean,col sep=comma]{figures/data/llm/fresh/hotels_spurious_energy.csv};
\addplot[stl strict,forget plot,restrict expr to domain={\coordindex}{27:35}] table[x=budget,y=mean,col sep=comma]{figures/data/llm/fresh/hotels_spurious_energy.csv};
\nextgroupplot[xlabel={},xmin=60,xmax=2040,xtick={100,500,1000,1500,2000},xticklabels={100,500,1000,1500,2000},minor xtick={200,300,400,750},ymode=log,ymin=.3,ymax=150,ytick={.5,2,10,50},yticklabels={0.5,2,10,50},ylabel={}]
\addplot[stl comparison utl,forget plot,restrict expr to domain={\coordindex}{0:8}] table[x=budget,y=mean,col sep=comma]{figures/data/llm/continued/hotels_spurious_energy.csv};
\addplot[stl comparison influence,forget plot,restrict expr to domain={\coordindex}{18:26}] table[x=budget,y=mean,col sep=comma]{figures/data/llm/continued/hotels_spurious_energy.csv};
\addplot[stl comparison pm,forget plot,restrict expr to domain={\coordindex}{9:17}] table[x=budget,y=mean,col sep=comma]{figures/data/llm/continued/hotels_spurious_energy.csv};
\addplot[stl strict,forget plot,restrict expr to domain={\coordindex}{27:35}] table[x=budget,y=mean,col sep=comma]{figures/data/llm/continued/hotels_spurious_energy.csv};
\nextgroupplot[xlabel={},xmin=60,xmax=2040,xtick={100,500,1000,1500,2000},xticklabels={100,500,1000,1500,2000},minor xtick={200,300,400,750},ymin=70,ymax=100,ytick={70,80,90,100},ylabel={$Q_{\rm clean}$ (\%)}]
\addplot[stl comparison utl,forget plot,restrict expr to domain={\coordindex}{0:8}] table[x=budget,y=mean,col sep=comma]{figures/data/llm/fresh/hotels_q_clean.csv};
\addplot[stl comparison influence,forget plot,restrict expr to domain={\coordindex}{18:26}] table[x=budget,y=mean,col sep=comma]{figures/data/llm/fresh/hotels_q_clean.csv};
\addplot[stl comparison pm,forget plot,restrict expr to domain={\coordindex}{9:17}] table[x=budget,y=mean,col sep=comma]{figures/data/llm/fresh/hotels_q_clean.csv};
\addplot[stl strict,forget plot,restrict expr to domain={\coordindex}{27:35}] table[x=budget,y=mean,col sep=comma]{figures/data/llm/fresh/hotels_q_clean.csv};
\nextgroupplot[xlabel={},xmin=60,xmax=2040,xtick={100,500,1000,1500,2000},xticklabels={100,500,1000,1500,2000},minor xtick={200,300,400,750},ymin=70,ymax=100,ytick={70,80,90,100},ylabel={}]
\addplot[stl comparison utl,forget plot,restrict expr to domain={\coordindex}{0:8}] table[x=budget,y=mean,col sep=comma]{figures/data/llm/continued/hotels_q_clean.csv};
\addplot[stl comparison influence,forget plot,restrict expr to domain={\coordindex}{18:26}] table[x=budget,y=mean,col sep=comma]{figures/data/llm/continued/hotels_q_clean.csv};
\addplot[stl comparison pm,forget plot,restrict expr to domain={\coordindex}{9:17}] table[x=budget,y=mean,col sep=comma]{figures/data/llm/continued/hotels_q_clean.csv};
\addplot[stl strict,forget plot,restrict expr to domain={\coordindex}{27:35}] table[x=budget,y=mean,col sep=comma]{figures/data/llm/continued/hotels_q_clean.csv};
\nextgroupplot[xmin=60,xmax=2040,xtick={100,500,1000,1500,2000},xticklabels={100,500,1000,1500,2000},minor xtick={200,300,400,750},ymin=97,ymax=99,ytick={97,98,99},ylabel={$P$ (\%)}]
\addplot[stl comparison utl,forget plot,restrict expr to domain={\coordindex}{0:8}] table[x=budget,y=mean,col sep=comma]{figures/data/additional/hotels_fresh_p.csv};
\addplot[stl comparison influence,forget plot,restrict expr to domain={\coordindex}{18:26}] table[x=budget,y=mean,col sep=comma]{figures/data/additional/hotels_fresh_p.csv};
\addplot[stl comparison pm,forget plot,restrict expr to domain={\coordindex}{9:17}] table[x=budget,y=mean,col sep=comma]{figures/data/additional/hotels_fresh_p.csv};
\addplot[stl strict,forget plot,restrict expr to domain={\coordindex}{27:35}] table[x=budget,y=mean,col sep=comma]{figures/data/additional/hotels_fresh_p.csv};
\nextgroupplot[xmin=60,xmax=2040,xtick={100,500,1000,1500,2000},xticklabels={100,500,1000,1500,2000},minor xtick={200,300,400,750},ymin=97,ymax=99,ytick={97,98,99},ylabel={}]
\addplot[stl comparison utl,forget plot,restrict expr to domain={\coordindex}{0:8}] table[x=budget,y=mean,col sep=comma]{figures/data/additional/hotels_continued_p.csv};
\addplot[stl comparison influence,forget plot,restrict expr to domain={\coordindex}{18:26}] table[x=budget,y=mean,col sep=comma]{figures/data/additional/hotels_continued_p.csv};
\addplot[stl comparison pm,forget plot,restrict expr to domain={\coordindex}{9:17}] table[x=budget,y=mean,col sep=comma]{figures/data/additional/hotels_continued_p.csv};
\addplot[stl strict,forget plot,restrict expr to domain={\coordindex}{27:35}] table[x=budget,y=mean,col sep=comma]{figures/data/additional/hotels_continued_p.csv};
\end{groupplot}
\node[anchor=south,font=\small\bfseries] at ($(group c1r1.north)+(0,2mm)$){Fresh};
\node[anchor=south,font=\small\bfseries] at ($(group c2r1.north)+(0,2mm)$){Continued};
\node[anchor=north] at ($(group c1r4.south west)!0.5!(group c2r4.south east)+(0,-11mm)$){\pgfplotslegendfromname{continued_hotels-legend}};
\end{tikzpicture}
\caption{Hotel results under fresh training (left) and continued training
(right). Rows share axes; colours, markers, and line styles match across
columns. Curves show means across three seeds without intervals to avoid
clutter. Dotted grey lines mark strict baselines. Energy uses a logarithmic
axis.}
\label{fig:continued-hotels}
\end{stlresultfigure}

Section~\ref{sec:experiments} discusses accuracy under $Q_{\rm adv}$ after
fresh training (Figure~\ref{fig:neural-experiments}(b)).  We report the
other metrics and compare fresh with continued training.

\paragraph{Clean accuracy and energy.}
Selection improves clean accuracy while reducing sensitivity to the altered
hotel fields (Figure~\ref{fig:continued-hotels}).  With 500 ties, PM and
Practical influence reach $92.9\%$ and $93.1\%$ clean accuracy, above the
$92.5\%$ that UTL reaches with 1,000 ties.  Their counterfactual energies,
$1.38$ and $1.25$ (95\% CIs $[1.01,1.75]$ and $[0.31,2.20]$), are also below
the $1.89$ that UTL reaches with 1,000 ties.  At 500 ties, UTL reaches
$85.7\%$ clean accuracy and energy $5.12$ (95\% CI $[3.45,6.80]$), and the
strict model has energy $27.38$.

\paragraph{Fresh versus continued training.}
Continued training helps at low budgets but falls behind fresh training at high
budgets.  At $k=100$, continued PM reaches $67.3\%$ adversarial accuracy,
compared with $55.9\%$ after fresh training.  At $k=2{,}000$, fresh PM reaches
$95.3\%$, compared with $80.2\%$ after continuation.  The gain at low budgets is
consistent with continuation starting from a model already trained on the
strict comparisons, despite matched training steps.
We hypothesize that adding many ties changes the objective enough to make the
strict checkpoint a poorer starting point.  This comparison does not
establish a change of basin.

\paragraph{Accuracy and energy need not agree.}
Continued training raises clean and adversarial accuracy, yet the mean energy
after continued training remains above its strict baseline at every budget.  At the full pool, continued
PM has $86.9\%$ clean accuracy and energy $32.75$ (95\% CI
$[30.57,34.94]$), whereas fresh PM has $95.8\%$ and $0.48$.  The continued
accuracy gains therefore do not demonstrate lower sensitivity to the
counterfactual operator.

\paragraph{In-distribution accuracy.}
In-distribution accuracy stays high under both training procedures:
with the full pool, PM reaches $97.9\%$ after fresh training and $98.6\%$ after
continuation, compared with $97.8\%$ without ties.

\subsubsection{Additional Results: Emojis}

\begin{stlresultfigure}
\begin{tikzpicture}
\begin{groupplot}[stl appendix axis,group style={group size=2 by 4,horizontal sep=16mm,vertical sep=8mm},width=.35\linewidth,height=.115\linewidth]
\nextgroupplot[xlabel={},xmin=60,xmax=2040,xtick={100,500,1000,1500,2000},xticklabels={100,500,1000,1500,2000},minor xtick={200,300,400,750},ymin=40,ymax=82,ytick={40,50,60,70,80},ylabel={$Q_{\rm adv}$ (\%)},legend to name=continued_emoji-legend,legend columns=4]
\addplot[stl comparison utl,forget plot,restrict expr to domain={\coordindex}{0:8}] table[x=budget,y=mean,col sep=comma]{figures/data/llm/fresh/emoji_q_adv.csv};
\addplot[stl comparison influence,forget plot,restrict expr to domain={\coordindex}{18:26}] table[x=budget,y=mean,col sep=comma]{figures/data/llm/fresh/emoji_q_adv.csv};
\addplot[stl comparison pm,forget plot,restrict expr to domain={\coordindex}{9:17}] table[x=budget,y=mean,col sep=comma]{figures/data/llm/fresh/emoji_q_adv.csv};
\addplot[stl strict,forget plot,restrict expr to domain={\coordindex}{27:35}] table[x=budget,y=mean,col sep=comma]{figures/data/llm/fresh/emoji_q_adv.csv};
\addlegendimage{stl comparison utl}\addlegendentry{UTL}
\addlegendimage{stl comparison pm}\addlegendentry{PM}
\addlegendimage{stl comparison influence}\addlegendentry{Practical influence}
\addlegendimage{stl strict}\addlegendentry{Strict only}
\nextgroupplot[xlabel={},xmin=60,xmax=2040,xtick={100,500,1000,1500,2000},xticklabels={100,500,1000,1500,2000},minor xtick={200,300,400,750},ymin=40,ymax=82,ytick={40,50,60,70,80},ylabel={}]
\addplot[stl comparison utl,forget plot,restrict expr to domain={\coordindex}{0:8}] table[x=budget,y=mean,col sep=comma]{figures/data/llm/continued/emoji_q_adv.csv};
\addplot[stl comparison influence,forget plot,restrict expr to domain={\coordindex}{18:26}] table[x=budget,y=mean,col sep=comma]{figures/data/llm/continued/emoji_q_adv.csv};
\addplot[stl comparison pm,forget plot,restrict expr to domain={\coordindex}{9:17}] table[x=budget,y=mean,col sep=comma]{figures/data/llm/continued/emoji_q_adv.csv};
\addplot[stl strict,forget plot,restrict expr to domain={\coordindex}{27:35}] table[x=budget,y=mean,col sep=comma]{figures/data/llm/continued/emoji_q_adv.csv};
\nextgroupplot[xlabel={},xmin=60,xmax=2040,xtick={100,500,1000,1500,2000},xticklabels={100,500,1000,1500,2000},minor xtick={200,300,400,750},ymode=log,ymin=.4,ymax=65,ytick={.5,2,10,50},yticklabels={0.5,2,10,50},ylabel={$E_{\rm cf}$}]
\addplot[stl comparison utl,forget plot,restrict expr to domain={\coordindex}{0:8}] table[x=budget,y=mean,col sep=comma]{figures/data/llm/fresh/emoji_spurious_energy.csv};
\addplot[stl comparison influence,forget plot,restrict expr to domain={\coordindex}{18:26}] table[x=budget,y=mean,col sep=comma]{figures/data/llm/fresh/emoji_spurious_energy.csv};
\addplot[stl comparison pm,forget plot,restrict expr to domain={\coordindex}{9:17}] table[x=budget,y=mean,col sep=comma]{figures/data/llm/fresh/emoji_spurious_energy.csv};
\addplot[stl strict,forget plot,restrict expr to domain={\coordindex}{27:35}] table[x=budget,y=mean,col sep=comma]{figures/data/llm/fresh/emoji_spurious_energy.csv};
\nextgroupplot[xlabel={},xmin=60,xmax=2040,xtick={100,500,1000,1500,2000},xticklabels={100,500,1000,1500,2000},minor xtick={200,300,400,750},ymode=log,ymin=.4,ymax=65,ytick={.5,2,10,50},yticklabels={0.5,2,10,50},ylabel={}]
\addplot[stl comparison utl,forget plot,restrict expr to domain={\coordindex}{0:8}] table[x=budget,y=mean,col sep=comma]{figures/data/llm/continued/emoji_spurious_energy.csv};
\addplot[stl comparison influence,forget plot,restrict expr to domain={\coordindex}{18:26}] table[x=budget,y=mean,col sep=comma]{figures/data/llm/continued/emoji_spurious_energy.csv};
\addplot[stl comparison pm,forget plot,restrict expr to domain={\coordindex}{9:17}] table[x=budget,y=mean,col sep=comma]{figures/data/llm/continued/emoji_spurious_energy.csv};
\addplot[stl strict,forget plot,restrict expr to domain={\coordindex}{27:35}] table[x=budget,y=mean,col sep=comma]{figures/data/llm/continued/emoji_spurious_energy.csv};
\nextgroupplot[xlabel={},xmin=60,xmax=2040,xtick={100,500,1000,1500,2000},xticklabels={100,500,1000,1500,2000},minor xtick={200,300,400,750},ymin=68,ymax=88,ytick={70,75,80,85},ylabel={$Q_{\rm clean}$ (\%)}]
\addplot[stl comparison utl,forget plot,restrict expr to domain={\coordindex}{0:8}] table[x=budget,y=mean,col sep=comma]{figures/data/llm/fresh/emoji_q_clean.csv};
\addplot[stl comparison influence,forget plot,restrict expr to domain={\coordindex}{18:26}] table[x=budget,y=mean,col sep=comma]{figures/data/llm/fresh/emoji_q_clean.csv};
\addplot[stl comparison pm,forget plot,restrict expr to domain={\coordindex}{9:17}] table[x=budget,y=mean,col sep=comma]{figures/data/llm/fresh/emoji_q_clean.csv};
\addplot[stl strict,forget plot,restrict expr to domain={\coordindex}{27:35}] table[x=budget,y=mean,col sep=comma]{figures/data/llm/fresh/emoji_q_clean.csv};
\nextgroupplot[xlabel={},xmin=60,xmax=2040,xtick={100,500,1000,1500,2000},xticklabels={100,500,1000,1500,2000},minor xtick={200,300,400,750},ymin=68,ymax=88,ytick={70,75,80,85},ylabel={}]
\addplot[stl comparison utl,forget plot,restrict expr to domain={\coordindex}{0:8}] table[x=budget,y=mean,col sep=comma]{figures/data/llm/continued/emoji_q_clean.csv};
\addplot[stl comparison influence,forget plot,restrict expr to domain={\coordindex}{18:26}] table[x=budget,y=mean,col sep=comma]{figures/data/llm/continued/emoji_q_clean.csv};
\addplot[stl comparison pm,forget plot,restrict expr to domain={\coordindex}{9:17}] table[x=budget,y=mean,col sep=comma]{figures/data/llm/continued/emoji_q_clean.csv};
\addplot[stl strict,forget plot,restrict expr to domain={\coordindex}{27:35}] table[x=budget,y=mean,col sep=comma]{figures/data/llm/continued/emoji_q_clean.csv};
\nextgroupplot[xmin=60,xmax=2040,xtick={100,500,1000,1500,2000},xticklabels={100,500,1000,1500,2000},minor xtick={200,300,400,750},ymin=91,ymax=98,ytick={92,94,96,98},ylabel={$P$ (\%)}]
\addplot[stl comparison utl,forget plot,restrict expr to domain={\coordindex}{0:8}] table[x=budget,y=mean,col sep=comma]{figures/data/additional/emoji_fresh_p.csv};
\addplot[stl comparison influence,forget plot,restrict expr to domain={\coordindex}{18:26}] table[x=budget,y=mean,col sep=comma]{figures/data/additional/emoji_fresh_p.csv};
\addplot[stl comparison pm,forget plot,restrict expr to domain={\coordindex}{9:17}] table[x=budget,y=mean,col sep=comma]{figures/data/additional/emoji_fresh_p.csv};
\addplot[stl strict,forget plot,restrict expr to domain={\coordindex}{27:35}] table[x=budget,y=mean,col sep=comma]{figures/data/additional/emoji_fresh_p.csv};
\nextgroupplot[xmin=60,xmax=2040,xtick={100,500,1000,1500,2000},xticklabels={100,500,1000,1500,2000},minor xtick={200,300,400,750},ymin=91,ymax=98,ytick={92,94,96,98},ylabel={}]
\addplot[stl comparison utl,forget plot,restrict expr to domain={\coordindex}{0:8}] table[x=budget,y=mean,col sep=comma]{figures/data/additional/emoji_continued_p.csv};
\addplot[stl comparison influence,forget plot,restrict expr to domain={\coordindex}{18:26}] table[x=budget,y=mean,col sep=comma]{figures/data/additional/emoji_continued_p.csv};
\addplot[stl comparison pm,forget plot,restrict expr to domain={\coordindex}{9:17}] table[x=budget,y=mean,col sep=comma]{figures/data/additional/emoji_continued_p.csv};
\addplot[stl strict,forget plot,restrict expr to domain={\coordindex}{27:35}] table[x=budget,y=mean,col sep=comma]{figures/data/additional/emoji_continued_p.csv};
\end{groupplot}
\node[anchor=south,font=\small\bfseries] at ($(group c1r1.north)+(0,2mm)$){Fresh};
\node[anchor=south,font=\small\bfseries] at ($(group c2r1.north)+(0,2mm)$){Continued};
\node[anchor=north] at ($(group c1r4.south west)!0.5!(group c2r4.south east)+(0,-11mm)$){\pgfplotslegendfromname{continued_emoji-legend}};
\end{tikzpicture}
\caption{Emoji results under fresh training (left) and continued training
(right). Rows share axes; colours, markers, and line styles match across
columns. Curves show means across three seeds without intervals to avoid
clutter. Dotted grey lines mark strict baselines. Energy uses a logarithmic
axis.}
\label{fig:continued-emoji}
\end{stlresultfigure}

Section~\ref{sec:experiments} discusses accuracy under $Q_{\rm adv}$ after
fresh training (Figure~\ref{fig:neural-experiments}(c)).  We report the
other metrics and compare fresh with continued training.

\paragraph{Clean accuracy and energy.}
As for hotels, selection improves clean accuracy and lowers counterfactual
energy, with PM ahead of Practical influence (Figure~\ref{fig:continued-emoji}).
With 500 ties, PM reaches $81.3\%$ clean accuracy and energy $2.05$ (95\% CI
$[1.34,2.75]$), better on both metrics than UTL with 750 ties ($80.5\%$ and
$2.74$).  At 500 ties, Practical influence reaches $79.5\%$ and $3.36$ (95\%
CI $[2.38,4.33]$), and UTL reaches $77.5\%$ and $5.15$ (95\% CI
$[4.58,5.72]$).

\paragraph{Fresh versus continued training.}
As with hotels, continued training helps at low budgets but loses its advantage as
more ties enter training.  Continued PM reaches $61.3\%$ adversarial accuracy
at $k=100$, versus $50.3\%$ after fresh training; at $k=2{,}000$, fresh PM
reaches $75.5\%$, versus $66.8\%$ after continuation.

\paragraph{Accuracy and energy need not agree.}
Continued PM improves clean accuracy from $71.1\%$ to $81.7\%$ at the full
pool, yet its mean energy rises from $20.69$ to $24.65$ (95\% CI
$[20.48,28.82]$).  Fresh PM reaches $84.0\%$ clean accuracy with energy
$0.70$.  At every budget, the mean energy after continued training stays above the strict baseline, so higher
accuracy does not by itself show that the model has stopped using emojis.

\paragraph{In-distribution accuracy.}
Fresh training reduces $P$ accuracy from $96.1\%$ without ties to $92.6\%$,
whereas continued PM retains $96.4\%$.  The stronger mitigation under fresh
training thus comes with an in-distribution trade-off, and continuation
retains more of the original preference pattern.

\subsubsection{Additional Results: Qwen3-4B}

\begin{stlresultfigure}
\begin{tikzpicture}
\begin{groupplot}[stl appendix axis,
group style={group size=2 by 4,horizontal sep=16mm,vertical sep=8mm},
width=.35\linewidth,height=.115\linewidth,
xmin=60,xmax=2040,xtick={100,500,1000,1500,2000},
xticklabels={100,500,1000,1500,2000},minor xtick={200,300,400,750}]
\nextgroupplot[xlabel={},ymin=10,ymax=90,ytick={10,30,50,70,90},ylabel={$Q_{\rm adv}$ (\%)},legend to name=qwen-results-legend,legend columns=4]
\addplot[stl comparison utl,forget plot,restrict expr to domain={\coordindex}{0:8}] table[x=budget,y=mean,col sep=comma]{figures/data/llm/qwen3_4b_seed0/hotels_q_adv.csv};
\addplot[stl comparison influence,forget plot,restrict expr to domain={\coordindex}{18:26}] table[x=budget,y=mean,col sep=comma]{figures/data/llm/qwen3_4b_seed0/hotels_q_adv.csv};
\addplot[stl comparison pm,forget plot,restrict expr to domain={\coordindex}{9:17}] table[x=budget,y=mean,col sep=comma]{figures/data/llm/qwen3_4b_seed0/hotels_q_adv.csv};
\addplot[stl strict,forget plot,restrict expr to domain={\coordindex}{27:35}] table[x=budget,y=mean,col sep=comma]{figures/data/llm/qwen3_4b_seed0/hotels_q_adv.csv};
\addlegendimage{stl comparison utl}\addlegendentry{UTL}
\addlegendimage{stl comparison pm}\addlegendentry{PM}
\addlegendimage{stl comparison influence}\addlegendentry{Practical influence}
\addlegendimage{stl strict}\addlegendentry{Strict only}
\nextgroupplot[xlabel={},ymin=47,ymax=83,ytick={50,60,70,80},ylabel={}]
\addplot[stl comparison utl,forget plot,restrict expr to domain={\coordindex}{0:8}] table[x=budget,y=mean,col sep=comma]{figures/data/llm/qwen3_4b_seed0/emoji_q_adv.csv};
\addplot[stl comparison influence,forget plot,restrict expr to domain={\coordindex}{18:26}] table[x=budget,y=mean,col sep=comma]{figures/data/llm/qwen3_4b_seed0/emoji_q_adv.csv};
\addplot[stl comparison pm,forget plot,restrict expr to domain={\coordindex}{9:17}] table[x=budget,y=mean,col sep=comma]{figures/data/llm/qwen3_4b_seed0/emoji_q_adv.csv};
\addplot[stl strict,forget plot,restrict expr to domain={\coordindex}{27:35}] table[x=budget,y=mean,col sep=comma]{figures/data/llm/qwen3_4b_seed0/emoji_q_adv.csv};
\nextgroupplot[xlabel={},ymode=log,ymin=1,ymax=26,ytick={1,2,5,10,20},yticklabels={1,2,5,10,20},ylabel={$E_{\rm cf}$}]
\addplot[stl comparison utl,forget plot,restrict expr to domain={\coordindex}{0:8}] table[x=budget,y=mean,col sep=comma]{figures/data/llm/qwen3_4b_seed0/hotels_spurious_energy.csv};
\addplot[stl comparison influence,forget plot,restrict expr to domain={\coordindex}{18:26}] table[x=budget,y=mean,col sep=comma]{figures/data/llm/qwen3_4b_seed0/hotels_spurious_energy.csv};
\addplot[stl comparison pm,forget plot,restrict expr to domain={\coordindex}{9:17}] table[x=budget,y=mean,col sep=comma]{figures/data/llm/qwen3_4b_seed0/hotels_spurious_energy.csv};
\addplot[stl strict,forget plot,restrict expr to domain={\coordindex}{27:35}] table[x=budget,y=mean,col sep=comma]{figures/data/llm/qwen3_4b_seed0/hotels_spurious_energy.csv};
\nextgroupplot[xlabel={},ymode=log,ymin=.4,ymax=32,ytick={.5,2,10,30},yticklabels={0.5,2,10,30},ylabel={}]
\addplot[stl comparison utl,forget plot,restrict expr to domain={\coordindex}{0:8}] table[x=budget,y=mean,col sep=comma]{figures/data/llm/qwen3_4b_seed0/emoji_spurious_energy.csv};
\addplot[stl comparison influence,forget plot,restrict expr to domain={\coordindex}{18:26}] table[x=budget,y=mean,col sep=comma]{figures/data/llm/qwen3_4b_seed0/emoji_spurious_energy.csv};
\addplot[stl comparison pm,forget plot,restrict expr to domain={\coordindex}{9:17}] table[x=budget,y=mean,col sep=comma]{figures/data/llm/qwen3_4b_seed0/emoji_spurious_energy.csv};
\addplot[stl strict,forget plot,restrict expr to domain={\coordindex}{27:35}] table[x=budget,y=mean,col sep=comma]{figures/data/llm/qwen3_4b_seed0/emoji_spurious_energy.csv};
\nextgroupplot[xlabel={},ymin=60,ymax=90,ytick={60,70,80,90},ylabel={$Q_{\rm clean}$ (\%)}]
\addplot[stl comparison utl,forget plot,restrict expr to domain={\coordindex}{0:8}] table[x=budget,y=mean,col sep=comma]{figures/data/llm/qwen3_4b_seed0/hotels_q_clean.csv};
\addplot[stl comparison influence,forget plot,restrict expr to domain={\coordindex}{18:26}] table[x=budget,y=mean,col sep=comma]{figures/data/llm/qwen3_4b_seed0/hotels_q_clean.csv};
\addplot[stl comparison pm,forget plot,restrict expr to domain={\coordindex}{9:17}] table[x=budget,y=mean,col sep=comma]{figures/data/llm/qwen3_4b_seed0/hotels_q_clean.csv};
\addplot[stl strict,forget plot,restrict expr to domain={\coordindex}{27:35}] table[x=budget,y=mean,col sep=comma]{figures/data/llm/qwen3_4b_seed0/hotels_q_clean.csv};
\nextgroupplot[xlabel={},ymin=71,ymax=88,ytick={72,76,80,84,88},ylabel={}]
\addplot[stl comparison utl,forget plot,restrict expr to domain={\coordindex}{0:8}] table[x=budget,y=mean,col sep=comma]{figures/data/llm/qwen3_4b_seed0/emoji_q_clean.csv};
\addplot[stl comparison influence,forget plot,restrict expr to domain={\coordindex}{18:26}] table[x=budget,y=mean,col sep=comma]{figures/data/llm/qwen3_4b_seed0/emoji_q_clean.csv};
\addplot[stl comparison pm,forget plot,restrict expr to domain={\coordindex}{9:17}] table[x=budget,y=mean,col sep=comma]{figures/data/llm/qwen3_4b_seed0/emoji_q_clean.csv};
\addplot[stl strict,forget plot,restrict expr to domain={\coordindex}{27:35}] table[x=budget,y=mean,col sep=comma]{figures/data/llm/qwen3_4b_seed0/emoji_q_clean.csv};
\nextgroupplot[ymin=95.8,ymax=98.3,ytick={96,97,98},ylabel={$P$ (\%)}]
\addplot[stl comparison utl,forget plot,restrict expr to domain={\coordindex}{0:8}] table[x=budget,y=mean,col sep=comma]{figures/data/llm/qwen3_4b_seed0/hotels_p.csv};
\addplot[stl comparison influence,forget plot,restrict expr to domain={\coordindex}{18:26}] table[x=budget,y=mean,col sep=comma]{figures/data/llm/qwen3_4b_seed0/hotels_p.csv};
\addplot[stl comparison pm,forget plot,restrict expr to domain={\coordindex}{9:17}] table[x=budget,y=mean,col sep=comma]{figures/data/llm/qwen3_4b_seed0/hotels_p.csv};
\addplot[stl strict,forget plot,restrict expr to domain={\coordindex}{27:35}] table[x=budget,y=mean,col sep=comma]{figures/data/llm/qwen3_4b_seed0/hotels_p.csv};
\nextgroupplot[ymin=92.5,ymax=96.5,ytick={93,94,95,96},ylabel={}]
\addplot[stl comparison utl,forget plot,restrict expr to domain={\coordindex}{0:8}] table[x=budget,y=mean,col sep=comma]{figures/data/llm/qwen3_4b_seed0/emoji_p.csv};
\addplot[stl comparison influence,forget plot,restrict expr to domain={\coordindex}{18:26}] table[x=budget,y=mean,col sep=comma]{figures/data/llm/qwen3_4b_seed0/emoji_p.csv};
\addplot[stl comparison pm,forget plot,restrict expr to domain={\coordindex}{9:17}] table[x=budget,y=mean,col sep=comma]{figures/data/llm/qwen3_4b_seed0/emoji_p.csv};
\addplot[stl strict,forget plot,restrict expr to domain={\coordindex}{27:35}] table[x=budget,y=mean,col sep=comma]{figures/data/llm/qwen3_4b_seed0/emoji_p.csv};
\end{groupplot}
\node[anchor=south,font=\small\bfseries] at ($(group c1r1.north)+(0,2mm)$){Hotels};
\node[anchor=south,font=\small\bfseries] at ($(group c2r1.north)+(0,2mm)$){Emojis};
\node[anchor=north] at ($(group c1r4.south west)!0.5!(group c2r4.south east)+(0,-11mm)$){\pgfplotslegendfromname{qwen-results-legend}};
\end{tikzpicture}
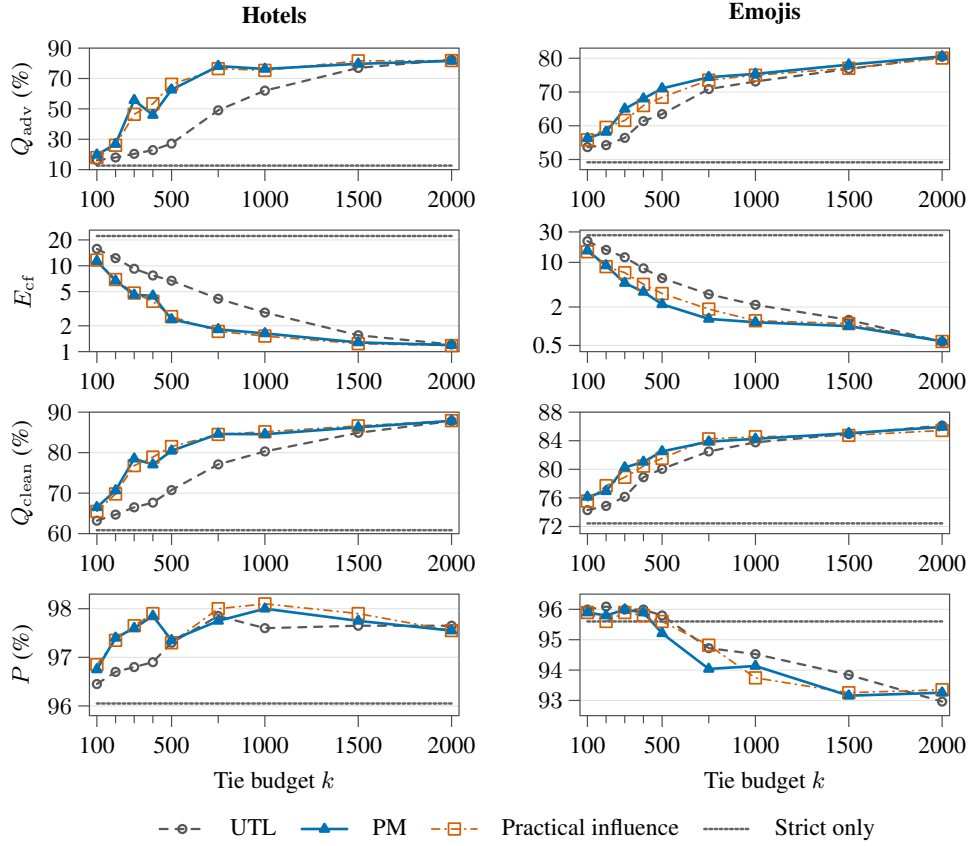
\caption{Results after fresh training for Qwen3-4B-Instruct-2507 on hotels (left)
and emojis (right). Y-axis ranges differ between columns to show
differences within each task. Curves show seed~0 with one UTL draw;
uncertainty intervals are not estimated. Dotted grey lines mark strict
baselines. Energy uses a logarithmic axis.}
\label{fig:qwen-results}
\end{stlresultfigure}

We also test whether the selection gains extend beyond Llama-3.2-1B by repeating the experiment with fresh training using Qwen3-4B-Instruct-2507, a larger model from a different family.  The model runs in unquantized bfloat16 with LoRA rank 8 and scaling factor 16.  We use a batch of 2 per device with 8 accumulation steps, which keeps the effective batch size of 16.  All other settings follow Table~\ref{tab:llm-hyperparameters}: all methods select from the same pool of 2,000 ties, and Practical influence uses a probe of 500 ties and 1,000 comparisons for the GGN.  These results use one seed and one UTL draw, so they carry no uncertainty estimates (Figure~\ref{fig:qwen-results}).

\paragraph{Hotels.} Both selectors improve adversarial accuracy over UTL at every tested budget below the full pool.  With 500 ties, PM and Practical influence reach
$62.6\%$ and $66.2\%$, above the $62.0\%$ that UTL reaches with 1,000 ties and the $27.2\%$ it reaches with 500.  Their clean accuracies also improve to $80.5\%$ and $81.5\%$, versus $70.8\%$, while energy falls to $2.38$ and $2.56$, versus $6.70$.  Neither selector is better than the other at every budget.  With a single seed, the curves are also less smooth than the Llama means: PM falls from $55.7\%$ at $k=300$ to $46.0\%$ at $k=400$.

\paragraph{Emojis.} Selection gives smaller but consistent gains over UTL at budgets below the full pool.  With 500 ties, PM reaches $71.1\%$ adversarial accuracy, above the $70.9\%$ that UTL reaches with 750 ties, compared with $68.4\%$ for Practical influence and $63.4\%$ for UTL at 500.  Clean accuracy rises to $82.5\%$ for PM, versus $81.5\%$ and $80.1\%$, and energy falls to $2.20$, versus $3.25$ and $5.67$.  The accuracy gains therefore accompany lower sensitivity to emoji removal.  At the full pool, all three methods reach about $80\%$ adversarial accuracy and energy near $0.58$.

\paragraph{Full pool.} At $k=M$, all methods train on the same ties.  Their remaining differences in adversarial accuracy, at most $0.9$ percentage points on hotels and $0.5$ on emojis, come from separate training runs.

\paragraph{Comparison with Llama.} On hotels, Qwen relies more on the spurious fields than Llama.  Its strict model reaches $12.7\%$ accuracy under $Q_{\rm adv}$ and $60.9\%$ under $Q_{\rm clean}$, compared with means of $30.5\%$ and $73.4\%$ for Llama.  With the full pool, Qwen reaches $81.6$--$82.5\%$ under $Q_{\rm adv}$ and energy $1.2$, compared with $95.3\%$ and $0.48$ for Llama.  Its strict energy, however, is lower: $22.21$ versus $27.38$.  On emojis, Qwen's accuracies are similar to or higher than Llama's.

\paragraph{In-distribution accuracy.} On hotels, accuracy under $P$ stays high: with the full pool, PM reaches $97.6\%$, versus $96.1\%$ without ties.  For emojis, it falls from $95.6\%$ to $93.3\%$, an in-distribution trade-off like the one observed for Llama.

\FloatBarrier

\newpage 
\section{Limitations and Future Work}
\label{app:limitations}
\paragraph{Controlled evaluation.}
Our theory holds for log-linear policies with a known decomposition into core and spurious features, at a stationary point of the DPO loss, and to first-order in the tie weight. We therefore test its quantitative predictions only where these conditions hold, and in neural networks and language models we ask whether the selection gains persist.  Controlled tasks provide the quantities needed for both tests.  Because we know the spurious attribute, we can construct ties whose responses share their core content, define $Q_{\rm adv}$ and $Q_{\rm clean}$ by reversing or removing the spurious correlation, and measure spurious reliance with the counterfactual energy. These choices isolate the predictions of the analysis.  Applying the same tests to natural preference data would require identifying a spurious attribute, validating that the transformation preserves utility, and defining appropriate evaluation distributions. Evaluating STL under these conditions and with larger models remains future work.

\paragraph{Known spurious attribute and strict checkpoint.} STL selects from a given pool of ties, and our pools come from an operator $T$ that changes a known spurious attribute.  Scoring also requires a strict checkpoint trained on the preference data.  Both requirements hold naturally when STL corrects a behavior observed after preference training.  If a trained model favors longer responses, for example, the spurious attribute is verbosity, and the trained model itself serves as the strict checkpoint. STL then chooses which ties to add before retraining.  When the spurious attribute is unknown, constructing a useful tie pool remains an open question.

\paragraph{Relation to RRM.} RRM~\citep{DBLP:conf/iclr/000200CWJGS0YSM25} offers an easy way to construct comparisons that it calls ties: it pairs two responses written for other prompts and labels them as a tie for the current prompt.  This construction needs no knowledge of the spurious attribute.  Our theory, however, does not apply to these pairs.  It requires the two responses of a tie to share their core features, so that the tie gradient has no core component (Section~\ref{subsec:first_order_effect}) and tie training adds curvature only along spurious directions~\citep{moya2026spurious}.  Two responses to different prompts generally differ in core content, so an RRM tie also acts on the core weights.  We could not test RRM ties in the log-linear model, because our synthetic data have no prompts: features depend on the response alone, so responses from other prompts are undefined.  Extending the analysis to ties with a nonzero core difference, and comparing selection on RRM ties empirically, remain open.

\ifshowunfinishedsections
\section{Computational Accounting}
\label{app:compute}
\fi

\end{document}